%% file: main_iclr.tex
\pdfoutput=1 %
\documentclass{article} %
\usepackage{etoc}
\usepackage{iclr2027_conference,times}
\usepackage[varqu]{zi4}   %

\input{math_commands.tex}

\usepackage{url}
\usepackage{booktabs}
\definecolor{blue1}{HTML}{2851CC}
\definecolor{red1}{HTML}{E00000}

\hypersetup{
    colorlinks,
    linkcolor = red1,
    citecolor = blue1,
    urlcolor={blue!80!black},
    pdftitle={A Finslerian Approach for Embedding Directed Data},
    pdfauthor={Gwendal Debaussart-Joniec, Théau Blanchard, Argyris Kalogeratos},
}

\usepackage[normalem]{ulem}

 \newcounter{marginNoteCounter}

\title{A Finslerian Approach for Embedding \\ Directed Data}

\author{Gwendal Debaussart-Joniec\thanks{equal contribution.}\\
  ENS Paris-Saclay, Universit\'e Paris-Saclay, CNRS, Centre Borelli\\
  F-91190 Gif-sur-Yvette, France \\
  \texttt{gwendal.debaussart@ens-paris-saclay.fr} \\
  \AND
  Théau Blanchard\textsuperscript{*} \\
  UMR 1346 HeKA, Universit\'e Paris Cité, Inria, Inserm\\
  GE Healthcare \\
  F-75015 Paris, France \\
  \texttt{theau.blanchard@inria.fr} \\
  \AND
  Argyris Kalogeratos \\
  ENS Paris-Saclay, Universit\'e Paris-Saclay, CNRS, Centre Borelli\\
  F-91190 Gif-sur-Yvette, France \\
	\texttt{argyris.kalogeratos@ens-paris-saclay.fr} \\
}

\def\arxivheadertext{}
\arxivfinalcopy
\begin{document}

\maketitle

\begin{abstract}
  Many datasets carry an intrinsic directionality: citations point backward in time, cells differentiate along lineages, and traffic follows preferred routes. Spectral embedding methods, including most of their extensions to directed graphs, discard this information: they symmetrize the data and map it into a Euclidean space where asymmetry cannot be represented. We instead model directed data as sampled from a Finsler manifold, whose distance depends on the direction of travel, and study the kernel operator built from this asymmetric distance. Through a moment expansion of this operator, we show that its symmetric and antisymmetric parts separate geometry from direction. As the bandwidth of the kernel vanishes, the symmetric part converges to a weighted Laplacian, recovering diffusion maps in the Riemannian case, while the antisymmetric part converges to a first-order transport operator that encodes the directionality. We prove that the corresponding graph operators, built from finitely many samples, converge uniformly and almost surely to these limits. For Randers metrics, this vector field is explicit and yields an embedding algorithm recovering both the manifold structure, from the spectrum of the symmetric part, and the underlying drift. We illustrate the approach on synthetic directed graphs and point-clouds.
\end{abstract}

\section{Introduction} \label{sec:introduction}

A citation never points to the future, a cell rarely un-differentiates, influence in a social network is rarely returned in equal measure, and traffic or commodity flows between two locations are rarely symmetric. These are just some of the many examples of \emph{directed data}, \ie data with asymmetric relationships between entities. This directionality can be given explicitly, as in directed graphs \citep{debaussart_parpic_2026, chung_laplacians_2005, cucuringu_hermitian_2020, sevi_generalized_2026}, or arise from points in a continuous space, as in RNA velocity \citep{bergen2021rna} or lineage data \citep{zweig_learning_2026}. For such data, \emph{Finsler geometry}, an extension of Riemannian geometry allowing direction-dependent metrics, offers a natural mathematical language. Yet, Finsler geometry remains largely unexplored in machine learning or manifold learning.

Graph Laplacians and diffusion operators underlie many of the tools we rely on for structured data, including spectral clustering \citep{von2007tutorial}, dimensionality reduction \citep{donoho2003hessian}, semi-supervised learning, and graph neural networks \citep{kipf2016semi}. Their theoretical grounding rests on a convergence result: when data is sampled from a compact Riemannian manifold, the graph Laplacian converges, as the sample size grows and the kernel bandwidth shrinks, to the Laplace--Beltrami operator \citep{coifman_diffusion_2006, belkin_laplacian_2001, lafon2004diffusion, hein2007graph, singer_from_2006, calder_improved_2022, wilson_peoples_spectral_2025}. This theory is symmetric by construction, as the kernel is a function of a symmetric distance.

Methods for directed graphs generally restore symmetry somewhere along the way, through the stationary distribution of a random walk \citep{chung_laplacians_2005, zhou2003learning}, an explicit symmetrization of the adjacency matrix \citep{satuluri_symmetrizations_2011}, an encoding of the direction in the complex phase of a Hermitian matrix \citep{cucuringu_hermitian_2020, fanuel_magnetic_2018, zhang_magnet_2021}, or vertex measures \citep{sevi_generalized_2026, debaussart_parpic_2026}. Directed embedding methods \citep{chen2007directed, ou2016asymmetric} keep the asymmetry in the objective, but, like the spectral ones, return points of a Euclidean space, in which the directionality of the data is lost. Closest to our work, \citet{perrault2011directed} and \citet{yuan2022continuum} take continuous limits of Laplacian-type operators built from an asymmetric kernel, and obtain a diffusion term together with an advection term. There, as in other uses of asymmetric kernels \citep{wu2010asymmetric, he2023learning, tao2024learning}, the asymmetry is introduced by hand, without an underlying geometry.

Finsler geometry is classical in mathematics \citep{finsler1918ueber, yeonHistoryBirthFinsler, bao2012introduction, shen2001lectures, ohtaComparisonFinslerGeometry2021}, and arises wherever the cost of moving depends on the direction of travel \citep{zermelo1931navigationsproblem, markvorsenFinslerGeodesicSpray2016, gahtan2026differentiable, asanjarani2021finslertraffic, melanokos_2008_finsler, chen2024region, pfeifer2019finsler}. In machine learning, it has been used to model anisotropic or asymmetric structures \citep{lopez_2021_Symmetric, Weber_finsler_2024, pouplin2023identifyinglatentdistancesfinslerian, zweig_learning_2026, shaska_finsler_2025}, and \citet{dages_finsler_2025, dages_harnessing_2026} generalize multidimensional scaling to a Finslerian embedding space. These methods learn direction-dependent representations, but the Finsler metric usually has a prescribed form \citep{Weber_finsler_2024, dages_finsler_2025}.

This work proposes a principled approach for representing directed data, grounded in Finsler geometry: we model the data as sampled from a Finsler manifold, so that the asymmetry comes from an arbitrary underlying geometry rather than from a hand-crafted kernel, and we recover this geometry, through a Randers approximation of the metric, in the embedding space (see \cref{fig:Finsler_setting}). Our contributions are as follows.
\begin{itemize}[itemsep=0em, topsep=0em, leftmargin=*]
  \item \textbf{Continuum limits} (\cref{sec:main_results}). Building on the convergence theory of graph Laplacians, we derive a moment expansion of the Finsler kernel operator. Its symmetric part converges to a weighted Laplacian of the Binet--Legendre metric, recovering diffusion maps in the Riemannian case, while its antisymmetric part converges to the derivative along the centroid field of the unit ball, which encodes the directionality.
  \item \textbf{Consistency} (\cref{sec:discrete_operators}). We prove that the corresponding graph operators, built from finitely many samples, converge uniformly and almost surely to these limits, with an explicit rate.
  \item \textbf{Directed embedding} (\cref{sec:randers_approximation}). We introduce the moment-Randers metric, a Randers approximation of the Finsler metric with the same limiting operators, and show that it can be estimated from the two graph operators alone. This yields an algorithm returning both an embedding and a Randers metric on it.
  \item \textbf{Experiments} (\cref{sec:experiments}). We illustrate the approach on point-clouds sampled from Randers and Matsumoto manifolds, as well as on directed stochastic block models.
\end{itemize}

\begin{figure}[t]
  \centering
  \includegraphics[width=0.85\textwidth]{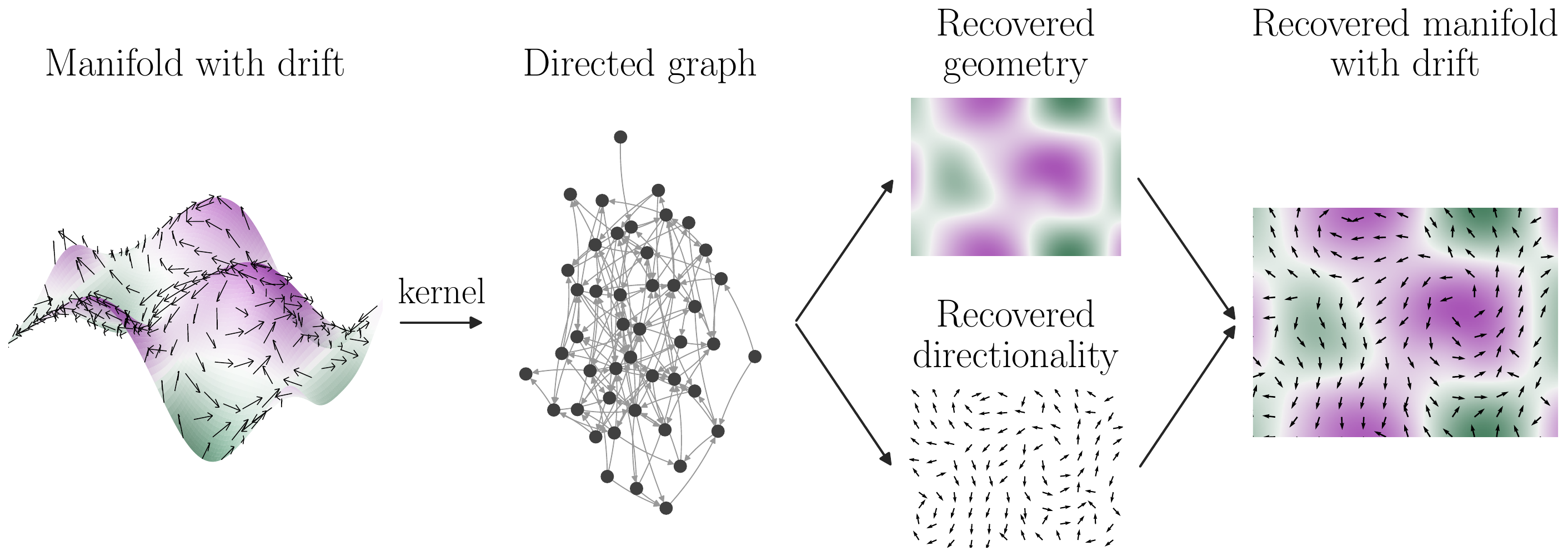}
  \caption{\textbf{Recovering geometry and directionality from Finsler manifolds.} A kernel built from the Finsler distance turns sampled points into a directed graph. The operators $\gL^s$ and $\gL^a$ of \cref{thm:sym_limit} recover the geometry and the directional structure, and are merged into a Randers metric on the embedding.}
  \label{fig:Finsler_setting}
\end{figure}

\section{Background and setting} \label{sec:background}

\subsection{Basic concepts}

Notations are summarized in \cref{tab:notation_table,tab:notation_table_discrete}, and \appref{appendix:finsler_geometry} gives a self-contained introduction to Finsler geometry. We refer to \citet{shen2001lectures} and \citet{ohtaComparisonFinslerGeometry2021} for additional details.

Let $\gM$ be an $m$-dimensional manifold. A \emph{Finsler metric} is a map $\FF: T\gM \to [0,\infty)$, smooth and positive away from the zero section, that is positively homogeneous, $\FF(x,t v) = t \FF(x,v)$ for $t > 0$, sub-additive in $v$, and strongly convex. It plays on each tangent space $T_x\gM$ the role that a norm plays on $\R^m$ and, exactly as in the Riemannian case, it defines the length of a curve and a distance $\dstF(x,y)$ as the infimum of the lengths of the curves joining $x$ to $y$. Riemannian metrics are the special case where $\FF(x,\cdot)$ is induced by an inner product. The essential difference is that homogeneity is only required for \emph{positive} scalars: in general $\FF(x,v) \neq \FF(x,-v)$, so that the distance $\dstF$ satisfies the triangle inequality but not symmetry,  since traveling from $x$ to $y$ need not cost as much as traveling back. The \emph{reverse metric} $\FF(x,-v)$ is again a Finsler metric, whose distance is $(x,y) \mapsto \dstF(y,x)$ (\cref{property:finsler_reverse}).

This work relies on a subclass of Finsler metrics and one canonical construction. \emph{Randers metrics} \citep{randers1941asymmetrical} are a natural generalization of Riemannian metrics. For a Riemannian metric $\alpha(x,v) = \sqrt{v^\top \rmA(x) v}$ and a vector field $b$ on $\gM$, the Randers metric is defined as
\begin{equation} \label{eq:randers_main}
  \FF(x,v) = \sqrt{v^\top \rmA(x) v} \ + \ b(x)^\top v, \qquad \lVert b(x) \rVert_{\rmA^{-1}} < 1.
\end{equation}
The constraint on $b$ ensures that $\FF$ is positive. Its indicatrix $\{v \in T_x\gM : \FF(x,v) = 1\}$ is an ellipsoid whose center is shifted in a direction determined by $b(x)$, so that all the directional information is carried by the single vector field $b$. We write $\tball_x = \{v \in T_x\gM : \FF(x,v) \leq 1\}$ for the unit ball of $\FF$ in the tangent space $T_x\gM$. The second moment of the uniform distribution on $\tball_x$ yields a Riemannian metric canonically attached to $\FF$: the \emph{Binet--Legendre metric} \citep{matveev2012binet},
\begin{equation} \label{eq:binet_legendre_main}
  \gBL^{ij}(x) = \frac{(m+2)}{\leb_x(\tball_x)} \int_{\tball_x} v^i v^j  \diff \leb_x(v),
\end{equation}
which is smooth whenever $\FF$ is smooth, and coincides with $\rmA$ when $\FF$ is Riemannian. The first moment of the same distribution is the \emph{centroid} of the unit ball,
\begin{equation} \label{eq:centroid_main}
  \rvc(x) = \frac{1}{\leb_x(\tball_x)} \int_{\tball_x} v \diff \leb_x(v) \in T_x\gM,
\end{equation}
which defines a vector field on $\gM$, the \emph{centroid field}. It vanishes whenever $\FF$ is reversible, since the unit ball is then symmetric about the origin, and thus measures the asymmetry of $\FF$ at first order; for Randers metrics, it is an explicit function of $\rmA$ and $b$ (\cref{prop:centroid_randers_metric}). The pair $(\rvc, \gBL)$, \ie the first two moments of the unit ball, will carry respectively the direction and the geometry recovered by our operators.

Unlike a Riemannian manifold, a Finsler manifold carries no canonical volume. In this work, we use the \emph{Busemann--Hausdorff measure} $\mBH$ (\cref{def:busemann_haussdorf_measure}). In a chart, with $\diff x$ the Lebesgue measure on $\gM$ and $\leb_x$ the Lebesgue measure on $T_x\gM$ in the coordinate basis $(\partial_i|_x)$, it reads
\begin{equation} \label{eq:busemann_hausdorff_main}
  \diff\mBH(x) = \sigBH(x) \diff x, \qquad \sigBH(x) = \frac{\vole(\eball^m)}{\leb_x(\tball_x)},
\end{equation}
where $\vole(\eball^m)$ is the volume of the Euclidean unit ball of $\R^m$. Although $\diff x$ and $\leb_x$ are chart-dependent, $\mBH$ is not. When $\FF$ is Riemannian, $\mBH$ coincides with the Riemannian volume. We assume that the data are $N$ \iid samples $X_1, \ldots, X_N$ from a probability measure $\mathbb{P}$ on $\gM$, and denote by $\rho$ its density with respect to $\mBH$, so that $\diff\mathbb{P}(x) = \rho(x) \diff\mBH(x)$.

\subsection{Our setting}

We make the following assumptions throughout:
\begin{itemize}[itemsep=0em, topsep=0em]
  \item[(A1)] The kernel $K: \R_+ \to \R_+$ is a smooth, positive, monotone  decreasing function with bounded derivative, and has a sub-exponential tail: there exist constants $C_K, \nu_K>0$ such that for any $r \in \R_+$, $K(r) \leq C_K e^{-\nu_K r}$.
  \item[(A2)] \label{assum:A2} The manifold $\gM$ is smooth, connected, compact and \emph{without boundary}\footnote{This assumption is standard in the study of convergence of graph Laplacians. When $\gM$ has a boundary, the analysis still holds for points ``far enough'' from it \citep{coifman_diffusion_2006}.}, and the Finsler metric $\FF$ is smooth and positive on the slit tangent bundle $T\gM \setminus \{0\}$. Moreover, the sampling density $\rho$ is smooth and bounded away from zero by a constant $\rho_{\min} > 0$.
\end{itemize}
Given $K$ satisfying (A1), we define the kernel $W: \MM \times \MM \to \R_+$ and its out/in-degrees as:
\begin{equation*}
    W(x,y) = \frac{1}{\eps^m} K\left(\frac{\dstF(x,y)}{\eps}\right), \quad \dout(x) = \int_\gM W(x,y) \diff\mathbb{P}(y), \quad \din(x) = \int_\gM W(y,x) \diff\mathbb{P}(y),
\end{equation*}
where $\eps$ is the kernel bandwidth. We further assume that the degrees are bounded below: for any $x \in \gM$, $\dout(x), \din(x) > \dmin > 0$. This is a standard assumption in the literature \citep{hein2007graph}, and it is satisfied when the kernel is sub-exponential and $\rho$ is positive. We study the kernel through its associated forward and backward transport operators $\gG_\FF$ and $\gG'_\FF$. For any $f:\gM\rightarrow \R$ smooth enough, these are defined as:
\begin{equation*}
  \label{eq:transport_operator}
    \gG_\FF[f](x) = \int_\gM W(x,y)f(y) \diff\mBH(y), \quad \gG'_\FF[f](x) = \int_\gM W(y,x) f(y) \diff\mBH(y).
\end{equation*}
The backward operator is the transport operator of the reversed Finsler metric, and is the $L^2(\gM,\mBH)$-adjoint of $\gG_\FF$, so that $\gG_\FF$ is self-adjoint precisely when $\FF$ is reversible (\cref{prop:adjoint_reverse}, proved in \appref{appendix:proofs:prop:adjoint_reverse}).

In practice, we only observe samples from $\mathbb{P}$, whose density $\rho$ is unknown. To reduce its influence, we use a $\theta$-normalization of the kernel $W$, similar to \citet{coifman_diffusion_2006}:
\begin{equation*}
  W^{(\theta)}(x,y) = \frac{W(x,y)}{\qe(x)^\theta \qe(y)^\theta}, \quad \gGt [f](x) = \int_\gM W^{(\theta)}(x,y) f(y) \diff\mBH(y),
\end{equation*}
where $\qe(x) = (\din(x) + \dout(x))/2$ is a normalization factor, and $\theta \in [0,1]$. The backward operator is defined using $W^{(\theta)}(y,x)$ instead. In general, the forward and backward operators differ; \cref{sec:main_results} exploits this difference to split the operator into a symmetric and an antisymmetric part.

\section{Continuum limits: separating geometry from direction}
\label{sec:main_results}
Our starting point is a moment expansion of the transport operator $\gG_\FF$, the Finslerian counterpart of the expansions used to derive continuum limits of graph Laplacians \citep{coifman_diffusion_2006, singer_from_2006}. The overall procedure is illustrated in \cref{fig:Finsler_setting}.
\begin{theorem}\label{thm:moment_expansion}
  For any $f:\gM\rightarrow \R$ smooth enough, we have
  \begin{equation*}
    \gG_\FF [f] (x) = m_0 f (x) + \eps \gA_1[f](x) + \frac{\eps^2}{2} \gA_2[f](x) + \bigO(\eps^3),
  \end{equation*}
  where $\gA_1$ and $\gA_2$ are first- and second-order differential operators, whose coefficients are explicitly given in \appref{appendix:proofs:thm:moment_expansion}. The leading term $m_0$ is the zeroth moment of the kernel $K$, and does not depend on $x$.
\end{theorem}

We now split $\gGt$ into its symmetric and antisymmetric parts $\gGts$ and $\gGta$, which is where the geometry and the direction separate:
\begin{equation*}
  \gGt = \gGts + \gGta, \quad \text{where} \quad \gGts = (\gGt + \gGt')/2  \quad \text{and} \quad \gGta = (\gGt - \gGt')/2.
\end{equation*}
Define the associated transport operators $\Pts$ and $\Pta$ as
\begin{equation*}
  \Ptb [f] = \frac{1}{\gGts[\rho]}\left( \gGtb[\rho f] - f \gGtb[\rho] \right),
\end{equation*}
for $\bullet \in \{s,a\}$. The following theorem gives their limits as $\eps \to 0$, after rescaling by $\eps^{-2}$ and $\eps^{-1}$, respectively.
\begin{theorem} \label{thm:sym_limit} The operators $\Ptb$ admit limiting operators $\gL^{\bullet}$ as $\eps \to 0$:
  \begin{align}
    \gL^s f &:= \lim_{\eps \to 0} \frac{1}{\eps^2} \Pts[f] = \frac{c_2}{\rho^{2(1-\theta)}} \operatorname{div}_{BH} \left( \rho^{2(1-\theta)} \nabla_\text{BL} f \right), \\
    \gL^a f &:= \lim_{\eps \to 0} \frac{1}{\eps} \Pta[f] = \tilde{m}_1^i \partial_i f = c_1 \, \rvc^i \partial_i f, \label{eq:antisym_operator}
  \end{align}
  where $\tilde{m}_1(x)$ is the normalized first moment of the kernel $K$ over the unit ball $\tball_x$ (\cref{def:kernel_moments}), $\rvc$ is the centroid field \cref{eq:centroid_main}, $c_1$ and $c_2$ are constants depending on the kernel and the dimension of the manifold, $\nabla_\text{BL}$ is the gradient with respect to the Binet--Legendre metric $\gBL$ of $\FF$, and $\operatorname{div}_{BH}$ is the divergence with respect to the Busemann--Hausdorff measure $\mBH$. In particular, for any Finsler metric $\FF$, $\gL^s$ is an elliptic operator, self-adjoint on $L^2(\gM, \rho^{2(1-\theta)} \mBH)$.
\end{theorem}
This result is proved in \appref{appendix:proofs:prop:sym_limit}, and the constants are given in \cref{props:moment_expression_unitball}; for instance, $c_1 = (m+1)\mu_1/(m \mu_0)$ and $c_2 = \mu_2 / (2m\mu_0)$ with $\mu_n = \int_0^\infty K(r) r^{m+n-1} \diff r$, see \cref{appendix:kernel_moments} for more details. The symmetric limit is thus a weighted Laplacian, whose metric is the Binet--Legendre metric of $\FF$ and whose reference measure is $\rho^{2(1-\theta)} \mBH$. When $\theta = 1$, it does not depend on the sampling density $\rho$, which is a desirable property for the embedding algorithm. In contrast, $\theta = 0$ keeps the whole drift induced by $\rho$, and $\theta = 1/2$ halves it (see also \cref{appendix:proofs:prop:moment_expression_unitball}). The antisymmetric limit is the derivative along the centroid field: the first moment of the unit ball carries the direction, and the second moment the geometry. Note also that the finite-$\eps$ operator $\Pts$ is always diagonalizable with real eigenvalues.

The divergence $\operatorname{div}_{BH}$ is taken with respect to the Busemann--Hausdorff measure, and not with respect to the Riemannian volume of $\gBL$. The two coincide up to a constant factor for \emph{Berwald} metrics, which are those whose connection depends on the position only and not on the direction (\cref{def:berwald_metric}). In that case, the symmetric limit simplifies as follows, proven in \appref{appendix:proofs:cor:berwald_limit}.
\begin{corollary}[Berwald and Riemannian metrics] \label{cor:berwald_limit}
  If $\FF$ is a Berwald metric, then
  \begin{equation} \label{eq:berwald_limit}
    \gL^s f = c_2 \left( \Delta_\text{BL} f + 2 (1-\theta) \langle \nabla \log \rho, \nabla f \rangle_\text{BL} \right),
  \end{equation}
  where $\Delta_\text{BL}$ is the Laplace--Beltrami operator of $\gBL$ and $\langle \nabla u, \nabla f \rangle_\text{BL} = \gBL^{ij} \partial_i u \partial_j f$. If moreover $\FF$ is Riemannian, \ie $\FF(x,v) = \sqrt{v^\top \rmA(x) v}$, then $\gBL = \rmA$, $\gL^a = 0$, and \cref{eq:berwald_limit} is the $\theta$-normalized diffusion generator of \citet{coifman_diffusion_2006}.
\end{corollary}

\section{Discrete operators and convergence guarantees} \label{sec:discrete_operators}
In the discrete setting, where the manifold is only accessed through $N$ points $X_1, \ldots, X_N$ sampled \iid from $\mathbb{P}$, the construction follows the same principle as the continuous one. From the empirical kernel matrix $\rmW_N(i,j) = W(X_i,X_j)$, we define the out- and in-degrees $\rmD_N(i,i) = \sum_j \rmW_N(i,j)$ and $\rmD'_N(i,i) = \sum_j \rmW_N(j,i)$, and their half-sum $\rmQ_N = (\rmD_N + \rmD'_N)/2$. The $\theta$-normalized matrix and its symmetric and antisymmetric parts are then
\begin{equation*}
  \rmW_N^{(\theta)} = \rmQ_N^{-\theta} \rmW_N \rmQ_N^{-\theta}, \qquad
  \rmW_N^{(\theta,\bullet)} = \tfrac{1}{2}\big(\rmW_N^{(\theta)} \pm (\rmW_N^{(\theta)})^\top\big),
\end{equation*}
with the sign $+$ for $\bullet = s$ and $-$ for $\bullet = a$. We write $\rmD_N^{(\theta,\bullet)}$ for the diagonal matrix of row sums of $\rmW_N^{(\theta,\bullet)}$, which discretizes $\gGtb[\rho]$. Finally, for $\rvf = (f(X_1),\dots,f(X_N))^\top \in \R^N$, we define the discrete operators:
\begin{equation}
  \label{eq:discrete_operators}
  \rmP^{(\theta,\bullet)}[f] = \big(\rmD_N^{(\theta,s)}\big)^{-1}\Big(\rmW_N^{(\theta,\bullet)} \rvf - \rmD_N^{(\theta,\bullet)} \rvf \Big), \qquad \text{with} \qquad
  \rmL^s[f] = \frac{\rmP^{(\theta,s)}[f]}{\eps^2}, \quad \rmL^a[f] = \frac{\rmP^{(\theta,a)}[f]}{\eps}.
\end{equation}
Both parts are normalized by the same $\rmD_N^{(\theta,s)}$, each is recentered by its own degree $\rmD_N^{(\theta,\bullet)}$, and each is rescaled by the order at which its limit appears in \cref{thm:sym_limit}.

\begin{theorem}\label{thm:convergence_discrete_op}
  Let $f \in \gC^3(\gM)$ and let the bandwidth $\eps = \eps(N)$ vary with the number of samples in such a way that $\eps(N) \to 0$ and $N \eps(N)^{m+2} / \log N \to \infty$ as $N \to \infty$. Then the discrete operators converge to the continuous ones, uniformly and almost surely: for $\bullet \in \{s,a\}$,
  \begin{equation*}
    \max_{1 \leq i \leq N} \big| \rmL^{\bullet}[f](X_i) - \gL^{\bullet} f(X_i) \big| \xrightarrow[N \to \infty]{a.s.} 0 .
  \end{equation*}
\end{theorem}
This result is proved in \appref{appendix:proofs:thm:convergence_discrete_op}, where the convergence rate is given explicitly.

\section{Moment-Randers metric and directed embedding} \label{sec:randers_approximation}

\cref{thm:sym_limit} shows that the operators $\gL^s$ and $\gL^a$ depend on the Finsler metric $\FF$ only through its Binet--Legendre metric, its Busemann--Hausdorff measure, and its centroid field. This motivates the definition of a Randers metric that has the same Binet--Legendre metric and centroid field as $\FF$, and that can be estimated from the graph operators of \cref{sec:discrete_operators}. Results stated in this section are proved in \appref{appendix:randers_approximation}, where additional details are also provided.

\subsection{Moment-Randers approximation} \label{sec:moment_randers}
\begin{definition} (Moment-Randers metric) \label{def:moment_randers}
  Let $\FF$ be a Finsler metric on $\gM$ with centroid field $\rvc$ and Binet--Legendre metric $\gBL$. Set $\rmH^{-1} = \gBL^{-1} - (m+2) \rvc\rvc^\top$, and $\lambda = 1 - \lVert \rvc \rVert_{\rmH}^2$. If $\lVert \rvc(x) \rVert_{\rmH} < 1$ for all $x \in \gM$, the \emph{moment-Randers metric} of $\FF$ is the Randers metric \cref{eq:randers_main} with data $(\rmA_R, b_R)$:
  \begin{equation*}
    \FF_R(x,v) = \sqrt{v^\top \rmA_R(x) v} + b_R(x)^\top v, \qquad
    \rmA_R = \frac{\lambda \rmH + \rmH \rvc \rvc^\top \rmH}{\lambda^2}, \qquad
    b_R = - \frac{\rmH \rvc}{\lambda}.
  \end{equation*}
\end{definition}
Its drift satisfies $\lVert b_R \rVert_{\rmA_R^{-1}} = \lVert \rvc \rVert_{\rmH}$, so the admissibility constraint of \cref{eq:randers_main} is exactly the hypothesis above. $\FF_R$ has the same Binet--Legendre metric and centroid field as $\FF$, but can have a different Busemann--Hausdorff measure\footnote{As detailed in \appref{appendix:randers_approximation}, the symmetric limits of both metrics have the same principal components.}, and if $\FF$ is a Randers metric, then $\FF_R = \FF$ as this is just the navigation map formulation (\cref{appendix:randers_approximation}). This means that $\gL^a$ is identical for both metrics, that $\gL^s$ only differs through the reference measure, and that the moment-Randers metric is a moment-matching Randers approximation of $\FF$. The construction of $\rmH$ relies on $\gBL$ and $\rvc$. We relate the well-definedness of $\rmH$ and of the moment-Randers metric to the norm of $\rvc$ in the Binet--Legendre metric, which, as shown below, is accessible from the data.

\begin{restatable}{proposition}{HvsGBLProp} \label{prop:H_vs_gBL}
  Let $\FF$ be a Finsler metric and let $x \in \gM$ be such that $\lVert \rvc(x) \rVert_{\gBL}^2 < (m+2)^{-1}$. Then the matrix $\rmH(x)$ is well defined and positive definite, and
  \begin{equation*}
    \lVert \rvc(x) \rVert_{\gBL}^2 = \frac{\lVert \rvc(x) \rVert_{\rmH}^2}{1 + (m+2) \lVert \rvc(x) \rVert_{\rmH}^2}, \qquad \text{hence} \qquad \lVert \rvc(x) \rVert_{\rmH} < 1 \iff \lVert \rvc(x) \rVert_{\gBL}^2 < (m+3)^{-1}.
  \end{equation*}
\end{restatable}

\subsection{Estimation through an embedding} \label{sec:randers_estimation}
The quantities $\rvc$ and $\gBL$ cannot be observed directly. We therefore estimate them, through an embedding, from the operators $\gL^s$ and $\gL^a$, which are themselves approximated by the graph operators $\rmL^s$ and $\rmL^a$ (\cref{thm:convergence_discrete_op}). Consider an embedding $\Psi(x) = (\psi_1(x), \ldots, \psi_\ell(x))$ and its Jacobian at $x$, $J(x)\in\R^{\ell\times m}$, assumed of full column rank $m$. By \cref{eq:antisym_operator}, the push-forward of $\rvc$ by $\Psi$ is recovered from the antisymmetric operator as
\begin{align*}
  V(x) :=\gL^a[\Psi](x) = c_1 J(x) \rvc(x).
\end{align*}
\begin{proposition} \label{prop:carre_du_champ}
  The carré du champ operator $\Gamma$ associated with $\gL^s$,
  \begin{align*}
    \Gamma(x)^{kl} = \frac{1}{2} \left(
      \gL^s[\psi_k \psi_l](x) - \psi_k(x) \gL^s[\psi_l](x) - \psi_l(x) \gL^s[\psi_k](x)\right),
  \end{align*}
  recovers the Binet--Legendre metric up to the embedding: $\Gamma(x) = c_2 J(x) \gBL(x)^{-1} J(x)^\top$, where $c_2$ is the constant of \cref{thm:sym_limit}.
\end{proposition}
Combining these results, the moment-Randers metric of $\FF$ is well defined at $x$ if and only if $\lVert \rvc(x) \rVert_{\gBL}^2 = \frac{c_2}{c_1^2} \lVert V(x) \rVert_{\Gamma^+}^2 < (m+3)^{-1}$, where $\Gamma^+$ is the pseudo-inverse of $\Gamma$. This criterion only involves $\gL^s$ and $\gL^a$, and allows us to define the embedded moment-Randers metric $\hat{\FF}_R$.
\begin{definition} \label{def:embedded_moment_randers}
  The embedded moment-Randers metric $\hat{\FF}_R$ is defined as the moment-Randers metric (\cref{def:moment_randers}) associated with the estimated Binet--Legendre metric $\hat{g}_{BL}$ and the estimated centroid of the unit tangent ball $\hat{\rvc}$. Let
  \begin{align*}
    & \hat{g}_{BL}(x)^{-1} = \Gamma(x),  &&\hat{\rvc}(x) = \frac{\sqrt{c_2}}{c_1} V(x), \\
    & \hat{\rmH}^{-1}(x) = \Gamma(x) - (m+2) \frac{c_2}{c_1^2} V(x) V(x)^\top, &&\hat{\lambda}(x) = 1 - \lVert \hat{\rvc}(x) \rVert_{\hat{\rmH}}^2,
  \end{align*}
  so that the embedded moment-Randers metric is given by
  \begin{equation*}
    \hat{\FF}_R(x,v) = \frac{1}{\hat{\lambda}(x)} \left(
      \sqrt{\hat{\lambda}(x) \lVert v \rVert_{\hat{\rmH}}^2 + \langle \hat{\rvc}(x), v \rangle_{\hat{\rmH}}^2} - \langle \hat{\rvc}(x), v \rangle_{\hat{\rmH}}
    \right).
  \end{equation*}
\end{definition}

\begin{restatable}{proposition}{EmpiricalRandersProp} \label{prop:embedded_randers}
  If $\lVert \hat{\rvc}(x) \rVert_{\hat{g}_{BL}} < 1 / \sqrt{m+3}$ and $J$ is of full column rank, then $\hat{\FF}_R$ is related to the original moment-Randers metric $\FF_R$ by
  \begin{align*}
    \hat{\FF}_R(x, J(x)v) = \frac{1}{\sqrt{c_2}} \FF_R(x, v), \quad \forall v \in T_x\gM.
  \end{align*}
  In particular, $\hat{\FF}_R$ is a valid Randers metric on $J(T_x\gM)$.
\end{restatable}

\subsection{Embedding algorithm} \label{sec:embedding_algo}
\begin{algorithm}[t]\small
  \caption{Finsler directed embedding}
  \label{alg:embedding}
  \begin{algorithmic}[1]
    \REQUIRE Directed distances $\dstF(X_i,X_j)$ between samples, or a weighted directed graph with adjacency matrix $\rmW_N$; kernel $K$, bandwidth $\eps$, normalization $\theta \in [0,1]$, intrinsic dimension $m$, embedding dimension $\ell$.
    \ENSURE Embedded samples $Y_i \in \R^\ell$, drifts $\rmV_i$, and Randers metrics $\hat{\FF}_R(Y_i,\cdot)$.
		\vspace{0.5em}
		\hrule
		\vspace{0.5em}
    \STATE $\rmW_N(i,j) \gets \eps^{-m} K\big(\dstF(X_i,X_j)/\eps\big)$ \hfill \COMMENT{skipped if $\rmW_N$ is given}
    \STATE $\rmQ_N \gets (\rmD_N + \rmD'_N)/2$, \quad $\rmW_N^{(\theta)} \gets \rmQ_N^{-\theta} \rmW_N \rmQ_N^{-\theta}$
    \STATE $\rmW_N^{(\theta,s)}, \rmW_N^{(\theta,a)} \gets$ symmetric and antisymmetric parts of $\rmW_N^{(\theta)}$
    \STATE Build $\rmL^s$ and $\rmL^a$ from \cref{eq:discrete_operators}
    \STATE $\rmY \gets (\Psi(X_1), \ldots, \Psi(X_N))^\top \in \R^{N \times \ell}$ \hfill \COMMENT{\eg leading eigenvectors of $\rmL^s$}
    \FOR{$i = 1, \ldots, N$}
      \STATE $\rmV_i \gets (\rmL^a \rmY)_i$, \quad $\Gamma_i^{kl} \gets \frac{1}{2}\big(\rmL^s[Y^k Y^l]_i - Y_i^k \rmL^s[Y^l]_i - Y_i^l \rmL^s[Y^k]_i\big)$
      \STATE $\hat{\rvc}_i \gets \frac{\sqrt{c_2}}{c_1} \rmV_i$, \quad $\hat{\rmH}_i^{-1} \gets \Gamma_i - (m+2) \hat{\rvc}_i \hat{\rvc}_i^\top$
      \IF{$\lVert \hat{\rvc}_i \rVert_{\Gamma_i^+}^2 < (m+3)^{-1}$}
        \STATE Set $\hat{\FF}_R(Y_i, \cdot)$ as in \cref{def:embedded_moment_randers}
      \ENDIF
    \ENDFOR
  \end{algorithmic}
\end{algorithm}
\cref{alg:embedding} summarizes the procedure. The embedding $\Psi$ can be any classical embedding, \eg Isomap \citep{tenenbaum2000global} or the leading eigenvectors of $\rmL^s$, in which case $J$ has full column rank for $\ell$ large enough. $\rmV_i$ and $\Gamma_i$ are the discrete counterparts of $V$ and $\Gamma$, obtained by replacing $\gL^\bullet$ with $\rmL^\bullet$; here $Y^k \in \R^N$ is the $k$-th column of $\rmY$. When the data is given as a weighted directed graph (\cref{sec:experiments}), $\rmW_N$ is its adjacency matrix.

\section{Experiments} \label{sec:experiments}

\textbf{Swiss roll dataset.}
We first consider a Randers metric on the Swiss roll, a 2D manifold embedded in $\R^3$, whose Riemannian part is supported on the manifold and whose drift $b$ is tangent to it (\appref{appendix:experiments:swiss_roll}). We sample $N=2000$ points from the manifold and construct a $k$-NN graph with $k=10$ neighbors. With Isomap as $\Psi$, \cref{fig:swiss_roll} shows that we recover both the geometry of the manifold and the directionality of the edges. As $\rvc$ and $\gBL$ are known in closed form for a Randers metric (\cref{prop:centroid_randers_metric}), \cref{tab:swiss_roll} evaluates the estimated strength of the asymmetry for varying $\lVert b \rVert_{\rmA^{-1}}$ and $N$, on a radius graph and with $J_i$ estimated by finite differences (\appref{appendix:experiments:swiss_roll}). The cosine between the recovered and true drifts is $1.00$ in every configuration, and the relative error on $\lVert \hat{\rvc} \rVert_{\hat{g}_{BL}}^2$ stays between $0.07$ and $0.17$. It decreases with $N$ at every drift strength, although slowly. The bandwidth $\eps$ is the median Finsler distance from a sample to its $10$-th nearest neighbor, and the graph construction is detailed in \appref{appendix:experiments:swiss_roll}. The rejections at large $\lVert b \rVert_{\rmA^{-1}}$ are expected. All samples are admissible for the true metric, since $\FF_R = \FF$ for a Randers metric, but $\lVert \rvc \rVert_{\gBL}^2$ approaches the threshold $(m+3)^{-1}$ as the drift grows, so that small estimation errors suffice to cross it.

\begin{table}[t]
\small
\centering
\caption{\textbf{Swiss roll}. Relative error on $\lVert \hat{\rvc} \rVert_{\hat{g}_{BL}}^2$ and cosine between $V_i$ and $c_1 J_i \rvc(X_i)$ (medians over samples), and admissible fraction; medians over $3$ seeds. Admissible means $\lVert \hat{\rvc} \rVert_{\hat{g}_{BL}}^2 < (m+3)^{-1}$.}
\label{tab:swiss_roll}
\begin{tabular}{cccc ccc ccc}
    \toprule
     & \multicolumn{3}{c}{Rel. error $\|\hat{\mathbf{c}}\|^2_{\hat{g}_{BL}}$} & \multicolumn{3}{c}{$\cos(V_i, J_i\mathbf{c})$} & \multicolumn{3}{c}{Admissible} \\
    \cmidrule(lr){2-4} \cmidrule(lr){5-7} \cmidrule(lr){8-10}
    $\|b\|_{\mathbf{A}^{-1}}$ & $N = 1000$ & 2000 & 4000 & $N = 1000$ & 2000 & 4000 & $N = 1000$ & 2000 & 4000 \\
    \midrule
    0.1 & 0.17 & 0.14 & 0.13 & 1.00 & 1.00 & 1.00 & 1.00 & 1.00 & 1.00 \\
    0.3 & 0.15 & 0.13 & 0.12 & 1.00 & 1.00 & 1.00 & 1.00 & 1.00 & 1.00 \\
    0.5 & 0.13 & 0.11 & 0.10 & 1.00 & 1.00 & 1.00 & 0.98 & 0.99 & 0.99 \\
    0.7 & 0.11 & 0.10 & 0.09 & 1.00 & 1.00 & 1.00 & 0.89 & 0.91 & 0.91 \\
    0.9 & 0.09 & 0.08 & 0.07 & 1.00 & 1.00 & 1.00 & 0.46 & 0.47 & 0.50 \\
    \bottomrule
\end{tabular}
\end{table}

\textbf{Approximating Matsumoto metrics.}
The method relies on a Randers approximation of the Finsler metric. We investigate whether the directional embeddings recovered from a non-Randers metric still capture the essential geometric properties of the data. We consider the Matsumoto metric \citep{matsumoto_theory_1992} induced by movement on a height map $h$ over $[-3.2,3.2]\times[-2.4,2.4]$ \citep{matsumoto1989slope, chansri2018geometry}, $\FF = \alpha^2/(\alpha - \beta)$, where $\alpha$ is the Riemannian metric of the graph of $h$ and $\beta = \diff h$ (\appref{appendix:experiments:matsumoto}). \cref{fig:matsumoto} shows the height map and the drift recovered with Isomap embeddings. The drift is aligned with the gradient of $h$, the defining feature of the Matsumoto metric, which suggests that our method captures the main directional structure of non-Randers metrics.

\textbf{Embedding directed graphs.} Finally, we apply our method to general directed graphs, without an underlying manifold. We generate directed graphs from a Directed Stochastic Block Model (DiSBM, \appref{appendix:experiments:disbm}). \cref{fig:disbm_graph} shows the leading eigenvectors of $\rmL^s$ for $N=1000$ nodes and $15$ communities: nodes cluster by community, and the drift follows the dominant, here backward, cyclic flow.

\begin{figure*}[t]
    \centering

    \begin{subfigure}[t]{0.31\textwidth}
        \centering
        \includegraphics[width=0.8\textwidth]{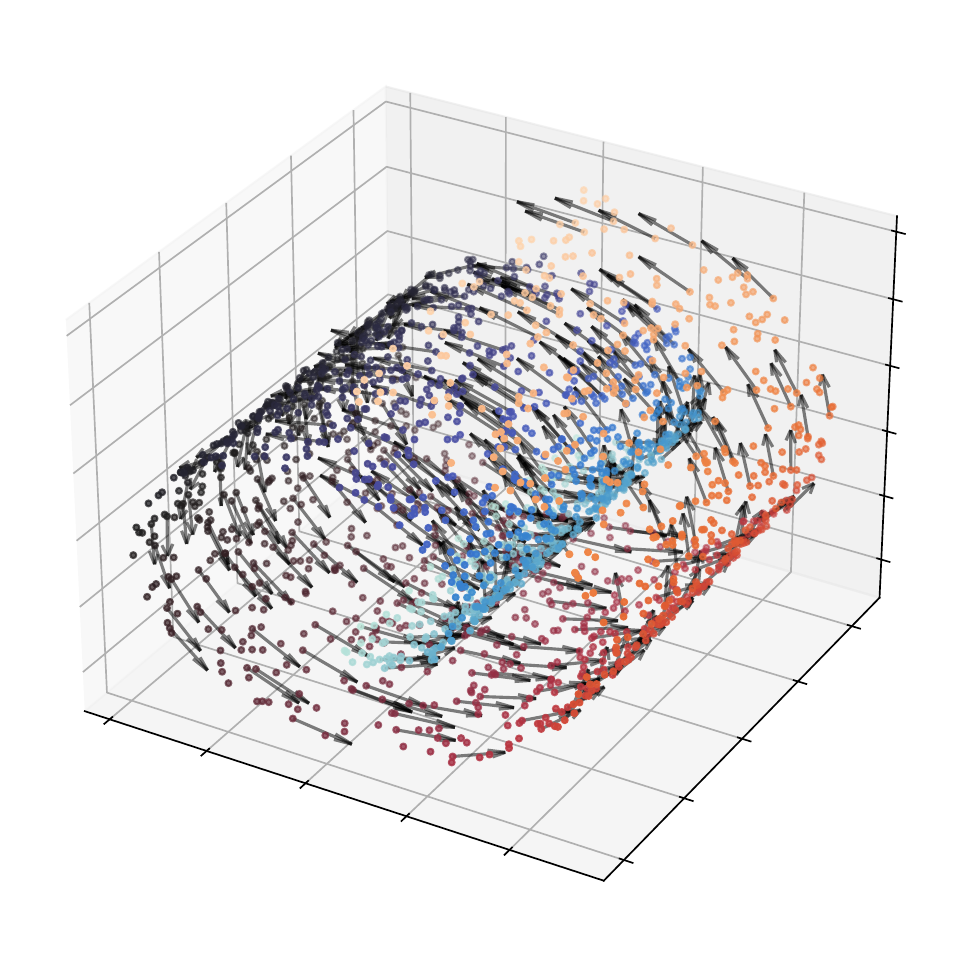}
    \end{subfigure}
    \quad
    \begin{subfigure}[t]{0.31\textwidth}
        \centering
        \includegraphics[width=0.8\textwidth]{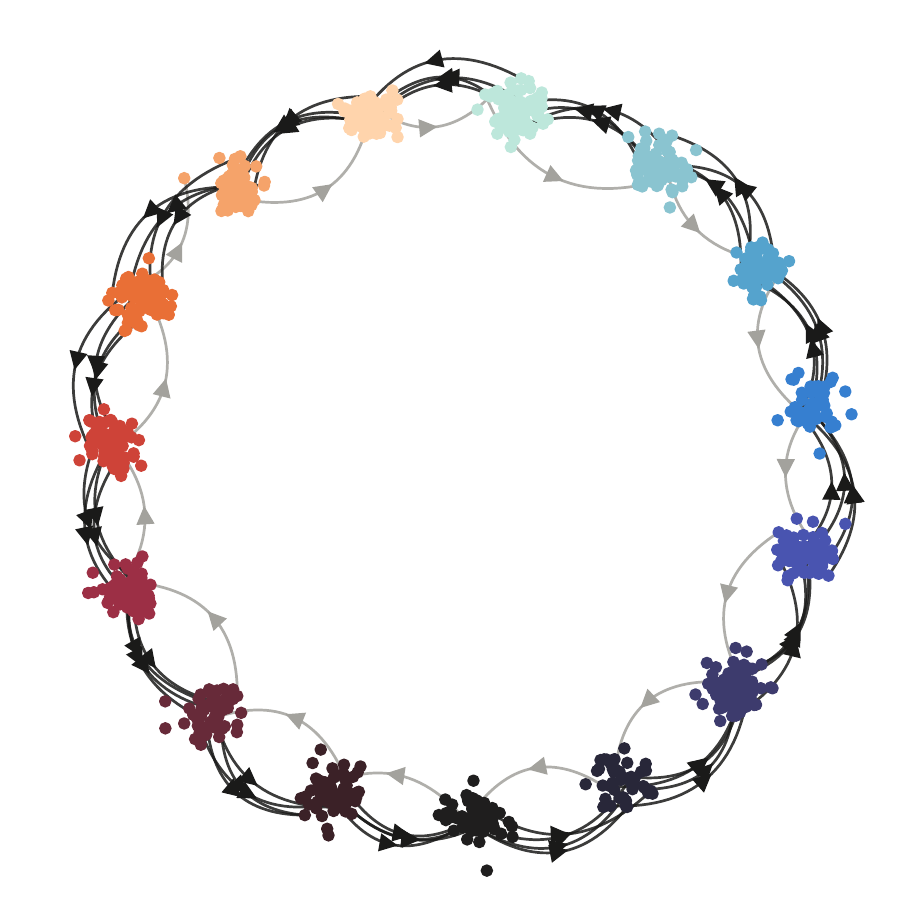}
    \end{subfigure}
    \quad
    \begin{subfigure}[t]{0.31\textwidth}
        \centering
        \includegraphics[width=0.8\textwidth]{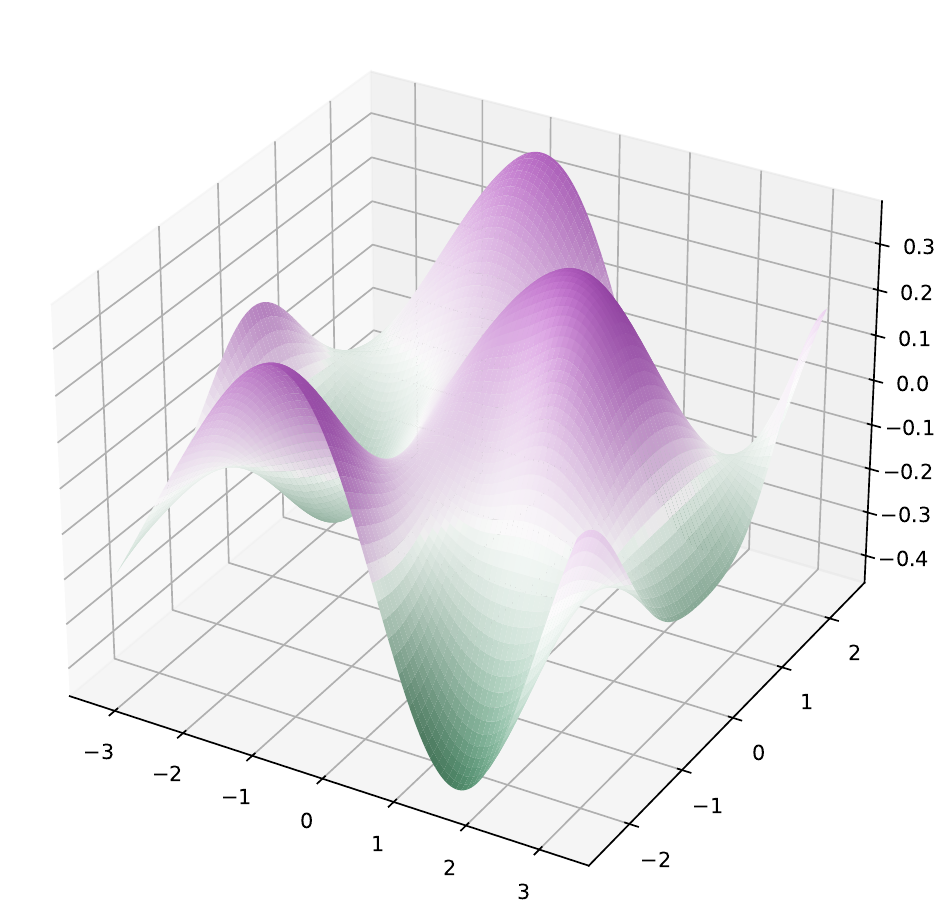}

    \end{subfigure}

    \vspace{0.25em}

    \begin{subfigure}[t]{0.31\textwidth}
        \centering
        \includegraphics[width=0.8\textwidth]{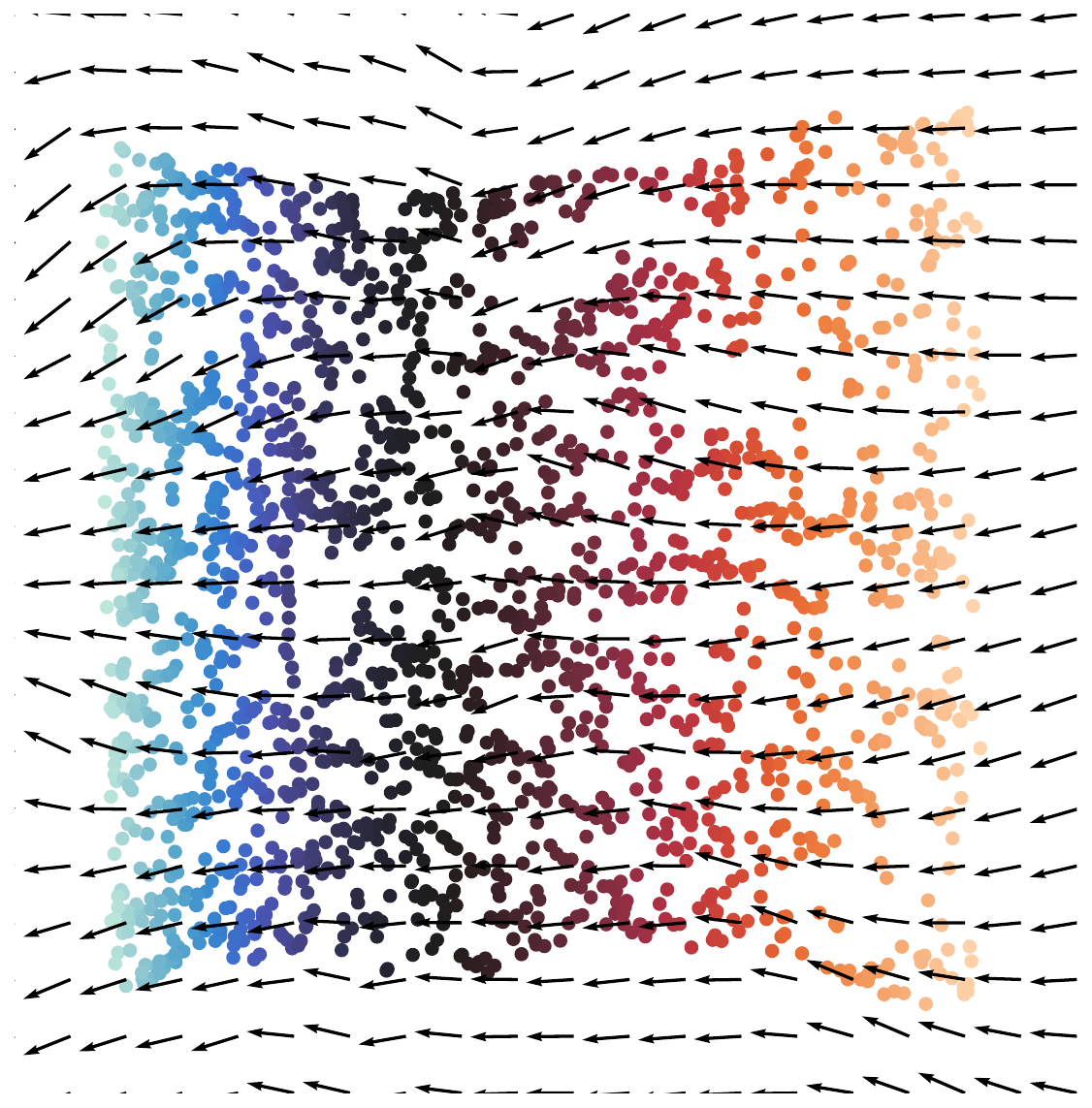}
        \caption{3D Swiss Roll Randers manifold.}
        \label{fig:swiss_roll}
    \end{subfigure}
    \quad
    \begin{subfigure}[t]{0.31\textwidth}
      \centering
      \includegraphics[width=0.8\textwidth]{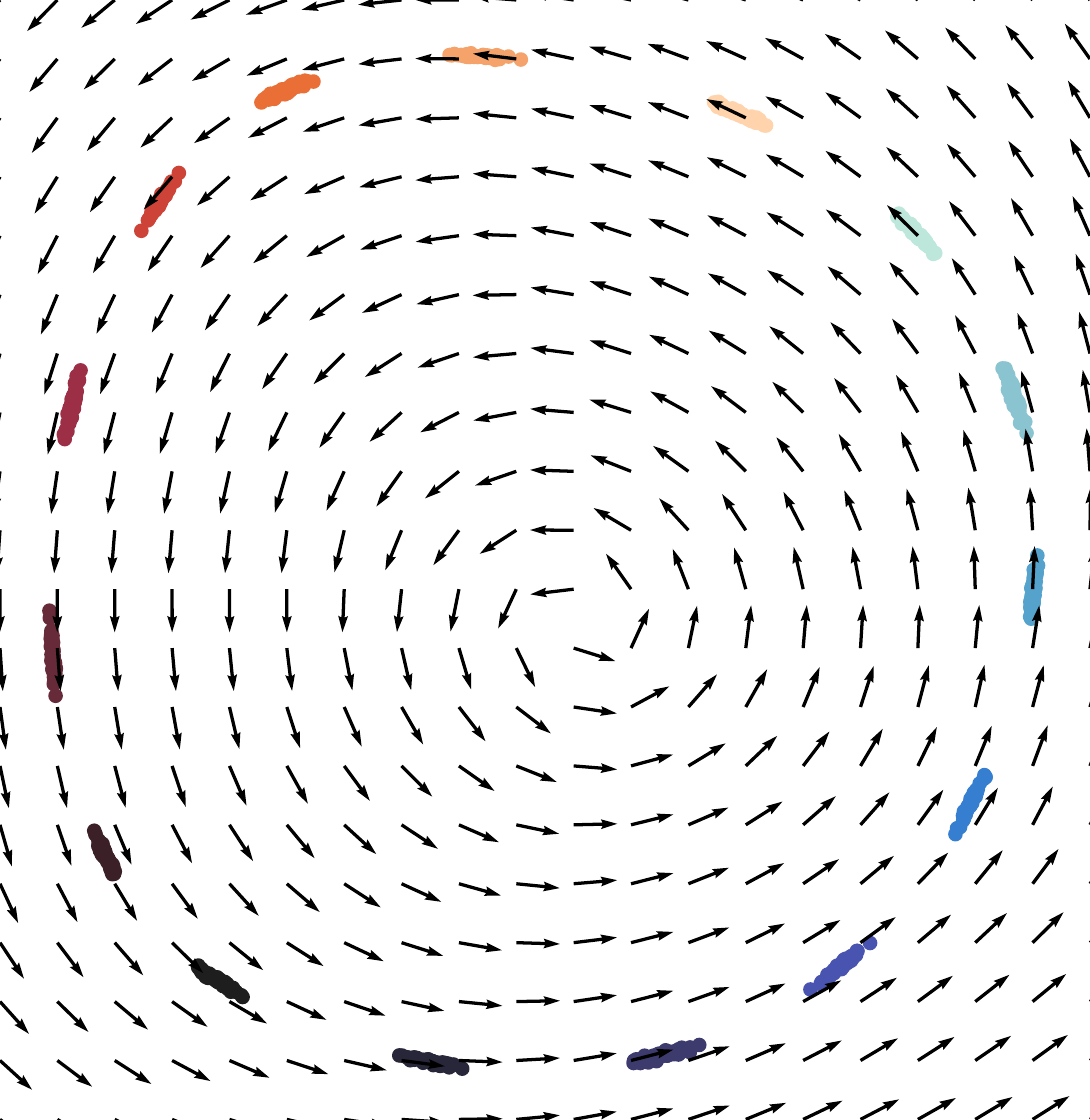}
        \caption{DiSBM graph.}
        \label{fig:disbm_graph}
    \end{subfigure}
    \quad
    \begin{subfigure}[t]{0.31\textwidth}
        \centering
        \includegraphics[width=0.8\textwidth]{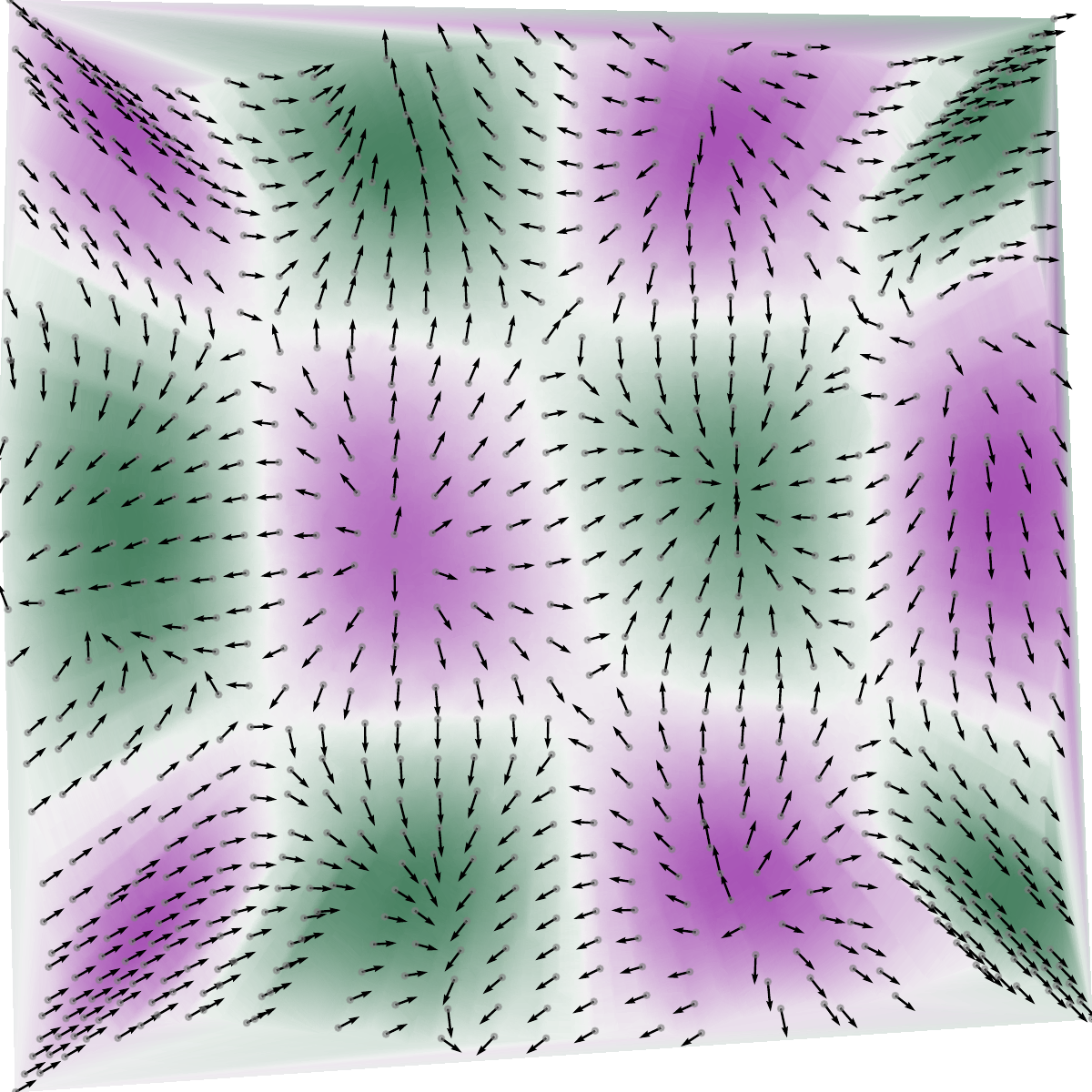}
        \caption{Matsumoto manifold.}
        \label{fig:matsumoto}
    \end{subfigure}

    \caption{\textbf{Directional embeddings.}
    Recovered directional embeddings on the three datasets. The top row shows the input data, and the bottom row shows the corresponding recovered embeddings and drifts.}
    \label{fig:directional_embeddings}
\end{figure*}

\section{Conclusion and Perspectives} \label{sec:conclusion}
We showed that Finsler geometry plays for directed graphs the role that Riemannian geometry plays for undirected ones. The symmetric and antisymmetric parts of the graph operators recover, respectively, the Binet--Legendre metric and the centroid field of the unit ball, which yields a moment-Randers approximation of the underlying metric and a directed embedding of the data. Experiments on Randers and Matsumoto manifolds and on directed stochastic block models show that both the geometry and the directionality are recovered. We hope that this work will open up new avenues for applied Finsler geometry in graph analysis. Future work could include the study of the spectral properties of the operators, the extension of the analysis to `standard' directed graph operators such as the flow-based Laplacian \citep{chung_laplacians_2005}, and the development of Finsler-based methods for directed graph analysis.

\subsection*{AI use statement}

In this work, we used generative AI tools to assist in the writing of proofs, to design or provide feedback on research methodology or experiments, and to help with translation.
We have not used generative AI tools to generate synthetic datasets, help develop theoretical models or conceptual frameworks, formulate mathematical claims, provide critical ingredients for proving mathematical claims, propose or refine hypotheses, implement methods, clean and reformat datasets, support qualitative and thematic data analysis, or interpret results. Additionally, we used generative AI tools to search for information and identify relevant literature. We have reviewed all AI-assisted work. We take responsibility for the final content of this work, including text, claims or artifacts produced with the aid of generative AI.

\subsection*{Ethics statement}

This work proposes a novel method for representation learning. As such, there are many potential societal consequences of our work, none of which we feel must be specifically highlighted here.

\subsection*{Reproducibility statement}

The code used to produce the results in this paper is available publicly at \url{https://github.com/Gwendal-Debaussart/Finsler-Embedding} and is licensed under the MIT License. All datasets used in this work are synthetic and can be regenerated with the released code. The theoretical results in this work are accompanied by complete proofs in the appendix, and all assumptions are clearly stated. A primer is provided in the appendix to introduce the reader to the relevant concepts of Finsler geometry, and to help them check the definitions, notations and proofs used in the main text.

\subsubsection*{Author Contributions}
\emph{This subsection follows the CRediT taxonomy.}

\emph{Gwendal Debaussart-Joniec:} Conceptualization, Formal Analysis, Investigation, Methodology, Validation, Visualization, Writing -- original draft, Writing -- review \& editing.

\emph{Théau Blanchard:} Conceptualization, Formal Analysis, Investigation, Methodology, Software, Validation, Visualization, Writing -- original draft, Writing -- review \& editing.

\emph{Argyris Kalogeratos:} Supervision, Writing -- review \& editing.

\subsubsection*{Acknowledgments}
The authors would like to deeply thank Jean-Marie Mirebeau for the discussions and feedback on the theoretical aspects of this work. Gwendal Debaussart-Joniec, and Argyris Kalogeratos acknowledge the support of the Industrial Analytics and Machine Learning (IdAML) Chair hosted at ENS Paris-Saclay, Université Paris-Saclay. Théau Blanchard is funded by GEHealthcare.

\bibliographystyle{iclr2027_conference}
\bibliography{iclr2027_conference}

\newpage
\appendix

\etocsetlocaltop.toc{part}          %
\etocdoesnotcheckemptiness
\etocsettocstyle{\section*{Appendix Contents}}{}
\etocsettocdepth{subsection}
\localtableofcontents               %

\section{Table of notations} \label{appendix:notations}

We provide in \cref{tab:notation_table,tab:notation_table_discrete} a summary of the main notations used throughout the paper.

\begin{table}[h!]
  \centering\small
  \caption{\textbf{Notation table for the continuous setting.} This table summarizes the notations attached to the manifold, the Finsler metric and the integral operators, with references to the equations where they are defined.}
  \label{tab:notation_table}
  \footnotesize
  \begin{tabular}{l|l}
  \toprule
    \textbf{Notation} & \textbf{Description} \\
  \Xhline{1\arrayrulewidth}
  $\gM$ & Compact $m$-dimensional manifold embedded in $\R^D$ \\
  $\lVert w \rVert_M$, $\langle w, w'\rangle_M$ & $\lVert w \rVert_M^2 = w^\top M w$ and $\langle w, w'\rangle_M = w^\top M w'$, for $M \succ 0$ \\
  $\gF$, $\dstF(\cdot, \cdot)$ & Finsler metric and its associated distance \\
  $\backFF$ & Reversed Finsler metric \\
  $\alpha, \beta, \rmA, b$ & Randers metric parameters \\
  $\lambda_\FF$, $i_\gM$ & Reversibility constant $\sup_{x,v} \FF(x,v)/\FF(x,-v)$ and injectivity radius of $\gM$ (A2) \\
  $\mBH$, $\sigBH$ & Busemann--Hausdorff measure on $\gM$ and its density in a chart (\cref{eq:busemann_hausdorff_main}) \\
  $\leb_x$ & Lebesgue measure on $T_x\gM$, in the coordinate basis \\
  $\dist$, $\vol$ & Riemannian distance and volume induced on $\gM$ by the ambient Euclidean norm \\
  $\tball_x$, $\tball_x(r)$ & Finsler unit ball $\{\FF(x,\cdot) \leq 1\}$ and open ball $\{\FF(x,\cdot) < r\}$ in the tangent space $T_x\gM$ \\
  $\fball(x,r)$ & Forward metric ball $\{y \in \gM : \dstF(x,y) < r\}$ on the manifold $\gM$ \\
  $\eball^m$ & Euclidean unit ball of $\R^m$ \\
  $\gI_x$ & Indicatrix $\{v \in T_x\gM : \FF(x,v) = 1\}$, boundary of $\tball_x$ \\
  $\rvc(x)$, $\rmS(x)$ & Centroid and raw second moment of the unit ball $\tball_x$ (\cref{eq:centroid,eq:raw_second_moment}) \\
  $\gBL$ & Binet--Legendre metric of $\FF$, with $\gBL^{ij} = (m+2)\rmS^{ij}$ (\cref{eq:binet_legendre}) \\
  $\mathbb{P}$, $\rho$ & Sampling distribution on $\gM$ and its density $\rho = \diff\mathbb{P}/\diff\mBH$ \\
  $K$ & Radial kernel profile on $\R_+$ \\
  $W$, $W^{(\theta)}$ & Directed kernel on $\gM \times \gM$ and its $\theta$-normalization \\
  $\dout$, $\din$, $\qe$ & Out- and in-degrees of $W$, and their half-sum $\qe = (\dout + \din)/2$ \\
  $\gG_\FF$, $\gG'_\FF$ & Integral operators built from $W$; $\gG'_\FF = \gG_\backFF = \gG_\FF^*$ (\cref{prop:adjoint_reverse}) \\
  $\gGt$, $\gGts$, $\gGta$ & $\theta$-normalized integral operator, and its symmetric and antisymmetric parts \\
  $\Pts$, $\Pta$ & Continuous transport operators built from $\gGts$ and $\gGta$ (\cref{sec:main_results}) \\
  $\gL^s$, $\gL^a$ & Their limits as $\eps \to 0$: a weighted Laplacian and a drift along $\rvc$ (\cref{thm:sym_limit}) \\
  $\expF$ & Finsler exponential map of $\FF$ \\
  $G^k$, $N^i_j$ & Geodesic spray coefficients and nonlinear connection of $\FF$ (\cref{def:nonlinear_connection}) \\
  $\ChrisF^k_{ij}$, $\RicF$ & Formal Christoffel symbols and Ricci curvature of $\FF$ \\
  $\tau$, $S$ & Distortion and S-curvature of $\FF$ (\cref{def:s_curvature}) \\
  \bottomrule
  \end{tabular}
\end{table}

\begin{table}[h!]
  \centering
  \caption{\textbf{Notation table: discrete setting, hyper-parameters and constants.} This table summarizes the notations attached to the sampled data, the matrices built from it, the hyper-parameters, and the moments and constants of the kernel.}
  \label{tab:notation_table_discrete}
  \footnotesize
  \begin{tabular}{l|l}
  \toprule
    \textbf{Notation} & \textbf{Description} \\
  \Xhline{1\arrayrulewidth}
  $X_1, \ldots, X_N$ & A sequence of $N$ \iid samples from $\mathbb{P}$ on $\gM$ \\
  $\rmW_N$, $\rmQ_N$ & Empirical kernel matrix and half-sum of degrees $\rmQ_N = (\rmD_N + \rmD'_N)/2$ (\cref{sec:discrete_operators}) \\
  $\rmD_N$, $\rmD'_N$ & Diagonal matrices of the out- and in-degrees of $\rmW_N$ \\
  $\rmW_N^{(\theta)}$, $\rmW_N^{(\theta,\bullet)}$ & $\theta$-normalized matrix and its symmetric ($\bullet = s$) and antisymmetric ($\bullet = a$) parts \\
  $\rmD_N^{(\theta,\bullet)}$ & Diagonal matrix of the row sums of $\rmW_N^{(\theta,\bullet)}$, discretizing $\gGtb[\rho]$ \\
  $\rmP^{(\theta,\bullet)}$ & Discrete transport operators, discretizing $\Ptb$ (\cref{eq:discrete_operators}) \\
  $\rmL^s$, $\rmL^a$ & Rescaled discrete operators $\rmP^{(\theta,s)}/\eps^2$ and $\rmP^{(\theta,a)}/\eps$ (\cref{eq:discrete_operators}) \\
  $\hat{q}_N$ & Empirical counterpart of the normalization $\qe$ \\
  $\Psi$, $Y_i$ & Spectral embedding $\Psi = (\psi_1, \ldots, \psi_\ell)$ of $\rmL^s$, and embedded points $Y_i = \Psi(X_i)$ \\
  $\rmH$, $\FF_R$ & Moment-Randers matrix of $\FF$ and its moment-Randers metric (\cref{def:moment_randers}) \\
  $\Gamma$, $V$ & Estimated Binet--Legendre metric and drift, computed from $\gL^s, \gL^a$ (\cref{prop:carre_du_champ}) \\
  $\hat{\rvc}$, $\hat{\rmH}$, $\hat{\FF}_R$ & Empirical centroid, matrix, and moment-Randers metric estimated from the embedding \\
  \Xhline{1\arrayrulewidth}
  $\eps$ & Kernel bandwidth parameter \\
  $\theta$ & Normalization exponent of the kernel, $\theta \in [0,1]$ \\
  $D$ & Ambient dimension of the manifold $\gM$ \\
  $m$ & Intrinsic dimension of the manifold $\gM$ \\
  $\ell$ & Dimension of the embedding, \ie number of eigenvectors of $\rmL^s$ retained \\
  \Xhline{1\arrayrulewidth}
  $m_n$, $\Lambda^k$, $s_n$, $\Upsilon$ & Moments of the kernel $K$ (\cref{def:kernel_moments}) \\
  $\tilde{m}_n$ & Normalized moments of the kernel, $\tilde{m}_n = m_n / m_0$ (\cref{props:moment_expression_unitball}) \\
  $\mu_n$ & Raw moments of the radial profile, $\mu_n = \int_0^{+\infty} K(r) r^{m+n-1} \diff r$ \\
  $C_K$, $\nu_K$ & Sub-exponential constants of the kernel: $K(r) \leq C_K e^{-\nu_K r}$ (A1) \\
  $c_1$ & Constant $(m+1)\mu_1 / (m\mu_0)$ appearing in $\gL^a$ (\cref{thm:sym_limit}) \\
  $c_2$ & Constant $\mu_2 / (2m\mu_0)$ appearing in $\gL^s$ and $\Gamma$ (\cref{thm:sym_limit,prop:carre_du_champ}) \\
  \bottomrule
  \end{tabular}
\end{table}

\section{A Primer of Riemannian and Finsler-Randers Geometry}
\label{appendix:finsler_geometry}

In this section we provide a brief introduction to Riemannian geometry and its extension to Finsler geometry. For a more complete and formal introduction to those subjects, we refer the reader to \citet{Sommer2020introduction} for Riemannian geometry, and to \citet{ohtaComparisonFinslerGeometry2021} for a comprehensive treatment of Finsler geometry and Randers metrics. Throughout this paper, we use the Einstein summation convention, where an index repeated once as a superscript and once as a subscript is summed over, \eg $\partial_i f\, v^i = \sum_{i=1}^m \partial_i f\, v^i$.

\paragraph{Riemannian geometry.} Riemannian manifolds can be described in multiple ways. Generally, they are defined as smooth manifolds (i.e.\ topological spaces that are locally homeomorphic to Euclidean spaces, with smoothly compatible charts) endowed with a smoothly varying inner product on the tangent space at each point, called the Riemannian metric. This inner product allows one to define distances and angles on the manifold, see \cref{fig:riemannian} for an illustration. One may think of a Riemannian manifold as a curved space in which the local cost of crossing a region is governed by the metric. More formally, given a smooth manifold $\gM$, for any point $x \in \gM$ we can define the \emph{tangent space} $T_x \gM$ as the vector space of all tangent vectors at $x$. A Riemannian manifold is then a pair $(\gM, \rmA)$, where $\rmA$ is a smoothly varying inner product on the tangent space at each point,
\begin{equation*}
    \langle \cdot, \cdot \rangle_{\rmA(x)} : T_x\gM \times T_x\gM \to \R, \qquad
    (v,w) \mapsto  v^\top \rmA(x) w.
\end{equation*}
In coordinates, components of tangent vectors carry upper indices, $v = (v^i)$, while components of covectors, such as the partial derivatives $\partial_i f$ of a function, carry lower indices. We write $\rmA_{ij}$ for the entries of $\rmA$, so that $\langle v, w \rangle_{\rmA} = \rmA_{ij} v^i w^j$, and $\rmA^{ij} = (\rmA^{-1})_{ij}$ for those of its inverse, so that $\rmA^{ik} \rmA_{kj} = \delta^i_j$. The same convention applies to any metric. This inner product induces a natural local norm, $\lVert v \rVert_{\rmA(x)} = \sqrt{v^\top \rmA(x) v}$ for $v \in T_x\gM$, which in turn defines the length and energy of a curve $\gamma: [0,1] \to \gM$:
\begin{equation*}
    L(\gamma) = \int_0^1 \lVert \dot{\gamma}(t) \rVert_{\rmA(\gamma(t))} \diff t
    \quad \text{ and } \quad
    E(\gamma) = \int_0^1 \lVert \dot{\gamma}(t) \rVert_{\rmA(\gamma(t))}^2 \diff t.
\end{equation*}
This lets us define the geodesic distance between two points as the length of the shortest path between them, which is also (up to reparametrization) the minimizer of the energy between these two points:
\begin{equation}
  \label{eq:geodesic_distance_def}
    \forall(a,b)\in\gM^2, \text{dst}(a,b) =
    \min_{\substack{\gamma:[0,1] \to \gM \\ \gamma(0)=a, \gamma(1)=b}} L(\gamma)
\end{equation}

\paragraph{Finsler geometry.} Finsler manifolds are a natural extension of Riemannian manifolds that relax two of their defining constraints: the local cost need not be a quadratic form, and it need not be symmetric. This lets Finsler geometry model \emph{local asymmetry} (the cost of crossing a region can depend on the direction of travel) while keeping a formulation close to that of Riemannian geometry, as illustrated in \cref{fig:finsler}.

\begin{definition}[Finsler metric] \label{def:finsler_metric}
  A Finsler metric on a smooth manifold $\gM$ is a function $\FF: T\gM \to [0, \infty)$ such that for each $x \in \gM$ and $v \in T_x\gM$, the following conditions hold:
  \begin{itemize}[itemsep=0pt, leftmargin=*, topsep=0pt]
    \item \textbf{Positive homogeneity}: $\FF(x, t v) = t \FF(x, v)$ for all $t > 0$ and $v \in T_x\gM$.
    \item \textbf{Strong convexity}: The Hessian of $\FF^2$ with respect to the velocity variable is positive definite for all nonzero $v \in T_x\gM$.
    \item \textbf{Triangle inequality}: $\FF(x, v + w) \leq \FF(x, v) + \FF(x, w)$ for all $v, w \in T_x\gM$.
  \end{itemize}
\end{definition}

Since homogeneity is only required for positive scalars $t$, these metrics induce local distances that depend on direction: we generally do not have $\FF(x, v) = \FF(x, -v)$.
\begin{remark}
  In some textbooks, one may also encounter the notation $\FF(v)$, dropping the explicit dependence on the base point $x$. Throughout this work, we will keep the explicit dependence on $x$ on the Finsler-based quantities.
\end{remark}

\begin{definition}[Fundamental tensor] \label{def:fundamental_tensor}
  The fundamental tensor $g$ of a Finsler metric $\FF$ is given by
  \begin{equation*}
    g_{ij}(x,v) = \frac{1}{2}\frac{\partial^2 \FF^2}{\partial v^i\partial v^j}(x,v),
  \end{equation*}
  where $(x,v)\in T\gM$ and $v^i$ are the components of $v$ in a local coordinate system. The fundamental tensor is a smoothly varying, positive definite bilinear form on the tangent space at each point, and it generalizes the notion of a Riemannian metric to Finsler geometry.
\end{definition}

A general Finsler tangent indicatrix $\{v \in T_x\gM : \FF(x,v) = 1\}$ need not be an ellipsoid at all, as Finsler metrics are only required to be strongly convex, not quadratic. Nonetheless, the fundamental tensor always yields the best local Riemannian (quadratic) approximation of $\FF$ at $(x,v)$.

Similarly to Riemannian geometry, shortest paths under a Finsler metric are also minimizers of the energy of the curve. Using Euler's homogeneous function theorem, one can show that the Finsler local energy satisfies
\begin{equation*}
    \FF^2(x,v)=v^\top g(x,v)v = \lVert v\rVert_{g(x,v)}^2.
\end{equation*}
Consequently, we can define the length and energy of curves, and thus geodesic distances, on Finsler manifolds by replacing $\rmA(\gamma_t)$ with $g(\gamma_t,\dot{\gamma}_t)$ in the definitions above. This lets us use the same class of algorithms as in the Riemannian case to solve the shortest-path problem, at marginal additional cost \citep{rygaard2025georce,arvanitidis2019fast}.

\paragraph{Balls.} Several kinds of balls appear in this work. For $x \in \gM$ and $r > 0$,
\begin{align*}
  &\tball_x = \{v \in T_x\gM : \FF(x,v) \leq 1\}, \ \, \tball_x(r) = \{v \in T_x\gM : \FF(x,v) < r\}, \ \,
  \fball(x,r) = \{y \in \gM : \dstF(x,y) < r\}.
\end{align*}
The first two live in the tangent space $T_x\gM$: $\tball_x$ is the (closed) unit tangent ball, whose boundary is the indicatrix $\gI_x = \{v \in T_x\gM : \FF(x,v) = 1\}$, and $\tball_x(r)$ is the open tangent ball of radius $r$, which coincides with $r\tball_x$ up to its boundary. The third one is the forward metric ball and lives on the manifold $\gM$. Since $\dstF$ is not symmetric, it differs in general from the backward ball. Below the injectivity radius, $\expF_x$ maps $\tball_x(r)$ onto $\fball(x,r)$. Finally, $\eball^m$ denotes the Euclidean unit ball of $\R^m$, and $\vole(\eball^m)$ its volume. Unlike $\leb_x(\tball_x)$, it does not depend on $x$.

\paragraph{Randers metrics.} A Randers metric \citep{randers1941asymmetrical} is a special type of Finsler metric that can be expressed as the sum of a Riemannian metric and a one-form.

\begin{definition}[Randers metric] \label{def:randers_metric}
  A Randers metric on a smooth manifold $\gM$ is a Finsler metric $\FF$ that can be expressed as
\begin{equation*}
    \FF(x, v) = \alpha(x, v) + \beta(x,v),
\end{equation*}
  where $\alpha$ is a Riemannian metric and $\beta$ is a $1$-form, such that $\alpha(x,v)^2 = v^\top \rmA(x) v$, with $\rmA:\gM \to \mathcal{S}^D_{++}$ a smooth map from the manifold to the space of positive definite matrices, and $\beta(x,v) = b(x)^\top v$, with $b: \gM \to \R^D$ a smooth vector field on the manifold verifying $\lVert b(x) \rVert_{\rmA^{-1}} <1$.
\end{definition}
In coordinates, $\beta(x,v) = b_i(x) v^i$, so that $b$ carries a lower index as a covector. Its index is raised with $\rmA$, $b^i = \rmA^{ij} b_j$, and $\lVert b \rVert_{\rmA^{-1}}^2 = \rmA^{ij} b_i b_j = b_i b^i$.

Randers metrics are a particular case of Finsler metrics that are generally easier to work with due to their specific structure: for such metrics we have an explicit form of $g$ and $g^{-1}$, enabling efficient implementation. Moreover, the Finsler indicatrix of a Randers metric is always an ellipsoid, simply not centered at the origin, obtained by translating the $\alpha$-indicatrix by a vector determined by $b(x)$.

\begin{example}
Perhaps the simplest example of a Randers metric is the Euclidean space $\R^2$ with the standard Euclidean metric $\alpha(x,v) = \lVert v \rVert_2$ and a constant vector field $b(x) = (b_1, b_2)$ with $\lVert b \rVert_2 < 1$. In this case, the Randers metric is given by
\begin{equation*}
    \FF(x, v) = \lVert v \rVert_2 + b^\top v = \sqrt{v_1^2 + v_2^2} + b_1 v_1 + b_2 v_2.
\end{equation*}
The indicatrix of this Randers metric is an ellipse that is shifted in the direction of the vector $b$. One may verify that, in this context, geodesics are straight lines, but the distance between two points depends on the direction of travel.
\end{example}

\paragraph{Berwald metrics.} Among Finsler metrics, an important intermediate class, strictly more general than Riemannian metrics but with enough rigidity to retain some of their linear structure, is given by Berwald metrics.

\begin{definition}[Berwald metric] \label{def:berwald_metric}
  A Finsler metric $\FF$ is a \emph{Berwald metric} if its formal Christoffel symbols $\ChrisF^k_{ij}(x,v)$ \citep{ohtaComparisonFinslerGeometry2021} do not depend on the direction $v$, \ie $\ChrisF^k_{ij}(x,v) = \Gamma^k_{ij}(x)$ for some coefficients $\Gamma^k_{ij}$ on $\gM$.
\end{definition}
Every Riemannian metric is trivially Berwald, since its fundamental tensor $g_{ij}(x,v) = \rmA_{ij}(x)$ does not depend on $v$ to begin with. In the Randers case, $\FF = \alpha + \beta$ is Berwald if and only if $\beta$ is parallel with respect to the Levi--Civita connection of $\alpha$ \citep{ohtaComparisonFinslerGeometry2021}, a condition on $b$ that constant vector fields such as the one in the previous example trivially satisfy, but that a direction field varying freely across $\gM$, as we consider in our experiments, generally will not.

\paragraph{Binet--Legendre metric.} The fundamental tensor $g(x,v)$ of a Finsler metric depends on the direction $v$. A direction-free Riemannian metric is obtained by averaging over the unit tangent ball $\tball_x$ \citep{matveev2012binet}.
\begin{definition}[Binet--Legendre metric] \label{def:binet_legendre}
  The Binet--Legendre metric $\gBL$ of $\FF$ is the Riemannian metric whose inverse is
  \begin{equation} \label{eq:binet_legendre}
    \gBL^{ij}(x) = \frac{m+2}{\leb_x(\tball_x)} \int_{\tball_x} v^i v^j \diff\leb_x(v) = (m+2) \rmS^{ij}(x).
  \end{equation}
\end{definition}
The factor $m+2$ ensures that $\gBL = \rmA$ when $\FF(x,v) = \sqrt{v^\top \rmA(x) v}$ is Riemannian (\appref{appendix:proofs:cor:berwald_limit}). Here $\gBL^{ij}$ denotes the entries of the inverse of $\gBL$, following the convention above. By contrast, the second moment $\rmS^{ij}$ of $\tball_x$, like the kernel moment $m_2^{ij}$ (\cref{def:kernel_moments}), is defined directly with upper indices, and is not the inverse of a matrix with lower indices. The metric $\gBL$ is what the symmetric limit $\gL^s$ sees (\cref{thm:sym_limit}), while the antisymmetric limit $\gL^a$ sees the centroid $\rvc(x)$ of $\tball_x$.

\paragraph{Busemann--Hausdorff measure.} Finsler manifolds carry no canonical volume; the two usual choices are the Busemann--Hausdorff and Holmes--Thompson measures, which coincide in the Riemannian case \citep{shen2001lectures}. We use the former.
\begin{definition}[Busemann--Hausdorff measure] \label{def:busemann_haussdorf_measure}
  The Busemann--Hausdorff measure $\mBH$ is the measure on $\gM$ which reads, in a chart, $\diff\mBH(x) = \sigBH(x) \diff x$, with density
  \begin{equation} \label{eq:busemann_haussdorf_measure}
    \sigBH(x) = \frac{\vole(\eball^m)}{\leb_x(\tball_x)},
  \end{equation}
  where $\leb_x(\tball_x)$ is the Lebesgue volume of $\tball_x$ in the coordinate basis $(\partial_i|_x)$. Under a change of chart, $\sigBH$ and $\diff x$ are multiplied by inverse factors, so that $\mBH$ does not depend on the chart.
\end{definition}
Under $\mBH$, every unit tangent ball has the volume of the Euclidean unit ball. This is what makes the zeroth kernel moment $m_0$ constant (\cref{props:moment_expression_unitball}), so that the degree normalization behaves as in the Riemannian case.

\paragraph{Distortion and S-curvature.} The Jacobian of $\expF_x$ with respect to $\mBH$ (\cref{lemma:jacobian_expansion}) involves the mismatch between $\mBH$ and the direction-dependent volume $\sqrt{\det g(x,v)} \diff x$, which is measured by the following two quantities.
\begin{definition}[Distortion and S-curvature] \label{def:distortion} \label{def:s_curvature}
  For $z \in \gM$ and $u \in T_z\gM \setminus\{0\}$, the distortion of $\FF$ with respect to $\mBH$ is
  \begin{equation*}
    \tau(z, u) = \log \left( \frac{\sqrt{\det g_{ij}(z, u)}}{\sigBH(z)} \right),
  \end{equation*}
  and the S-curvature is its derivative along the geodesic $\gamma$ with $\gamma(0) = z$ and $\dot\gamma(0) = u$, \ie $S(z,u) = \frac{\diff}{\diff t} \tau(\gamma(t), \dot\gamma(t)) \big|_{t=0}$.
\end{definition}
The S-curvature vanishes for Berwald metrics, and in particular for Riemannian ones \citep[Proposition 4.3]{shen2016introduction}, but not for general Randers metrics. It produces the terms $s_0$ and $s_1$ of \cref{def:kernel_moments}, which are removed by the centering of $\Ptb$.

\paragraph{Consequences of assumption (A2).} The proofs use three uniform properties of $(\gM, \FF)$, which follow from the compactness of $\gM$ and the smoothness and positivity of $\FF$ on $T\gM \setminus \{0\}$.
\begin{enumerate}[label=(\roman*), itemsep=0pt, topsep=0pt, leftmargin=*]
  \item \emph{Bounded asymmetry.} The reversibility constant $\lambda_\FF = \sup_{x, v \neq 0} \FF(x,v)/\FF(x,-v)$ is finite, as the supremum of a continuous function on the compact unit sphere bundle. Hence $\lambda_\FF^{-1} \dstF(y,x) \leq \dstF(x,y) \leq \lambda_\FF \dstF(y,x)$.
  \item \emph{Completeness.} $\gM$ is forward and backward complete, so that any two points are joined by a minimizing geodesic and $\expF$ is defined on all of $T\gM$ \citep{bao2012introduction}.
  \item \emph{Injectivity radius.} The injectivity radius is bounded below by some $i_\gM > 0$, so that $\expF_x$ is a diffeomorphism from $\tball_x(i_\gM)$ onto its image for every $x$.
\end{enumerate}

\begin{figure}
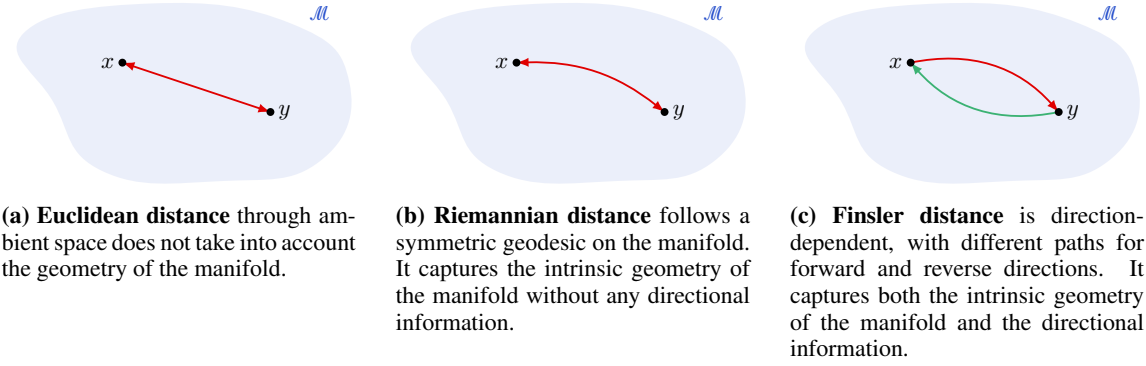

  \centering
  \begin{subfigure}[t]{0.31\textwidth}
    \centering
    \includestandalone[width=\textwidth]{euclidean_distance}
    \caption{\textbf{Euclidean distance} through ambient space does not take into account the geometry of the manifold.}
    \label{fig:euclidean}
  \end{subfigure}
  \hfill
  \begin{subfigure}[t]{0.31\textwidth}
    \centering
    \includestandalone[width=\textwidth]{riemannian_distance}
    \caption{\textbf{Riemannian distance} follows a symmetric geodesic on the manifold. It captures the intrinsic geometry of the manifold without any directional information.}
    \label{fig:riemannian}
  \end{subfigure}
  \hfill
  \begin{subfigure}[t]{0.31\textwidth}
    \centering
    \includestandalone[width=\textwidth]{finsler_distance}
    \caption{\textbf{Finsler distance} is direction-dependent, with different paths for forward and reverse directions. It captures both the intrinsic geometry of the manifold and the directional information.}
    \label{fig:finsler}
  \end{subfigure}
  \caption{\textbf{Comparison of distance metrics on a manifold.} Each panel shows the same manifold $\gM$ with two points $x$ and $y$, but different notions of distance between them. (a) Euclidean distance is the straight line in ambient space. (b) Riemannian distance follows a symmetric geodesic on the manifold. (c) Finsler distance is direction-dependent, with different paths for forward (orange) and reverse (purple) directions.}
  \label{fig:finsler_metric}
\end{figure}

\section{Additional experimental results} \label{appendix:additional_experiments}

We provide in this section additional details on the experiments presented in \cref{sec:experiments}. We show how we construct the Randers metric and the drift for the Swiss Roll dataset, the Matsumoto metric approximation, and the Directed Stochastic Block Model (DiSBM).

\paragraph{On the illustrations.} The presented method recovers the drift field as a pointwise quantity at each embedded sample. In the figures, we display the drift field as a vector field on a grid of points in the embedding space, obtained by interpolating the drift vectors at the embedded samples with a neural network.

\subsection{Swiss Roll dataset} \label{appendix:experiments:swiss_roll}
We consider the Swiss Roll dataset $(X_i)_i$ generated using \textit{sklearn} with $N=2000$ points and a noise of $0.1$. We consider this dataset as sampled from a Randers manifold. The Riemannian metric is given by:
\begin{equation*}
  \rmA^{-1}_{ii}(z) = \sum_k \exp\left(- \frac{\lVert z - z_k \rVert^2}{h}\right) + \xi, \quad \text{and} \quad \rmA^{-1}_{ij} = 0 \text{ for } i \neq j,
\end{equation*}
where $\xi= 10^{-3}$ and $h$ is chosen as the average distance to the 5th nearest neighbor.

We can then construct the drift as tangential to the dataset. Let $\varphi_i= \operatorname{atan2}(X_{i,z}, X_{i,x})$ be the angle of the point $X_i$ in the $xz$-plane. The initial tangent vector is given by $b(X_i) = (-\sin(\varphi_i), 0, \cos(\varphi_i))$, which is then normalized so that $\lVert b(X_i) \rVert_{\rmA^{-1}} = 0.9$.

The directed graph is computed as a KNN graph with $k=10$ neighbors, and the edge weights, \ie geodesic distances between nodes, are computed using the \textit{GEORCE} algorithm \citep{rygaard2025georce}. The $\eps$ in the Gaussian kernel is chosen as the standard deviation of the geodesic distances between the nodes. The points are then embedded in $\R^2$ using the Isomap algorithm. The embedding is then used to estimate the Randers metric and the drift following the method described in the paper.

\paragraph{Quantitative evaluation.} For \cref{tab:swiss_roll}, we keep the same metric and vary the drift norm $\beta = \lVert b \rVert_{\rmA^{-1}}$ and the number of points, with $3$ seeds per configuration. Here, we use radius graphs rather than $k$-NN graphs. With a strong drift, the kernel extends further along the drift than across it, so that a $k$-NN graph would truncate it unevenly and bias its moments, and hence the limit of \cref{thm:sym_limit}. The bandwidth $\eps$ is the median Finsler distance from a sample to its $10$-th nearest neighbor. The graph is then truncated to the ball $\dstF(X_i,X_j) < 3\eps$, outside of which the kernel is negligible. As above, the edge weights are $K(\dstF(X_i,X_j)/\eps)$ with the Gaussian kernel, where $\dstF$ is computed with \textit{GEORCE}. In other words, the graph keeps every pair on which the kernel is non-negligible, and the geometry is carried by the weights.

The ground truth $\rvc$ and $\gBL$ is given by \cref{prop:centroid_randers_metric}.  The tangent plane of the roll is estimated by a local PCA on the $15$ nearest neighbors. The Jacobian $J_i$ of Isomap at $X_i$ is estimated by least squares on the same neighbors, $Y_j - Y_i \approx J_i (X_j - X_i)$ in these tangent coordinates. It is only needed for the direction of the drift, since $\lVert \hat{\rvc} \rVert_{\hat{g}_{BL}}$ does not depend on the embedding. \Cref{tab:swiss_roll} reports the cosine between $V_i$ and its limit $c_1 J_i \rvc(X_i)$: the direction of the drift is recovered at every strength and sample size.

\subsection{Matsumoto Metric approximation} \label{appendix:experiments:matsumoto}

We focus on the case of a graph sampled from a Matsumoto metric \citep{matsumoto_theory_1992}, more precisely on the metric induced by movement on a height map \citep{matsumoto1989slope, chansri2018geometry}. We define on $[-3.2,3.2]\times[-2.4,2.4]$ the height function $h$:

\begin{align*}
  h(x,y) &= 0.4\sin(x)\cos(y) + 0.15 \exp\left(-((x+1)^2 + (y-0.5)^2)\right)- 0.15 \exp\left(-((x-1)^2 + (y+0.5)^2)\right).
\end{align*}

The Matsumoto metric is then given by:
\begin{align*}
    \alpha((x,y),v)^2 &= (1+\partial_x h(x,y)^2) v_x^2 + 2 \partial_x h(x,y) \partial_y h(x,y) v_x v_y + (1+\partial_y h(x,y)^2) v_y^2, \\
    \beta((x,y),v) &= \partial_x h(x,y) v_x + \partial_y h(x,y) v_y, \\
    \FF((x,y),v) &= \frac{\alpha((x,y),v)^2}{\alpha((x,y),v) - \beta((x,y),v)}.
\end{align*}

We then sample $N=3000$ points uniformly on $[-3.2,3.2]\times[-2.4,2.4]$ that yield the samples $(x_i, y_i, h(x_i, y_i))_i$. Similarly to the Swiss Roll dataset, we construct a directed graph from the sampled points using a KNN graph with $k=10$ neighbors, and where the edge weights are computed using the \textit{GEORCE} algorithm \citep{rygaard2025georce}. The $\eps$ in the Gaussian kernel is chosen as the standard deviation of the geodesic distances between the nodes. The samples are then embedded in $\R^2$ using the Isomap algorithm. The embedding is then used to estimate the Randers metric and the drift following the method described in the paper.

\subsection{Directed Stochastic Block Model} \label{appendix:experiments:disbm}
\paragraph{Formal definition.} We start by giving a formal definition of the DiSBM. Let $N$ be the number of nodes in the graph, $B$ the number of communities, and $n \in \sN^B$ the number of nodes in each community. We denote by $\Pi \in [0,1]^{B \times B}$ the block matrix that defines the connection probabilities between communities, such that $\sum_j \Pi_{i,j} = 1$ for every $i$. We consider in particular the case of directed graphs that are generated from a circular block model, where the communities are arranged in a circle and the connection probabilities are defined as follows:
\begin{equation*}
  \Pi_{i,j} = \begin{cases}
    p & \text{if } j = i + 1 \mod B, \\
    q & \text{if } j = i - 1 \mod B, \\
    r & \text{if } j = i, \\
    0 & \text{otherwise,}
  \end{cases}
\end{equation*}
where $p, q, r \in [0,1]$ and $p + q + r = 1$. The probabilities $p, q, r$ are the probabilities of connecting to the next, previous, and same community, respectively. The adjacency matrix $\rmQ \in \{0,1\}^{N \times N}$ of the graph is then generated by sampling each entry $\rmQ_{i,j}$ independently according to a Bernoulli distribution with parameter $\Pi_{c(i),c(j)}$, where $c(i)$ is the community label of node $i$. The model is illustrated in \cref{fig:disbm_cycle} for $B=3$.

\paragraph{Experiment setting.}
We consider a DiSBM with $B=15$ communities, $N = 1000$, $r=0.1$, $p=0.4$ and $q=0.5$. Note that $p > q$ is not necessary: as soon as $p \neq q$, the flow between communities has a preferred orientation, forward if $p > q$ and backward if $q > p$. The community labels are drawn uniformly at random, and the adjacency matrix is generated according to the DiSBM model. The $\eps$ in the kernel is chosen as $1$. The points are then embedded in $\R^2$ using the first two non-trivial eigenvectors of $\rmL^s$, with $\theta = 1$. The embedding is then used to estimate the Randers metric and the drift following the method described in the paper.

\begin{figure}
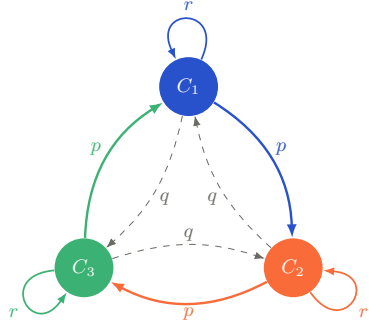
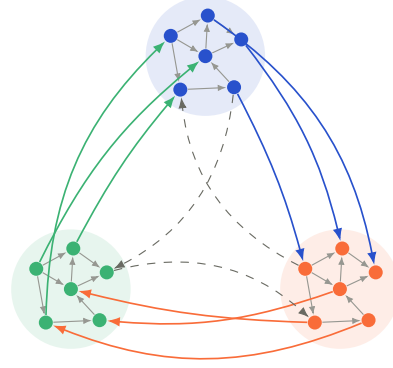

  \centering
  \begin{subfigure}[t]{0.4\textwidth}
    \centering
    \includestandalone[width=0.85\textwidth]{disbm_cycle_model}
    \caption{\textbf{Block model.} Each community sends its edges to the next community with probability $p$, to the previous one with probability $q$, and to itself with probability $r$.}
    \label{fig:disbm_cycle_model}
  \end{subfigure}
  \hspace{0.06\textwidth}
  \begin{subfigure}[t]{0.4\textwidth}
    \centering
    \includestandalone[width=0.85\textwidth]{disbm_cycle_graph}
    \caption{\textbf{Sampled graph.} Intra-community edges (gray) have no preferred orientation, while edges between communities mostly follow the cycle (colored by source), edges going against the cycle are less frequent (dashed).}
    \label{fig:disbm_cycle_graph}
  \end{subfigure}
  \caption{\textbf{Cyclic Directed Stochastic Block Model.} Illustration with $B=3$ communities and $p > q$. (a) Connection probabilities of the block matrix $\Pi$. (b) A realization of the adjacency matrix $\rmQ$.}
  \label{fig:disbm_cycle}
\end{figure}

\paragraph{Direction of the drift.}
Although DiSBM has no underlying manifold, the direction of the flow between communities is known: forward if $p > q$, backward if $q > p$, and none if $p = q$. We measure whether the recovered drift follows it as $p - q$ varies. The $N$, $B$ and $r$ parameters remains as in the previous paragraph. Let $Y_i \in \R^2$ be the embedding of node $i$ and $c(i)$ its community. The center of community $k$ is
\begin{equation*}
  \mu_k = \frac{1}{\lvert C_k \rvert} \sum_{j \in C_k} Y_j,
\end{equation*}
where $C_k = \{ j : c(j) = k \}$ is the set of its nodes. The forward direction at community $k$ is the tangent to the cycle of communities, $t_k = \mu_{k+1} - \mu_{k-1}$, with indices taken modulo $B$. The chord $\mu_{k+1} - \mu_k$ would be off by half the angle between two consecutive communities. The \emph{signed cosine} of node $i$ is the cosine between its drift $\rmV_i$ and the forward direction at its community,
\begin{equation*}
  s_i = \frac{\langle \rmV_i, t_{c(i)} \rangle}{\lVert \rmV_i \rVert \, \lVert t_{c(i)} \rVert} \in [-1, 1],
\end{equation*}
which is close to $1$ when the drift points forward, close to $-1$ when it points backward, and spread around $0$ when it has no preferred direction. It is signed because it is measured against the forward direction whatever the actual flow, so that it changes sign when the flow is reversed. We also report the estimated strength of the asymmetry, $\lVert \hat{\rvc}_i \rVert_{\hat{g}_{BL}}$. \cref{fig:disbm_direction} shows that the drift points in the direction of the flow as soon as $p \neq q$: the median signed cosine is already $\pm 0.89$ for $\lvert p - q \rvert = 0.02$, and above $0.99$ from $\lvert p - q \rvert = 0.1$ on. When $p = q$, the drift has no preferred direction. The strength of the asymmetry is symmetric in $p - q$ and grows linearly with $\lvert p - q \rvert$.

\begin{figure}
  \centering
  \begin{subfigure}[t]{0.48\textwidth}
    \centering
    \includegraphics[width=\textwidth]{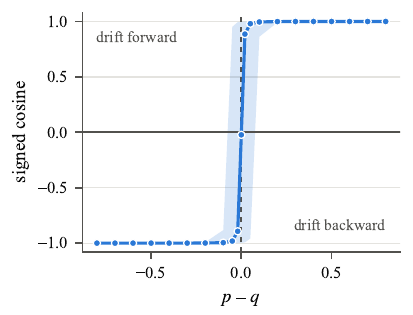}
    \caption{Signed cosine according to the direction of the flow.}
    \label{fig:disbm_direction_cos}
  \end{subfigure}
  \hfill
  \begin{subfigure}[t]{0.48\textwidth}
    \centering
    \includegraphics[width=\textwidth]{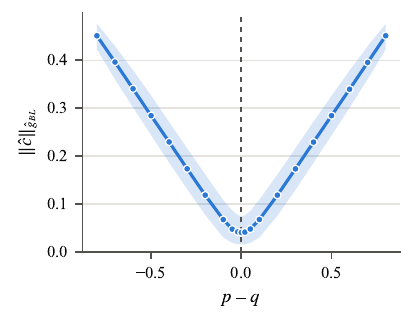}
    \caption{Estimated strength of the asymmetry.}
    \label{fig:disbm_direction_cnorm}
  \end{subfigure}
  \caption{\textbf{Embedding of the DiSBM and its dynamics.} Direction of the drift on the DiSBM as a function of $p - q$, with $N = 1000$ nodes, $B = 15$ communities and $r = 0.1$. Lines are medians over the nodes of $5$ random graphs, and shaded bands their $10$--$90\%$ range.}
  \label{fig:disbm_direction}
\end{figure}

\section{Proofs} \label{appendix:proofs}

This appendix gathers the proofs of the results stated in the main text.

\subsection{Forward and backward transport operators}
\label{appendix:proofs:prop:adjoint_reverse}

\begin{proposition}[Relation between the forward and backward transport operators]\label{prop:adjoint_reverse}
The backward transport operator $\gG'_\FF$ is equal to the transport operator of the reversed Finsler metric $\backFF$, and is the $L^2(\gM,\mBH)$-adjoint of $\gG_\FF$, \ie $\gG'_\FF = (\gG_\FF)^* = \gG_\backFF$.
\end{proposition}

Before proving \cref{prop:adjoint_reverse}, we prove the following lemma, which allows us to formally work with the reversed Finsler metric $\backFF$ and its associated distance.

\begin{lemma}[Finsler metric and its reverse] \label{property:finsler_reverse}
The reversed Finsler metric $\backFF$ is a Finsler metric on $\gM$. Moreover, its geodesics are the reverse of the geodesics of $\FF$, and the geodesic distance $\text{dst}_{\backFF}$ is the reverse of the geodesic distance $\text{dst}_\FF$.
\end{lemma}

\begin{proof}[Proof of \cref{property:finsler_reverse}]
  First, we show that $\backFF$ is indeed a Finsler metric. By definition, $\backFF(x,v) = \FF(x, -v)$ for all $x \in \gM$ and $v \in T_x\gM$. Hence $\backFF(x,v) > 0$ for all $v \neq 0$, and $\backFF(x,0) = 0$. For any $t > 0$, we have $\backFF(x, t v) = \FF(x, -t v) = t \FF(x, -v) = t \backFF(x,v)$, satisfying the positive homogeneity property. Finally, the Hessian of $\backFF^2$ with respect to $v$ is the same as the Hessian of $\FF^2$ with respect to $-v$, which is positive definite by the definition of a Finsler metric. Hence, $\backFF$ satisfies all three properties of a Finsler metric.

  For the second claim, let $\gamma: [0, T] \to \gM$ be the geodesic of $\FF$ connecting two points $x,y \in \gM$, so that $\gamma(0) = x$ and $\gamma(T) = y$. Then, the reverse curve $\sigma(t) = \gamma(T-t)$ is a curve connecting $y$ to $x$. We have that $\dot{\sigma}(t) = -\dot{\gamma}(T-t)$, so that
  \begin{align*}
    \int_0^T \backFF(\sigma(t), \dot{\sigma}(t)) \diff t &= \int_0^T \FF(\sigma(t), -\dot{\sigma}(t)) \diff t \\
    &= \int_0^T \FF(\gamma(T-t), \dot{\gamma}(T-t)) \diff t &\text{by the definition of $\sigma$} \\
    &= \int_0^T \FF(\gamma(s), \dot{\gamma}(s)) \diff s & \text{by change of variable $s = T-t$} \\
    &= \text{dst}_\FF(x, y).
  \end{align*}
  Since $\gamma$ is the geodesic of $\FF$, it minimizes the integral of $\FF$ along curves connecting $x$ and $y$. Therefore, $\sigma$ minimizes the integral of $\backFF$ along curves connecting $y$ and $x$, and is thus the geodesic of $\backFF$ connecting $y$ to $x$. This shows that the geodesics of $\backFF$ are the reverse of the geodesics of $\FF$, and that the geodesic distance $\text{dst}_{\backFF}(x,y) = \text{dst}_\FF(y, x)$.
\end{proof}

Applying this lemma, we can now prove \cref{prop:adjoint_reverse}.

\begin{proof}[Proof of \cref{prop:adjoint_reverse}]
  By \cref{property:finsler_reverse}, $\dstF(x,y) = \dst_\backFF(y,x)$ for all $x,y\in\gM$, so that, denoting by $W_\backFF$ the kernel built from $\backFF$ in place of $\FF$,
  \begin{equation*}
    W(x,y) = \frac{1}{\eps^m} K\left(\frac{\dstF(x,y)}{\eps}\right) = \frac{1}{\eps^m} K\left(\frac{\dst_\backFF(y,x)}{\eps}\right) = W_\backFF(y,x).
  \end{equation*}
  Hence $\gG'_\FF f(x) = \int_\gM W(y,x) f(y) \diff\mBH(y) = \int_\gM W_\backFF(x,y) f(y) \diff\mBH(y) = \gG_\backFF f(x)$.

  For the adjoint claim, by Fubini's theorem,
  \begin{align*}
    \langle \gG_\FF f, h \rangle &= \iint_{\gM \times \gM} W(x,y) f(y) h(x)  \diff\mBH(y) \diff\mBH(x)  \\
    &= \int_\gM f(y) \left( \int_\gM W(x,y) h(x) \diff\mBH(x) \right) \diff\mBH(y) = \langle f, \gG'_\FF h \rangle. \qedhere
  \end{align*}
\end{proof}

\subsection{Proof of \texorpdfstring{\cref{thm:moment_expansion}}{}}
\label{appendix:proofs:thm:moment_expansion}

Despite its simple statement, the proof of \cref{thm:moment_expansion} is rather dense, and requires a number of technical lemmas. For ease of reading, we first give a proof sketch, and defer the technical lemmas to the paragraphs below. We strongly encourage the interested reader to read \citet{ohtaComparisonFinslerGeometry2021} and \citet{shen2016introduction} for a comprehensive treatment of Finsler geometry, and \citet{coifman_diffusion_2006} for a similar proof in the Riemannian case.

\emph{Proof sketch.} The proof of \cref{thm:moment_expansion} relies on the expansion of the integral operator $\gG_\FF$. To do so, we expand the function $f$ around a point $x$ using the Finsler exponential map $\expF_x$. This expansion will be used to expand the integral operator $\gG_\FF$ in terms of the derivatives of $f$ at $x$, and the moments of the kernel $K$. This expansion is possible since we assume that the kernel $K$ is sub-exponential, and thus that the integral can be localized around $x$ for small $\eps$. The change of variable through $\expF_x$ brings a Jacobian, which we expand in terms of the density $\sigBH$, the S-curvature and the Ricci curvature of $\FF$ (\cref{lemma:jacobian_expansion}). Once this expansion is obtained, we can then use \cref{prop:adjoint_reverse} and the fact that $\gG'_\FF = \gG_\backFF$ to obtain the expansion of $\gG'_\FF$. Finally, one can note the following identity:
\begin{equation*}
  \gGt[f] = \frac{1}{\qe^\theta} \gG_\FF \left[\frac{f}{\qe^\theta} \right],
\end{equation*}
so that the expansion of $\gGt$ can be easily obtained from the expansion of $\gG_\FF$. The final result is obtained by taking the symmetric and antisymmetric parts of the expansion of $\gGt$, and applying the proper normalization.

We give here the form of the differential operators that appear in the expansion of $\gG_\FF$. To do so, we start by defining the \emph{kernel moments} of $K$ as follows:
\begin{definition}[Kernel moments] \label{def:kernel_moments}
  Let $x \in \gM$ and $n \in \N$. Using the formal Christoffel symbols $\ChrisF^k_{ij}$, the Ricci curvature $\RicF$ and the S-curvature $S$ (\cref{def:s_curvature}) of the Finsler metric $\FF$, $\dot S$ denoting the derivative of $S$ along the geodesic flow, we define
  \begin{align*}
    &m_n (x) = \sigBH(x) \int_{T_x\gM} K(\FF(x,v)) v^{\otimes n} \diff\leb_x(v), \\
    &\Lambda^k(x) = \sigBH(x) \int_{T_x\gM} K(\FF(x,v)) \ChrisF^k_{ij}(x,v) v^i v^j \diff\leb_x(v), \\
    &s_n(x) = \sigBH(x) \int_{T_x\gM} K(\FF(x,v)) S(x,v) v^{\otimes n} \diff\leb_x(v), \\
    &\Upsilon(x) = \sigBH(x) \int_{T_x\gM} K(\FF(x,v)) \\
    & \qquad \qquad \times \left( S(x,v)^2 - \dot S(x,v) - \frac{1}{3} \RicF(x,v) \right) \diff\leb_x(v),
  \end{align*}
  where $\sigBH$ is the density of the Busemann--Hausdorff measure (\cref{def:busemann_haussdorf_measure}).
\end{definition}
With this definition at hand, we can clarify the definition of the differential operators $\gA_1$ and $\gA_2$ in \cref{thm:moment_expansion}:
\begin{align*}
  \gA_1[f](x) &= m_1^{i}(x) \partial_i f(x) - s_0(x) f, \\
  \gA_2[f](x) &= m_2^{ij}(x) \partial_i \partial_j f(x) - (\Lambda^k(x) + 2 s_1^k (x)) \partial_k f(x) + \Upsilon(x) f(x).
\end{align*}

\begin{remark}
  This expansion is a generalization of the moment expansion of the kernel in the Riemannian setting, see \eg \citet[Lemma 8]{coifman_diffusion_2006}. Several differences are to be noted:
  \begin{itemize}[itemsep=0em, topsep=0em]
    \item Due to the non-reversibility of the metric, $m_1$ does not vanish in general.
    \item The second-order term contains the term $\Lambda^k(x)$, which simplifies in the Riemannian setting, and the curvature term $\Upsilon(x)$, which only multiplies $f$ and reduces to an average of the Ricci curvature in the Riemannian setting.
    \item The S-curvature brings the terms $s_0$ and $s_1$, among which a zeroth-order term $-\eps s_0 f$ at first order in $\eps$, and contributes to $\Upsilon$. These contributions vanish in the Riemannian setting, and more generally for Berwald metrics, whose S-curvature with respect to $\mBH$ is identically zero \citep[Proposition 4.3]{shen2016introduction}.
  \end{itemize}
\end{remark}

We now turn to the technical lemmas that are used in the proof of \cref{thm:moment_expansion}.

\subsubsection*{Main lemmas}

We start by giving the expansion of a function $f$ around a point $x$ using the Finsler exponential map $\expF_x$. This expansion will be used to expand the integral operator $\gG_\FF$ in terms of the derivatives of $f$ at $x$, and the moments of the kernel $K$. It is to be contrasted with the expansion using the Riemannian exponential map $\expa_x$ \citep{coifman_diffusion_2006}.
\begin{lemma}[Function expansion using Finsler metric]\label{lemma:f_expansion}
    Let $x\in\gM$ and $v\in T_x\gM$ be a unit tangent vector at $x$. Let also $\eps\in\mathbb{R}^+_*$. Then
  \begin{equation}
    f(\expF_x(\eps v)) = f(x) + \eps \partial_i f(x) v^i + \frac{\eps^2}{2} \left(\partial_j \partial_i f(x) - \partial_k f(x) \ChrisF^k_{ij}(x,v) \right)v^j v^i + O(\eps^3),
  \end{equation}
  where $\ChrisF^k_{ij}$ is the Christoffel symbol of the Finsler metric $\FF$\footnote{The Christoffel symbols of a Finsler metric are also called \emph{formal Christoffel symbols} \citep{ohtaComparisonFinslerGeometry2021}}.
\end{lemma}

\begin{proof}
  Let $\gamma(\eps) = \expF_x(\eps v)$, and $\phi(\eps)= f(\gamma(\eps))$. We have that $\phi: \R \to \R$ and $\phi(0) = f(x)$. Taylor expansion of $\phi$ around $0$ yields:
\begin{equation*}
  \phi(\eps) = \phi(0) + \phi'(0)\eps + \frac{1}{2}\phi''(0)\eps^2 + O(\eps^3)
\end{equation*}
Now we need to express $\phi'(0)$ and $\phi''(0)$ in terms of the derivatives of $f$ at $x$.

For the \emph{first derivative}, by the chain rule,
\begin{align*}
  \phi'(0) = \frac{d}{d \eps} f(\gamma(\eps))\bigg|_{\eps=0} = \partial_i f(\gamma(\eps)) \dot{\gamma}^i (\eps)\bigg|_{\eps=0} = \partial_i f(x) v^i
\end{align*}

For the \emph{second derivative}, by the product rule,
\begin{align*}
  \phi''(\eps) &= \frac{d}{d \eps} (df(\gamma(\eps)) \dot{\gamma}(\eps)) \\
            &= \underbrace{\frac{d}{d \eps}\Big(df(\gamma(\eps))\Big) \dot{\gamma}^i(\eps)}_{\text{(A)}} + \underbrace{df(\gamma(\eps)) \ddot{\gamma}^i(\eps)}_{\text{(B)}}.
\end{align*}
Using the chain rule again, the first term (A) gives
\begin{align*}
  \frac{d}{d \eps}\Big(df(\gamma(\eps))\Big) \dot{\gamma}^i(\eps) \bigg|_{\eps=0} = \partial_j \partial_i f(\gamma(\eps)) \dot{\gamma}^j(\eps) \dot{\gamma}^i(\eps) \bigg|_{\eps=0} = \partial_j \partial_i f(x) v^j v^i.
\end{align*}
For the second term (B), since $\gamma$ is a geodesic under the Finsler metric $\FF$, the geodesic equation \citep[Chap. 3.3]{ohtaComparisonFinslerGeometry2021} gives
\begin{equation*}
  \ddot{\gamma}^i(\eps) = -\tilde{G}^i(\gamma(\eps), \dot{\gamma}(\eps)) = - \ChrisF^i_{jk}(x, \dot{\gamma}(\eps)) \dot{\gamma}(\eps)^j \dot{\gamma}(\eps)^k
\end{equation*}
where $\tilde{G}^i$ is the geodesic spray coefficient of the Finsler metric $\FF$, and $\ChrisF$ is the Christoffel symbol of the Finsler metric $\FF$. Consequently,
\begin{equation*}
  df(\gamma(\eps)) \ddot{\gamma}^k(\eps)\bigg|_{\eps=0} = -\partial_k f(x) \ChrisF^k_{ij}(x, v) v^i v^j,
\end{equation*}
where we changed the index from $i$ to $k$. Hence, the second derivative of $\phi$ at $0$ is
\begin{equation*}
  \phi''(0) = (\partial_j \partial_i f(x) - \partial_k f(x) \ChrisF^k_{ij}(x, v) )v^j v^i.
\end{equation*}
The Taylor expansion of $f$ around $x$ is thus
\begin{equation*}
  f(\expF_x(\eps v)) = f(x) + \eps \partial_i f(x) v^i + \frac{\eps^2}{2} \left(\partial_j \partial_i f(x) - \partial_k f(x) \ChrisF^k_{ij}(x, v) \right)v^j v^i + O(\eps^3).
\end{equation*}
This concludes the proof.
\end{proof}

To use the previous expansion, we need to localize the integral operator $\gG_\FF$ around a point $x \in \gM$, so that the integral is instead taken over the tangent space $T_x \gM$, which can be identified with $\R^m$. This is possible thanks to the decay of the kernel $K$, which allows us to restrict the integral to a small neighborhood of $x$.

\begin{lemma}[Integral localization]\label{lemma:integral_localization}
  Let $K$ satisfy (A1) and $x\in\gM$. Then, for any $f \in L_1 (\gM)$ smooth enough,
  \begin{equation*}
    \gG_\FF [f] (x) = \int_{\tball_x(r_0 / \eps)} K\left(\FF(x,v)\right) f(\expF_x(\eps v)) \Jac(\eps v) \diff \leb_x(v) + O(\eps^3),
  \end{equation*}
  where $\Jac$ is the Jacobian of the change of variable through the exponential map, and $0 < r_0 < i_\gM$ is a constant smaller than the injectivity radius of $\gM$, which does not depend on $x$.
\end{lemma}

\begin{proof}
We start by splitting the integral into two parts, one over the ball $\fball (x, r_0) = \{ y \in \gM : \dstF(x,y) < r_0 \}$ and one over its complement.

\emph{On the far region.} Since $\gM$ is compact, $f$ is bounded, and we denote $M_f = \sup_{y\in\gM} |f(y)| < \infty$. By definition, it holds that $\dstF(x,y) \geq r_0$ for all $y \in \gM \setminus \fball (x, r_0)$. Using the sub-exponential property of the kernel (A1), it holds:
\begin{align*}
    \eps^{-m}\int_{\gM \setminus \fball (x, r_0)} K\left(\frac{\dstF(x,y)}{\eps}\right) f(y) \diff\mBH(y) &\leq \eps^{-m} M_f \int_{\gM \setminus \fball (x, r_0)} C_K \exp\left(- \frac{\nu_K r_0}{\eps}\right) \diff\mBH(y) \\
    &\leq M_f C_K \eps^{-m} \exp\left(- \frac{\nu_K r_0}{\eps}\right) \mBH(\gM),
\end{align*}
which is $O(\eps^N)$ for any $N>0$, and in particular $O(\eps^3)$.

\emph{On the near region.} Because $r_0 < i_\gM$, the exponential map $\expF_x$ is a diffeomorphism from the tangent ball $\tball_x(r_0) \subset T_x\gM$ onto the forward metric ball $\fball (x, r_0) \subset \gM$ (it is $C^1$ at the origin and smooth elsewhere), so that we can do the change of variable $y = \expF_x (\eps v)$, with $v \in \tball_x(r_0/\eps)$. For such $v$, $\eps\FF(x, v) < r_0 < i_\gM$, so that the Finslerian distance writes $\dstF(x, \expF_x (\eps v)) = \eps \FF(x, v)$. The integral over the near region can then be rewritten as:
\begin{align*}
    \eps^{-m} \int_{\fball (x, r_0)} K\left(\frac{\dstF(x,y)}{\eps}\right) f(y) \diff\mBH(y) &= \int_{\tball_x(r_0 / \eps)} K\left(\FF(x,v)\right) f(\expF_x(\eps v)) \Jac( \eps v) \diff \leb_x(v),
\end{align*}
where $\Jac(u)$ is the Jacobian of the change of variable $y = \expF_x(u)$, \ie $\diff\mBH(y) = \Jac(u) \diff\leb_x(u)$. The factor $\eps^{-m}$ is absorbed in the change of variable $u = \eps v$. %
Finally, we can combine the two regions to get the desired result.
\end{proof}

\begin{lemma}[Jacobian expansion] \label{lemma:jacobian_expansion}
  Let $x \in \gM$, $v \in T_x\gM\setminus\{0\}$, and $\eps > 0$ with $\eps\FF(x,v) < i_\gM$. Then the Jacobian $\Jac$ of \cref{lemma:integral_localization} satisfies
  \begin{equation*}
    \Jac(\eps v) = \sigBH(x)\left(1 - \eps S(x,v) + \frac{\eps^2}{2}\left( S(x,v)^2 - \dot{S}(x,v) - \frac{1}{3}\RicF (v) \right) + \bigO\big(\eps^3\big)\right),
  \end{equation*}
  where $S$ and $\dot S$ are the S-curvature (\cref{def:s_curvature}) and its derivative along the geodesic flow, and where the constant in $\bigO$ depends neither on $x$ nor on $v$.
\end{lemma}

\begin{proof}
By definition of the Jacobian, $\Jac$ has the form:
\begin{equation*}
  \Jac(\eps v) = \underbrace{\sigBH(\expF_x(\eps v))}_{\text{change of measure} (\star)} \space \times \space \underbrace{\det(d(\expF_x)_{\eps v})}_{\text{change of variable} (\star \star)},
\end{equation*}
where $\sigBH(\expF_x (\eps v))$ comes from the change of measure from the Busemann--Hausdorff measure to the Lebesgue measure on $T_x \gM$, and $\det(d(\expF_x)_{\eps v})$ comes from the change of variable $y = \expF_x (\eps v)$. This proof is split in two parts, the expansion of $\sigBH(\expF_x (\eps v))$ and the expansion of $\det(d(\expF_x)_{\eps v})$.

\emph{Expansion of $(\star)$.} By definition of the distortion (\cref{def:distortion}), we have, for every $u \in T_z\gM\setminus\{0\}$, that $\sigBH(z) = \exp(-\tau(z,u))\sqrt{\det g(z,u)}$. The left-hand side does not depend on $u$, so that we may choose the velocity of the geodesic $\gamma$, with $\gamma(0) = x$ and $\dot\gamma(0) = v$, at both of its ends:
\begin{equation*}
  \sigBH(x) = \exp(-\tau(x,v))\sqrt{\det g(x,v)}, \qquad \sigBH(\expF_x(\eps v)) = \exp(-\tau(\gamma(\eps),\dot\gamma(\eps)))\sqrt{\det g(\gamma(\eps),\dot\gamma(\eps))}.
\end{equation*}
Taking the ratio of these two identities, we have that
\begin{equation*}
  \sigBH(\expF_x(\eps v)) = \sigBH(x) \frac{\sqrt{\det g(\gamma(\eps),\dot\gamma(\eps))}}{\sqrt{\det g(x,v)}} \exp\big(\tau(x,v) - \tau(\gamma(\eps),\dot\gamma(\eps))\big).
\end{equation*}
By definition of the S-curvature (\cref{def:s_curvature}), and since $t \mapsto \gamma(s+t)$ is the geodesic with initial conditions $(\gamma(s),\dot\gamma(s))$, the first two derivatives of $t \mapsto \tau(\gamma(t),\dot\gamma(t))$ are $S(\gamma(t),\dot\gamma(t))$ and $\dot S(\gamma(t),\dot\gamma(t))$, so that
\begin{equation*}
  \tau(\gamma(\eps),\dot\gamma(\eps)) - \tau(x,v) = \eps S(x,v) + \frac{\eps^2}{2}\dot S(x,v) + O(\eps^3).
\end{equation*}
Hence, we have that
\begin{equation*}
  \sigBH(\expF_x(\eps v)) = \sigBH(x) \frac{\sqrt{\det g(\gamma(\eps),\dot\gamma(\eps))}}{\sqrt{\det g(x,v)}}\left(1 - \eps S(x,v) + \frac{\eps^2}{2}\left(S(x,v)^2 - \dot S(x,v)\right) + O(\eps^3)\right).
\end{equation*}

\emph{Expansion of $(\star \star)$.} Let $(e_1,\ldots,e_m)$ be a basis of $T_x\gM$ which is orthonormal for $g(x,v)$, and such that $e_m = v/\FF(x,v)$. For each $i$, we consider the variation of geodesics $\gamma_i(t,s) = \expF_x(t(v + se_i))$, and we define
  \begin{equation*}
    J_i(t) = \partial_s\gamma_i(t,0) = d(\expF_x)_{tv}(te_i),
  \end{equation*}
  which is a Jacobi field along the geodesic $\gamma(t) = \expF_x(tv)$ \citep[Prop. 5.5]{ohtaComparisonFinslerGeometry2021}. For the sake of completeness, we check here its initial conditions $J_i(0)$ and $\dot{J}_i(0)$. First,
  \begin{equation*}
    J_i(0) = \partial_s\gamma_i(0,0) = \partial_s x = 0.
  \end{equation*}
  Second, $\partial_t\gamma_i(0,s) = v + se_i$, and by the symmetry of second derivatives\footnote{While the exponential map is only $C^1$ at $0$, the Jacobi fields admit a $C^\infty$ extension, see \citep[Lemma 5.4.1]{bao2012introduction}.}, we have that $\partial_t\partial_s\gamma_i(0,0) = \partial_s\partial_t\gamma_i(0,0)$. Thus, we get $\dot J_i(0) = e_i$. As a Jacobi field, $J_i$ satisfies the Jacobi equation \citep[Def. 5.1]{ohtaComparisonFinslerGeometry2021}:
  \begin{equation*}
    \ddot{J}_i (t) + R(t) J_i (t) = 0
  \end{equation*}
  where $R$ is the curvature tensor of the Finsler metric $\FF$, evaluated along the geodesic $\gamma(t)$. By definition, for $t=0$, we have that $R(0) J_i(0) = 0$. Hence, the Jacobi equation at $t=0$ gives $\ddot{J}_i(0) = 0$. Differentiating the Jacobi equation gives
  \begin{equation*}
    \dddot{J}_i (t) + \dot{R}(t) J_i (t) + R(t) \dot{J}_i(t) = 0,
  \end{equation*}
  which, evaluated at $t=0$, gives $\dddot{J}_i(0) = - R(0) \dot{J}_i(0) = - R(0) e_i$. To write these relations in matrix form, we express the $J_i(t)$ in the basis $(E_1(t),\ldots,E_m(t))$ of $T_{\gamma(t)}\gM$ obtained by parallel transport of $(e_1,\ldots,e_m)$ along $\gamma$, for the Chern connection with reference vector $\dot\gamma$. This basis remains orthonormal for $g(\gamma(t),\dot\gamma(t))$, and the covariant derivatives along $\gamma$ are the ordinary derivatives of the components \citep{ohtaComparisonFinslerGeometry2021}. Letting $\rmJ(t)$ be the matrix of the components of $J_1(t),\ldots,J_m(t)$ in this basis, and $\rmR(t)$ the matrix of $R(t)$, we have that $\rmJ(0) = \ddot{\rmJ}(0) = 0$, $\dot{\rmJ}(0) = I_m$ and $\dddot{\rmJ}(0) = - \rmR(0)$. Hence, the Taylor expansion of $\rmJ$ around $0$ is given by
  \begin{equation*}
    \rmJ(t) = t I_m - \frac{t^3}{6} \rmR(0) + O(t^4).
  \end{equation*}
  Since $J_i(\eps)/\eps = d(\expF_x)_{\eps v}e_i$, the matrix $\rmJ(\eps)/\eps$ represents $d(\expF_x)_{\eps v}$ in the orthonormal bases $(e_i)$ of $T_x\gM$ and $(E_i(\eps))$ of $T_{\gamma(\eps)}\gM$. Its determinant is thus
  \begin{align*}
    \det\left(\frac{1}{\eps}\rmJ(\eps)\right) &= \det\left(I_m - \frac{\eps^2}{6} \rmR(0) + O(\eps^3)\right) \\
    &= 1 - \frac{\eps^2}{6} \operatorname{tr}(\rmR(0)) + O(\eps^3) \\
    &= 1 - \frac{\eps^2}{6} \RicF(v) + O(\eps^3),
  \end{align*}
  where $\RicF(v)$ is the Ricci curvature of $\FF$ in the direction $v$, \ie the trace of $R_v$ in a basis which is orthonormal for $g(x,v)$. Since $e_m = v/\FF(x,v)$ and $R_v v = 0$, only the $m-1$ transverse directions contribute to this trace.

  Going back to the original basis, the determinant in $(\star\star)$ is taken in the coordinate bases of $T_x\gM$ and $T_{\gamma(\eps)}\gM$, and not in the orthonormal ones. Let $\rmP$ and $\rmQ$ be the matrices whose columns are the coordinates of $e_1,\ldots,e_m$ and of $E_1(\eps),\ldots,E_m(\eps)$. The matrix of $d(\expF_x)_{\eps v}$ in the coordinate bases is $\rmQ (\rmJ(\eps)/\eps) \rmP^{-1}$, so that
  \begin{equation*}
    \det\big(d(\expF_x)_{\eps v}\big) = \frac{\det\rmQ}{\det\rmP} \det\left(\frac{1}{\eps}\rmJ(\eps)\right).
  \end{equation*}
  The orthonormality of the two bases reads $\rmP^\top g(x,v) \rmP = I_m$ and $\rmQ^\top g(\gamma(\eps),\dot\gamma(\eps)) \rmQ = I_m$, so that $\det\rmP = \det g(x,v)^{-1/2}$ and $\det\rmQ = \det g(\gamma(\eps),\dot\gamma(\eps))^{-1/2}$. Both are positive when $(e_i)$ is positively oriented, since the determinant of the coordinates of $(E_i(t))$ is then continuous in $t$, never vanishes, and equals $\det\rmP$ at $t = 0$. Hence
  \begin{equation*}
    \det\big(d(\expF_x)_{\eps v}\big) = \frac{\sqrt{\det g(x,v)}}{\sqrt{\det g(\gamma(\eps),\dot\gamma(\eps))}}\left(1 - \frac{\eps^2}{6}\RicF(v) + O(\eps^3)\right).
  \end{equation*}

  \emph{Conclusion.} Multiplying the expansions of $(\star)$ and $(\star\star)$, the two ratios of $\sqrt{\det g}$ cancel exactly, and
  \begin{equation*}
    \Jac(\eps v) = \sigBH(x)\left(1 - \eps S(x,v) + \frac{\eps^2}{2}\left(S(x,v)^2 - \dot S(x,v)\right) + O(\eps^3)\right)\left(1 - \frac{\eps^2}{6}\RicF(v) + O(\eps^3)\right).
  \end{equation*}
  Expanding the product, the cross term $\eps S(x,v)\cdot\frac{\eps^2}{6}\RicF(v)$ is of order $\eps^3$, so that
  \begin{equation*}
    \Jac(\eps v) = \sigBH(x)\left(1 - \eps S(x,v) + \frac{\eps^2}{2}\left( S(x,v)^2 - \dot{S}(x,v) - \frac{1}{3}\RicF (v) \right) + \bigO\big(\eps^3\big)\right).
  \end{equation*}
 This concludes the proof.
\end{proof}

\subsubsection*{Proof of the moment expansion of the kernel}

We can now prove the moment expansion of the kernel of \cref{thm:moment_expansion}. The proof is a direct application of the previous lemmas. We recall that it states that, for any $f:\gM\rightarrow \R$ smooth enough,
\begin{align*}
  \gG_\FF [f] (x) = m_0 f + \eps \gA_1[f](x) + \frac{\eps^2}{2} \gA_2[f](x) + \bigO(\eps^3).
\end{align*}

\begin{proof}
The proof utilizes \cref{lemma:integral_localization} and both expansions of $f(\expF_x(\eps v))$ and $\Jac( \eps v)$ in $\eps$ (\cref{lemma:f_expansion,lemma:jacobian_expansion}). Putting those together, it holds that:
\begin{align*}
  \gG_\FF [f] (x) &= \int_{\tball_x(r_0 / \eps)} K(\FF(x,v)) f(\expF_x(\eps v)) \Jac( \eps v)  \diff\leb_x(v) + O(\eps^3) \\
  &= \sigBH(x) \int_{\tball_x(r_0 / \eps)} K(\FF(x,v)) \bigg[ f(x) + \eps \partial_i f(x) v^i \\
  & \quad + \frac{\eps^2}{2} \left(\partial_j \partial_i f(x) - \partial_k f(x) \ChrisF^k_{ij}(x, v) \right)v^j v^i + O(\eps^3) \bigg] \\
  & \qquad \times \bigg[1 - \eps S(x,v) + \frac{\eps^2}{2} \left( S(x,v)^2 - \dot S(x,v) - \frac{1}{3} \RicF (x, v) \right) + O(\eps^3)\bigg] \diff\leb_x(v) + O(\eps^3).
\end{align*}
Expanding the product and keeping the terms up to order $\eps^2$, the integrand is $K(\FF(x,v))$ times
\begin{align*}
  f(x) &+ \eps \left( \partial_i f(x) v^i - S(x,v) f(x) \right) \\
  &+ \frac{\eps^2}{2} \Bigg[ \partial_i \partial_j f(x) v^i v^j - \partial_k f(x) \ChrisF^k_{ij}(x,v) v^i v^j - 2 S(x,v) \partial_k f(x) v^k \\
  & \qquad + \left( S(x,v)^2 - \dot S(x,v) - \frac{1}{3} \RicF(x,v) \right) f(x) \Bigg],
\end{align*}
up to terms of order $\eps^3$ and higher. The moments of \cref{def:kernel_moments} are integrals over the whole tangent space $T_x\gM$, whereas this integral is over $\tball_x(r_0/\eps)$. As in \cref{lemma:integral_localization}, the sub-exponential decay of $K$ (A1) makes the integral over $T_x\gM \setminus \tball_x(r_0/\eps)$ exponentially small, so that the domain can be extended to $T_x\gM$ up to $O(\eps^3)$. By collecting the terms of order $0$, $\eps$ and $\eps^2$, and including back the density $\sigBH(x)$, we recover the definitions of the moments (\cref{def:kernel_moments}) and thus
\begin{align*}
  \gG_\FF [f] (x) &=  m_0(x) f(x) + \eps \left( m_1^i (x) \partial_i f(x) - s_0(x) f(x) \right) \\
  & \qquad + \frac{\eps^2}{2} \Big(m_2^{ij} (x) \partial_i \partial_j f(x) - \left(\Lambda^k(x) + 2 s_1^k(x)\right) \partial_k f(x) + \Upsilon(x) f(x) \Big) + \bigO(\eps^3).
\end{align*}
Hence to recover the operators $\gA_1$ and $\gA_2$ of \cref{thm:moment_expansion}, we define:
\begin{align*}
  \gA_1[f](x) &= m_1^i (x) \partial_i f(x) - s_0(x) f(x), \\
  \gA_2[f](x) &= m_2^{ij} (x) \partial_i \partial_j f(x) - \left(\Lambda^k(x) + 2 s_1^k(x)\right) \partial_k f(x) + \Upsilon(x) f(x).
\end{align*}
\end{proof}

\subsection{Proof of \texorpdfstring{\cref{thm:sym_limit}}{}} \label{appendix:proofs:prop:sym_limit}

This subsection is devoted to the proof of \cref{thm:sym_limit}, which builds upon the previous expansion of the kernel. To do so, we first prove some properties of the reverse Finsler metric and its nonlinear connection.

\subsubsection*{Lemmas on the reverse Finsler metric and nonlinear connection}

Some of the following properties are stated in \citet[Section 2.5]{ohtaComparisonFinslerGeometry2021}, but are not proved there. We provide here proofs for completeness.

\begin{proposition}[Geodesic spray of the reverse Finsler metric] \label{property:spray_reverse}
  Let $\backFF(x,v) = \FF(x, -v)$, and let $G^k$ and $\backG^k$ be the geodesic spray coefficients of $\FF$ and $\backFF$, respectively. Then, we have that $\backG^k(x,v) = G^k(x,-v)$ for all $v \in T_x\gM$.
\end{proposition}
\begin{proof}
  We first recall the relevant notations. For $x \in \gM$ and $v \in T_x\gM$, let $g_{ij}(x,v)$ and $\backg_{ij}(x,v)$ be the fundamental tensors of $\FF$ and $\backFF$, respectively. By \cref{def:fundamental_tensor}, we have that:
  \begin{equation*}
    g_{ij}(x,v) = \frac{1}{2} \frac{\partial^2 \FF^2(x,v)}{\partial v^i \partial v^j} \quad \textnormal{and} \quad \backg_{ij}(x,v) = \frac{1}{2} \frac{\partial^2 \backFF^2(x,v)}{\partial v^i \partial v^j}.
  \end{equation*}
  And the geodesic spray coefficients \citep{ohtaComparisonFinslerGeometry2021} are given by:
  \begin{align*}
    G^k(x,v) &= \frac{1}{4} g^{kl}(x,v) \left( \frac{\partial^2 \FF^2(x,v)}{\partial x^i \partial v^l} v^i - \frac{\partial \FF^2(x,v)}{\partial x^l} \right), \\
    \backG^k(x,v) &= \frac{1}{4} \backg^{kl}(x,v) \left( \frac{\partial^2 \backFF^2(x,v)}{\partial x^i \partial v^l} v^i - \frac{\partial \backFF^2(x,v)}{\partial x^l} \right).
  \end{align*}
  To start, we show that the fundamental tensor of the reverse Finsler metric is related to the fundamental tensor of the original Finsler metric by $\backg_{ij}(x,v) = g_{ij}(x,-v)$. Indeed:
  \begin{align*}
    \backg_{ij}(x,v) &= \frac{1}{2} \frac{\partial^2 \backFF^2(x,v)}{\partial v^i \partial v^j} = \frac{1}{2} \frac{\partial^2 \FF^2(x,-v)}{\partial v^i \partial v^j} = g_{ij}(x,-v).
  \end{align*}
  This implies that $\backg^{kl}(x,v) = g^{kl}(x,-v)$. Similarly,
  \begin{equation*}
    \frac{\partial^2 \backFF^2(x,v)}{\partial x^i \partial v^l} = -\frac{\partial^2 \FF^2(x,-v)}{\partial x^i \partial v^l} \quad \textnormal{and} \quad \frac{\partial \backFF^2(x,v)}{\partial x^l} = \frac{\partial \FF^2(x,-v)}{\partial x^l}.
  \end{equation*}
  Putting everything together, we have that:
  \begin{align*}
    \backG^k(x,v) &= \frac{1}{4} g^{kl}(x,-v) \left( -\frac{\partial^2 \FF^2(x,-v)}{\partial x^i \partial v^l} v^i - \frac{\partial \FF^2(x,-v)}{\partial x^l} \right) \\
    &= \frac{1}{4} g^{kl}(x,-v) \left( \frac{\partial^2 \FF^2(x,-v)}{\partial x^i \partial v^l} (-v)^i - \frac{\partial \FF^2(x,-v)}{\partial x^l} \right) \\
    &= G^k(x,-v).
  \end{align*}
  This concludes the proof.
\end{proof}

Next, we define the nonlinear connection of a Finsler metric, which is a key quantity in Finsler geometry. It is used to define the curvature tensor and the Ricci curvature of a Finsler metric, which appeared in \cref{thm:moment_expansion}.
\begin{definition}[Nonlinear connection] \label{def:nonlinear_connection}
  The nonlinear connection of a Finsler metric $\FF$ is defined as:
  \begin{equation*}
    N^i_j(x,v) = \frac{\partial G^i}{\partial v^j}(x,v).
  \end{equation*}
\end{definition}

\begin{lemma}[Nonlinear connection under reversal] \label{lemma:N_reversal}
  Let $\backN^i_j$ and $N^i_j$ be the nonlinear connections of $\backFF$ and $\FF$ respectively. Then $\backN^i_j(x,v) = -N^i_j(x,-v)$.
\end{lemma}
\begin{proof}
  Using \cref{property:spray_reverse} ($\backG^i(x,v) = G^i(x,-v)$) and the chain rule,
  \begin{align*}
    \backN^i_j(x,v) &= \frac{\partial \backG^i}{\partial v^j}(x,v)
    = \frac{\partial}{\partial v^j}\left[G^i(x,-v)\right]
    = - \frac{\partial G^i}{\partial v^j}(x,-v)
    = - N^i_j(x,-v). \qedhere
  \end{align*}
\end{proof}
This allows us to prove the following property of the Ricci curvature of the reverse Finsler metric.
\begin{proposition}[Ricci curvature of the reverse Finsler metric] \label{property:ricci_reverse}
  Let $\backRicF$ and $\RicF$ be the Ricci curvatures of $\backFF$ and $\FF$ respectively. Then $\backRicF(x,v) = \RicF(x,-v)$.
\end{proposition}
\begin{proof}
  We first introduce a few auxiliary quantities. The curvature tensor of a Finsler metric $\FF$ is defined as
  \begin{align*}
    R^i_j(x,v) =  \frac{\partial G^i}{\partial x^j}(x,v)
                  -
                  \sum_{k=1}^m \left(
                      \frac{\partial N^i_j}{\partial x^k }(x,v) v^k -
                      \frac{\partial N^i_j}{\partial v^k }(x,v) G^k(x,v)
                      \right)
                  -
                  \sum_{k=1}^m N^i_k(x,v) N^k_j(x,v),
  \end{align*}
  inducing a linear endomorphism $R_v(x, \cdot): T_x\gM \to T_x\gM$ defined as $R_v(x,w) = \sum_{i,j=1}^m R^i_j(x,v) w^j \frac{\partial}{\partial x^i}\big\rvert_{x}$. The Ricci curvature is the trace of this endomorphism, for $v \neq 0$:
  \begin{align*}
    \RicF(x,v) = \sum_{i=1}^{m-1} g(x,v)_{kl} R_v(x,e_i)^k e_i^l,
  \end{align*}
  where $\mathcal{B}=\{ v/\FF(x,v)\} \cup \{e_i\}_{i=1}^{m-1} \subset T_x\gM$ is an orthonormal basis with respect to the fundamental tensor $g(x,v)$ of $\FF$. We similarly define $\backR^i_j(x,v)$, $\backR_v(x,w)$ and $\backRicF(x,v)$ for the reverse Finsler metric $\backFF$.

  By \cref{lemma:N_reversal}, $\backN^i_j(x,v) = -N^i_j(x,-v)$. Differentiating this identity in $x^k$ (which passes through unchanged, as in the proof of \cref{property:spray_reverse}) and in $v^k$ (which picks up an extra sign, as in the proof of \cref{lemma:N_reversal} itself) gives
  \begin{equation*}
    \frac{\partial \backN^i_j}{\partial x^k}(x,v) = -\frac{\partial N^i_j}{\partial x^k}(x,-v), \qquad
    \frac{\partial \backN^i_j}{\partial v^k}(x,v) = \frac{\partial N^i_j}{\partial v^k}(x,-v).
  \end{equation*}
  Together with $\backG^k(x,v) = G^k(x,-v)$ (\cref{property:spray_reverse}), we get
  \begin{align*}
    \frac{\partial \backN^i_j}{\partial x^k }(x,v) v^k -
                      \frac{\partial \backN^i_j}{\partial v^k }(x,v) \backG^k(x,v)
                      &= -\frac{\partial N^i_j}{\partial x^k }(x,-v) v^k -
                      \frac{\partial N^i_j}{\partial v^k }(x,-v) G^k(x,-v) \\
                      &= \frac{\partial N^i_j}{\partial x^k }(x,-v) (-v)^k -
                      \frac{\partial N^i_j}{\partial v^k }(x,-v) G^k(x,-v),
  \end{align*}
  and $\sum_k \backN^i_k(x,v)\backN^k_j(x,v) = \sum_k N^i_k(x,-v) N^k_j(x,-v)$ (product of two sign-flipped terms). By \cref{property:spray_reverse}:
  \begin{equation*}
    \frac{\partial \backG^i}{\partial x^j}(x,v) = \frac{\partial G^i}{\partial x^j}(x,-v).
  \end{equation*}
  Putting everything together, we have that
  \begin{align*}
    \backR^i_j(x,v) &= \frac{\partial G^i}{\partial x^j}(x,-v) - \sum_{k=1}^m \left(
                      \frac{\partial N^i_j}{\partial x^k }(x,-v) (-v)^k -
                      \frac{\partial N^i_j}{\partial v^k }(x,-v) G^k(x,-v)
                    \right) \\
                    &\qquad - \sum_{k=1}^m N^i_k(x,-v) N^k_j(x,-v)
                    = R^i_j(x,-v),
  \end{align*}
  so that $\backR_v(x,w) = R_{-v}(x,w)$. Since the fundamental tensor of $\backFF$ is $\backg(x,v) = g(x,-v)$ (\cref{property:spray_reverse}'s proof), a basis orthonormal for $g(x,-v)$ is also orthonormal for $\backg(x,v)$, so we may use the same basis $\mathcal B$ for both. Hence
  \begin{align*}
    \backRicF(x,v) &= \sum_{i=1}^{m-1} \backg(x,v)_{kl} \backR_v(x,e_i)^k e_i^l
    = \sum_{i=1}^{m-1} g(x,-v)_{kl} R_{-v}(x,e_i)^k e_i^l
    = \RicF(x,-v). \qedhere
  \end{align*}
\end{proof}

Finally, we need the behavior of the S-curvature under reversal. Since $-\tball_x$ and $\tball_x$ have the same volume\footnote{The unit ball of $\backFF$ at $x$ is $-\tball_x$, and the linear map $v \mapsto -v$ has determinant $\pm 1$, so that it preserves the Lebesgue measure $\leb_x$. Hence $\leb_x(-\tball_x) = \leb_x(\tball_x)$ and, by \cref{eq:busemann_haussdorf_measure}, $\FF$ and $\backFF$ have the same density $\sigBH$.}, $\FF$ and $\backFF$ share the Busemann--Hausdorff measure, and we denote by $\backtau$, $\backS$ and $\dot{\backS}$ the distortion, the S-curvature and its derivative along geodesics of $\backFF$, with respect to $\mBH$.

\begin{lemma}[S-curvature under reversal] \label{lemma:S_reversal}
  For every $x \in \gM$ and $v \in T_x\gM \setminus \{0\}$, we have that $\backtau(x,v) = \tau(x,-v)$, $\backS(x,v) = -S(x,-v)$ and $\dot{\backS}(x,v) = \dot S(x,-v)$.
\end{lemma}
\begin{proof}
  The distortion only involves the fundamental tensor and the density $\sigBH$. Since $\backg_{ij}(x,v) = g_{ij}(x,-v)$ (see the proof of \cref{property:spray_reverse}) and $\sigBH$ is shared by $\FF$ and $\backFF$, we get $\backtau(x,v) = \tau(x,-v)$.

  Let $\gamma$ be the geodesic of $\FF$ with $\gamma(0) = x$ and $\dot\gamma(0) = -v$, defined on some interval $(-\delta, \delta)$, and let $\bar\gamma(t) = \gamma(-t)$, so that $\bar\gamma(0) = x$ and $\dot{\bar\gamma}(0) = v$. Using the geodesic equation of $\FF$ and \cref{property:spray_reverse}, we have that
  \begin{equation*}
    \ddot{\bar\gamma}^k(t) = \ddot\gamma^k(-t) = -2 G^k\big(\gamma(-t), \dot\gamma(-t)\big) = -2 G^k\big(\bar\gamma(t), -\dot{\bar\gamma}(t)\big) = -2 \backG^k\big(\bar\gamma(t), \dot{\bar\gamma}(t)\big),
  \end{equation*}
  so that $\bar\gamma$ is the geodesic of $\backFF$ with initial conditions $(x,v)$. Hence, for $t \in (-\delta, \delta)$,
  \begin{equation*}
    \backtau\big(\bar\gamma(t), \dot{\bar\gamma}(t)\big) = \tau\big(\bar\gamma(t), -\dot{\bar\gamma}(t)\big) = \tau\big(\gamma(-t), \dot\gamma(-t)\big).
  \end{equation*}
  Differentiating once and twice at $t = 0$ gives $\backS(x,v) = -S(x,-v)$ and $\dot{\backS}(x,v) = \dot S(x,-v)$.
\end{proof}

\begin{lemma} \label{property:moments_summary}
  The moments of the kernel for the reverse Finsler metric are related to the moments of the kernel for the Finsler metric. In particular, for $n \in \N$ we have that:
  \begin{align*}
    &\backm_n(x) = (-1)^n m_n(x), \qquad \backLambda^k(x) = \Lambda^k(x), \\
    &\backs_n(x) = (-1)^{n+1} s_n(x), \qquad \backUpsilon(x) = \Upsilon(x).
  \end{align*}
\end{lemma}
\begin{proof}

Since $\FF$ and $\backFF$ share the density $\sigBH$, which does not depend on $v$, it factors out of all the integrals below, and we omit it. First, we show the relation between $\backm_n$ and $m_n$. By the change of variable $u = -v$,
\begin{align*}
  \backm_n(x) &= \int_{T_x\gM} K(\backFF(x,v)) v^{\otimes n} \diff\leb_x(v) \\
  &= \int_{T_x\gM} K(\FF(x,-u)) (-u)^{\otimes n} \diff\leb_x(u) \\
  &= (-1)^n \int_{T_x\gM} K(\FF(x,u)) u^{\otimes n} \diff\leb_x(u) \\
  &= (-1)^n m_n(x).
\end{align*}
The change of variable $u = -v$ has Jacobian determinant $\det(-I_m) = (-1)^m$, whose absolute value is $1$, so that it preserves $\leb_x$, and the sign factor $(-1)^n$ comes only from the $v^{\otimes n}$ term. This argument is also used below for the other terms.

It remains to treat the $\backLambda^k(x)$, $\backs_n(x)$ and $\backUpsilon(x)$ terms. For $\backLambda^k(x)$, recall that $\backChrisF^k_{ij}(v)v^i v^j = 2 \backG^k(x,v)$ and $\ChrisF^k_{ij}(x, v) v^i v^j = 2 G^k(x,v)$. Using $\backG^k(x,v) = G^k(x,-v)$ (\cref{property:spray_reverse}),
\begin{align*}
  \backLambda^k(x) &=  2 \int_{T_x\gM} K(\backFF(x,v)) \backG^k(x,v) \diff\leb_x(v) \\
  &= 2 \int_{T_x\gM} K(\FF(x,-v)) G^k(x,-v) \diff\leb_x(v) \\
  &= 2 \int_{T_x\gM} K(\FF(x,u)) G^k(x,u) \diff\leb_x(u) & \text{by change of variable $u = -v$} \\
  &= \Lambda^k(x).
\end{align*}
For the S-curvature terms, using \cref{lemma:S_reversal} and the same change of variable,
\begin{align*}
  \backs_n(x) &= \int_{T_x\gM} K(\backFF(x,v)) \backS(x,v) v^{\otimes n} \diff\leb_x(v) = - \int_{T_x\gM} K(\FF(x,-v)) S(x,-v) v^{\otimes n} \diff\leb_x(v) \\
  &= (-1)^{n+1} \int_{T_x\gM} K(\FF(x,u)) S(x,u) u^{\otimes n} \diff\leb_x(u) = (-1)^{n+1} s_n(x).
\end{align*}
Finally, for $\backUpsilon(x)$, \cref{lemma:S_reversal} gives $\backS(x,v)^2 - \dot{\backS}(x,v) = S(x,-v)^2 - \dot S(x,-v)$, and \cref{property:ricci_reverse} gives $\backRicF(x,v) = \RicF(x,-v)$, so that, by the change of variable $u = -v$,
\begin{align*}
  \backUpsilon(x) &= \int_{T_x\gM} K(\backFF(x,v)) \left( \backS(x,v)^2 - \dot{\backS}(x,v) - \frac{1}{3} \backRicF(x,v) \right) \diff\leb_x(v) \\
  &= \int_{T_x\gM} K(\FF(x,-v)) \left( S(x,-v)^2 - \dot S(x,-v) - \frac{1}{3} \RicF(x,-v) \right) \diff\leb_x(v) \\
  &= \int_{T_x\gM} K(\FF(x,u)) \left( S(x,u)^2 - \dot S(x,u) - \frac{1}{3} \RicF(x,u) \right) \diff\leb_x(u) \\
  &= \Upsilon(x).
\end{align*}
Putting everything together, we have shown that $\backm_n(x) = (-1)^n m_n(x)$, $\backLambda^k(x) = \Lambda^k(x)$, $\backs_n(x) = (-1)^{n+1} s_n(x)$ and $\backUpsilon(x) = \Upsilon(x)$, which concludes the proof.
\end{proof}

\textbf{Recap.} The goal of this subsection was to study the behavior of the reverse Finsler metric $\backFF$ and its associated transport operator $\gG_\backFF$. We have shown that the moments of the kernel for the reverse Finsler metric are related to the moments of the kernel for the Finsler metric, with $\backm_n(x) = (-1)^n m_n(x)$, $\backLambda^k(x) = \Lambda^k(x)$, $\backs_n(x) = (-1)^{n+1} s_n(x)$ and $\backUpsilon(x) = \Upsilon(x)$. This allows us to derive the moment expansion of the transport operator $\gG_\backFF$ in terms of the moments of the kernel for $\FF$. In particular, we have the following proposition.

\begin{proposition}\label{prop:reverse_moment_expansion}
  The transport operator $\gG_\backFF$ associated with the reverse Finsler metric $\backFF$ has the following moment expansion:
  \begin{align*}
    \gG_\backFF [f] =  & m_0 f - \eps \left( m_1^i \partial_i f - s_0 f \right) + \frac{\eps^2}{2} \left(m_2^{ij}  \partial_i \partial_j f - \left(\Lambda^k + 2 s_1^k\right) \partial_k f + \Upsilon f \right)  + O(\eps^3).
  \end{align*}
\end{proposition}
This follows from \cref{thm:moment_expansion} applied to $\backFF$, which shares the density $\sigBH$ with $\FF$, together with \cref{property:moments_summary}. Compared with the expansion of $\gG_\FF$, the terms of odd order in $\eps$ change sign under reversal, while the terms of even order do not.

\subsubsection*{Proof of the moment expansion of the symmetric and anti-symmetric transport operators}

Using \cref{prop:reverse_moment_expansion}, the moment expansions of the symmetric and antisymmetric parts $\gGs = (\gG_\FF + \gG_\backFF)/2$ and $\gGa = (\gG_\FF - \gG_\backFF)/2$ of the transport operator follow directly: up to order $\eps^2$, the symmetric part retains the even-order terms and the antisymmetric part the odd-order ones. This yields
\begin{align*}
  \gGs [f] &= m_0 f + \frac{\eps^2}{2} \left(m_2^{ij}  \partial_i \partial_j f - \left(\Lambda^k + 2 s_1^k\right) \partial_k f + \Upsilon f \right) + O(\eps^3) \quad \text{and} \\
  \gGa [f] &= \eps \left( m_1^i  \partial_i f - s_0 f \right) + O(\eps^3).
\end{align*}
It is important to note that, in practice, one does not observe the transport operator $\gG_\FF [f]$ directly, as points are rarely sampled uniformly from the manifold. Instead, one typically observes $\gG_\FF [\rho f]$ applied to the density $\rho$ of the sampling distribution. Looking at the expansion of $\gG_\FF [\rho f]$, we have:
\begin{align*}
  \gG_\FF [\rho f] = & m_0 \rho f + \eps \left( m_1^i  \partial_i (\rho f) - s_0 \rho f \right) \\
  & + \frac{\eps^2}{2} \left(m_2^{ij}  \partial_i \partial_j (\rho f) - \left(\Lambda^k + 2 s_1^k\right) \partial_k (\rho f) + \Upsilon \rho f \right) + O(\eps^3)
\end{align*}
so that every term is `polluted' by the density $\rho$. This calls for a normalization, which we now introduce.

\subsubsection*{$\theta$-normalized transport operator}

Our goal is now to define a $\theta$-normalized transport operator $\gGt$ that mitigates the impact of the sampling density $\rho$ on the limit operator. We define $\gGt$ as:
\begin{equation*}
  \gGt [f](x) = \frac{1}{\qe(x)^\theta} \gG_\FF \left[\frac{f}{\qe^\theta} \right](x),
\end{equation*}
where $\qe(x)$ is a normalization factor. Specifically, we take $\qe = (\din + \dout)/2$. This choice allows us to clean up the expansion of $\gGt [\rho f]$ and analyze its behavior as $\eps \to 0$. Since $\qe$ is invariant under reversal, the normalization $\qe(x)^\theta \qe(y)^\theta$ is symmetric in $(x,y)$, and the symmetric and antisymmetric parts of $\gGt$ are obtained from those of $\gG_\FF$ as
\begin{align*}
  \gGts [f] &= \frac{1}{\qe^\theta} \gGs \left[\frac{f}{\qe^\theta} \right] \quad \text{and} \quad
  \gGta [f] = \frac{1}{\qe^\theta} \gGa \left[\frac{f}{\qe^\theta} \right].
\end{align*}
We use $\qe$ rather than $\din$ or $\dout$ alone because its expansion has no term of order $\eps$, which makes it a cleaner proxy for the sampling density $\rho$. Indeed, the expansions of $\din$ and $\dout$ read
\begin{align*}
  \dout = \gG_\FF [\rho] &= m_0 \rho + \eps \left( m_1^i  \partial_i \rho - s_0 \rho \right) \\
  & \qquad + \frac{\eps^2}{2} \left(m_2^{ij}  \partial_i \partial_j \rho - \left(\Lambda^k + 2 s_1^k\right) \partial_k \rho + \Upsilon \rho \right) + O(\eps^3),\\
  \din = \gG_\backFF [\rho] &= m_0 \rho - \eps \left( m_1^i  \partial_i \rho - s_0 \rho \right) \\
  & \qquad + \frac{\eps^2}{2} \left(m_2^{ij}  \partial_i \partial_j \rho - \left(\Lambda^k + 2 s_1^k\right) \partial_k \rho + \Upsilon \rho \right) + O(\eps^3).
\end{align*}
Hence $\qe = m_0 \rho + O(\eps^2)$, whereas $\din$ and $\dout$ each carry a term of order $\eps$. Moreover, $m_0 = m \mu_0 \vole(\eball^m)$ does not depend on $x$ (\cref{appendix:proofs:prop:moment_expression_unitball}), so that $\qe$ is proportional to $\rho$ up to $O(\eps^2)$. We now define $\phit(x) = \rho(x) / \qe(x)^\theta$, and look at the expansion of the observable quantity $\gGt [\rho f]$:
\begin{align*}
  \gGt [\rho f] &= \frac{1}{\qe^\theta} \gG_\FF [\phit f]  \\
    &= \frac{1}{\qe^\theta} \Big[ m_0 \phit f + \eps \left( m_1^i  \partial_i (\phit f) - s_0 \phit f \right) \\
    & \qquad + \frac{\eps^2}{2} \left(m_2^{ij}  \partial_i \partial_j (\phit f) - \left(\Lambda^k + 2 s_1^k\right) \partial_k (\phit f) + \Upsilon \phit f \right) + O(\eps^3) \Big].
\end{align*}
Expanding the derivatives with the Leibniz rule gives
\begin{align*}
  \partial_i (\phit f) &= \phit \partial_i f + f \partial_i \phit, \\
  m_2^{ij} \partial_i \partial_j (\phit f) &= m_2^{ij} \phit \partial_i \partial_j f + 2 m_2^{ij} \partial_i \phit \partial_j f + m_2^{ij} f \partial_i \partial_j \phit ,
\end{align*}
where we used the fact that $m_2^{ij} = m_2^{ji}$. Plugging these into the expansion of $\gGt [\rho f]$, we obtain:
\begin{align}
  \gGt [\rho f] &= \frac{1}{\qe^\theta} \Bigg[
    m_0 \phit f + \eps \left( m_1^i \left(\phit \partial_i f + f \partial_i \phit\right) - s_0 \phit f \right) \nonumber \\
    & \quad + \frac{\eps^2}{2} \bigg(m_2^{ij}(\phit \partial_i \partial_j f + 2 \partial_i \phit \partial_j f +  f \partial_i \partial_j \phit)  - \left(\Lambda^k + 2 s_1^k\right) \left(\phit \partial_k f + f \partial_k \phit\right) \label{eq:Gt_expansion}  \\
    & \qquad + \Upsilon \phit f \bigg) + O(\eps^3) \nonumber
  \Bigg].
\end{align}

We then define the symmetric and antisymmetric normalized transport operators $\Pts$ and $\Pta$ as
\begin{align}
  \gP^{(\theta, \bullet)} [f] = \frac{1}{\dts} \left( \gG^{(\theta, \bullet)}[\rho f] - f \gG^{(\theta, \bullet)}[\rho] \right) \label{eq:gP_theta}
\end{align}
for $\bullet \in \{s, a\}$, where $\dts = \gGts[\rho]$ is the symmetric normalized degree defined below. This expression centers the operators $\gG$ around their respective degrees, so that $\gP$ vanishes on constant functions.

\paragraph{Expansion for the degrees of the $\theta$-normalized operator.} The in- and out-degrees $\dint$ and $\doutt$ of the $\theta$-normalized kernel expand as
\begin{align*}
  \doutt = \gGt [\rho] &= \frac{1}{\qe^\theta} \Big(m_0 \phit + \eps \left( m_1^i \partial_i \phit - s_0 \phit \right) \\
  & \qquad + \frac{\eps^2}{2} \left(m_2^{ij} \partial_i \partial_j \phit - \left(\Lambda^k + 2 s_1^k\right) \partial_k \phit + \Upsilon \phit \right) \Big) + O(\eps^3),
\end{align*}
and $\dint = \backgGt [\rho]$ has the same expansion, with the opposite sign in front of the term of order $\eps$. By definition, the degrees of the symmetric and antisymmetric normalized kernels are
\begin{align*}
  \dts &= \gGts [\rho] = \frac{1}{2} \left(\gGt[\rho] + \backgGt[\rho]\right) = \frac{\doutt + \dint}{2}, \\
  \dta &= \gGta [\rho] = \frac{1}{2} \left(\gGt[\rho] - \backgGt[\rho]\right) = \frac{\doutt - \dint}{2}.
\end{align*}
Plugging the expansions of $\dint$ and $\doutt$ into the above expressions, we have that:
\begin{align}
  \dts &=  \frac{1}{\qe^\theta} \left(m_0 \phit + \frac{\eps^2}{2} \left(m_2^{ij} \partial_i \partial_j \phit - \left(\Lambda^k + 2 s_1^k\right) \partial_k \phit + \Upsilon \phit \right) \right) + O(\eps^3), \label{eq:dts_expansion} \\
  \dta &= \frac{1}{\qe^\theta} \eps \left( m_1^i \partial_i \phit - s_0 \phit \right) + O(\eps^3). \label{eq:dta_expansion}
\end{align}
Using the fact that $1/(1+x) = 1 - x + O(x^2)$, and noting that the bracket of order $\eps^2$ in \cref{eq:dts_expansion} is $\gA_2[\phit]$ (\cref{def:kernel_moments}), we can write the following expansion for $1 / \dts$:
\begin{align}
  \frac{1}{\dts} &= \left(\frac{m_0 \phit}{\qe^\theta} \right)^{-1} \left(
    1 + \frac{\eps^2}{2 m_0 \phit} \gA_2[\phit] + O(\eps^3)
  \right)^{-1} \nonumber \\
  &=\frac{\qe^\theta}{m_0 \phit} \left(1 - \frac{\eps^2}{2 m_0 \phit} \gA_2[\phit] + O(\eps^3)\right) \nonumber \\
  &= \frac{\qe^\theta}{m_0 \phit} + O(\eps^2) \label{eq:symmetric_normalized_degree_expansion}
\end{align}
We now expand the symmetric and antisymmetric normalized transport operators $\Pts$ and $\Pta$ of \cref{eq:gP_theta}.

\paragraph{Expansion for the symmetric part.} By putting together \cref{eq:Gt_expansion} and \cref{eq:dts_expansion}, we have that:
\begin{align*}
  \gGts [\rho f] - f \dts &=
  \frac{\eps^2}{2 \qe^\theta} \Bigg(
    m_2^{ij} \left(\phit \partial_i \partial_j f + 2  \partial_i \phit \partial_j f +  f \partial_i \partial_j \phit - f\partial_i\partial_j \phit\right) \\
    &\quad - \left(\Lambda^k + 2 s_1^k\right) \left(\phit \partial_k f + f \partial_k \phit - f \partial_k \phit\right) \\
    &\quad + \Upsilon \left(\phit f - \phit f\right)
    \Bigg) + O(\eps^3)\\
  &=\frac{\eps^2}{2 \qe^\theta} \Bigg(
    m_2^{ij} \left(\phit \partial_i \partial_j f + 2  \partial_i \phit \partial_j f\right) -  \phit \left(\Lambda^k + 2 s_1^k\right) \partial_k f
    \Bigg) + O(\eps^3).
\end{align*}
Hence, by \cref{eq:symmetric_normalized_degree_expansion}, we have that:
\begin{align*}
  \Pts [f] &= \frac{1}{\dts} \left(\gGts [\rho f] - f \dts\right) \\
  &= \frac{\qe^\theta}{m_0 \phit} \frac{\eps^2}{2 \qe^\theta} \Bigg(
    m_2^{ij} \left(\phit \partial_i \partial_j f + 2  \partial_i \phit \partial_j f\right) -  \phit \left(\Lambda^k + 2 s_1^k\right) \partial_k f
    \Bigg) + O(\eps^3) \\
  &= \frac{\eps^2}{2 m_0} \Bigg(
    m_2^{ij} \left( \partial_i \partial_j f + 2 \partial_i \log \phit \partial_j f\right) -  \left(\Lambda^k + 2 s_1^k\right) \partial_k f
    \Bigg) + O(\eps^3).
\end{align*}
where we used $\partial_i \log \phit = \partial_i \phit / \phit$. Note that the curvature term $\Upsilon$, which only multiplies $f$, is removed by the centering.

\paragraph{Expansion for the anti-symmetric part.} By putting together \cref{eq:Gt_expansion} and \cref{eq:dta_expansion}, we have that:
\begin{align*}
  \gGta [\rho f] - f \dta &=
  \frac{\eps}{\qe^\theta} \Big( m_1^i \left(\phit \partial_i f + f \partial_i \phit\right) - s_0 \phit f - f \left( m_1^i \partial_i \phit - s_0 \phit \right) \Big) + O(\eps^3)\\
  &=\frac{\eps}{\qe^\theta} m_1^i \phit \partial_i f + O(\eps^3).
\end{align*}
In particular, the term $s_0$ brought by the S-curvature only multiplies $f$, and is removed by the centering. Hence, by \cref{eq:symmetric_normalized_degree_expansion}, we have that:
\begin{align*}
  \Pta [f] &= \frac{1}{\dts} \left(\gGta [\rho f] - f \dta\right) \\
  &= \frac{\qe^\theta}{m_0 \phit} \frac{\eps}{\qe^\theta} m_1^i \phit \partial_i f + O(\eps^3) \\
  &= \frac{\eps}{m_0} m_1^i \partial_i f + O(\eps^3).
\end{align*}

 Consequently, if we introduce the normalized quantities $\tilde{m}_n = m_n / m_0$, $\tilde{\Lambda}^k = \Lambda^k / m_0$ and $\tilde{s}_1^k = s_1^k / m_0$, we have that:
\begin{align*}
  \Pts [f] &= \frac{\eps^2}{2} \Big(
    \tilde{m}_2^{ij} \left( \partial_i \partial_j f + 2 \partial_i \log \phit \partial_j f\right) -  \left(\tilde{\Lambda}^k + 2 \tilde{s}_1^k\right) \partial_k f
    \Big) + O(\eps^3),\\
  \Pta [f] &= \eps \tilde{m}_1^i \partial_i f + O(\eps^3).
\end{align*}

\paragraph{Impact of the normalization parameter.} The antisymmetric normalized transport operator $\Pta$ depends neither on the sampling density $\rho$ nor on the normalization parameter $\theta$. We thus focus on the impact of the normalization parameter $\theta$ on the symmetric normalized transport operator $\Pts$. This normalization appears through the term $\partial_i \log \phit$, by definition of $\phit = \rho / \qe^\theta$. We have that:
\begin{align*}
  \partial_i \log \phit &= \partial_i \log \left(\frac{\rho}{\qe^\theta}\right) = \partial_i \log \rho - \theta \partial_i \log \qe.
\end{align*}
Moreover, we have seen that $\qe = m_0 \rho + O(\eps^2)$, where $m_0 = m \mu_0 \vole(\eball^m)$ does not depend on $x$ (\cref{appendix:proofs:prop:moment_expression_unitball}), so that $\partial_i \log \qe = \partial_i \log \rho + O(\eps^2)$. Hence, we have that:
\begin{align*}
  \partial_i \log \phit &= (1 - \theta) \partial_i \log \rho + O(\eps^2).
\end{align*}
The leading-order term in the expansion of $\Pts$ is thus
\begin{align*}
  \Pts [f] &= \frac{\eps^2}{2} \Big(
    \tilde{m}_2^{ij} \left( \partial_i \partial_j f + 2 (1 - \theta) \partial_i \log \rho \partial_j f\right) -  \left(\tilde{\Lambda}^k + 2 \tilde{s}_1^k\right) \partial_k f
    \Big) + O(\eps^3).
\end{align*}
In particular, for $\theta = 1$ the sampling density disappears from the limit, exactly as in the Riemannian setting \citep{coifman_diffusion_2006}. Taking the limits for both operators when $\eps \to 0$ yields $\gL^a f = \tilde{m}_1^i \partial_i f$, and
\begin{equation} \label{eq:pf_general_limit}
  \gL^s f = \frac{1}{2} \tilde{m}_2^{ij} \partial_i \partial_j f + (1-\theta) \tilde{m}_2^{ij} \partial_i \log \rho \partial_j f - \frac{1}{2} \big(\tilde{\Lambda}^k + 2 \tilde{s}_1^k\big) \partial_k f.
\end{equation}

\paragraph{Divergence form of the symmetric limit.} It remains to write \cref{eq:pf_general_limit} in divergence form. This relies on the following identities between the coefficients of \cref{thm:moment_expansion}.

\begin{lemma}[Divergence identities] \label{lemma:moment_identities}
  For every $x \in \gM$, the kernel moments of \cref{def:kernel_moments} satisfy
  \begin{equation*}
    s_0 = -\frac{1}{2} \operatorname{div}_{BH}(m_1) \qquad \text{and} \qquad \Lambda^k + 2 s_1^k = - \frac{1}{\sigBH} \partial_j \left( \sigBH m_2^{jk} \right),
  \end{equation*}
  where $\operatorname{div}_{BH} X = \sigBH^{-1} \partial_i(\sigBH X^i)$. Equivalently, the operators of \cref{thm:moment_expansion} read
  \begin{equation*}
    \gA_1 f = m_1^i \partial_i f + \frac{1}{2} \operatorname{div}_{BH}(m_1) f, \qquad \gA_2 f = \frac{1}{\sigBH} \partial_j \left( \sigBH m_2^{jk} \partial_k f \right) + \Upsilon f,
  \end{equation*}
  so that $\gA_1$ is skew-adjoint and $\gA_2$ is self-adjoint on $L^2(\gM, \mBH)$.
\end{lemma}
\begin{proof}
  We work in a chart, write $\sigma = \sigBH$, and $M_n(x) = \int_{T_x\gM} K(\FF(x,v)) v^{\otimes n} \diff\leb_x(v)$, so that $m_n = \sigma M_n$. The proof relies on two classical facts about the geodesic spray coefficients $G^k$ of $\FF$ (\cref{property:spray_reverse}), whose geodesics satisfy $\ddot\gamma^k = -2 G^k(\gamma, \dot\gamma)$. First, the S-curvature reads \citep{shen2001lectures}
  \begin{equation} \label{eq:shen_S}
    S(x,v) = \frac{\partial G^l}{\partial v^l}(x,v) - v^l \partial_l \log \sigma(x).
  \end{equation}
  Second, $\FF$ is constant along geodesics, so that differentiating $t \mapsto \FF(\gamma(t), \dot\gamma(t))$ at $t = 0$ gives
  \begin{equation} \label{eq:F_conservation}
    v^j \frac{\partial \FF}{\partial x^j}(x,v) = 2 G^l(x,v) \frac{\partial \FF}{\partial v^l}(x,v).
  \end{equation}
  Let $P$ be a polynomial in $v$. An integration by parts in $v$, followed by \cref{eq:F_conservation}, gives
  \begin{align*}
    \int_{T_x\gM} K(\FF) \frac{\partial G^l}{\partial v^l} P \diff\leb_x(v)
    &= - \int_{T_x\gM} K'(\FF) \frac{\partial \FF}{\partial v^l} G^l P \diff\leb_x(v) - \int_{T_x\gM} K(\FF) G^l \frac{\partial P}{\partial v^l} \diff\leb_x(v) \\
    &= - \frac{1}{2} \int_{T_x\gM} K'(\FF) v^j \frac{\partial \FF}{\partial x^j} P \diff\leb_x(v) - \int_{T_x\gM} K(\FF) G^l \frac{\partial P}{\partial v^l} \diff\leb_x(v) \\
    &= - \frac{1}{2} \partial_j \int_{T_x\gM} K(\FF) v^j P \diff\leb_x(v) - \int_{T_x\gM} K(\FF) G^l \frac{\partial P}{\partial v^l} \diff\leb_x(v).
  \end{align*}
  The boundary terms vanish, and the derivative can be taken out of the integral, since $K(\FF(x,v))$ and $K'(\FF(x,v))$ decay exponentially in $v$ by (A1), while $G^l$, its derivatives and $P$ grow polynomially. Moreover, in a chart, $\leb_x$ is the Lebesgue measure of $\R^m$, which does not depend on $x$.

  \emph{First identity.} Taking $P = 1$, we get $\int K(\FF) \partial_{v^l} G^l \diff\leb_x = -\frac{1}{2} \partial_j M_1^j$. Hence, by \cref{eq:shen_S},
  \begin{equation*}
    s_0 = \sigma \int_{T_x\gM} K(\FF) S \diff\leb_x(v) = - \frac{\sigma}{2} \partial_j M_1^j - \partial_j \sigma M_1^j = - \frac{1}{2\sigma} \partial_j \left( \sigma^2 M_1^j \right) = - \frac{1}{2} \operatorname{div}_{BH}(m_1).
  \end{equation*}

  \emph{Second identity.} Taking $P = v^k$, we get $\int K(\FF) \partial_{v^l} G^l v^k \diff\leb_x = -\frac{1}{2} \partial_j M_2^{jk} - \int K(\FF) G^k \diff\leb_x$. Moreover, $\ChrisF^k_{ij}(x,v) v^i v^j = 2 G^k(x,v)$, so that $\Lambda^k = 2 \sigma \int K(\FF) G^k \diff\leb_x$. Hence, by \cref{eq:shen_S},
  \begin{align*}
    \Lambda^k + 2 s_1^k &= 2 \sigma \int_{T_x\gM} K(\FF) \left( G^k + \frac{\partial G^l}{\partial v^l} v^k \right) \diff\leb_x(v) - 2 \partial_j \sigma M_2^{jk} \\
    &= - \sigma \partial_j M_2^{jk} - 2 \partial_j \sigma M_2^{jk} = - \frac{1}{\sigma} \partial_j \left( \sigma^2 M_2^{jk} \right) = - \frac{1}{\sigma} \partial_j \left( \sigma m_2^{jk} \right).
  \end{align*}

  The expressions of $\gA_1$ and $\gA_2$ follow by substitution in \cref{def:kernel_moments}. Finally, for smooth $f$ and $h$, Stokes' theorem \citep[Theorem 16.11]{lee2012introduction} on the compact manifold without boundary $\gM$ gives $\int_\gM \operatorname{div}_{BH}(f h  m_1) \diff\mBH = 0$, that is
  \begin{equation*}
    \int_\gM \gA_1 [f] h  \diff\mBH = - \int_\gM f  \gA_1 [h]  \diff\mBH,
  \end{equation*}
  and similarly $\int_\gM \operatorname{div}_{BH}(h  m_2 \nabla f - f  m_2 \nabla h) \diff\mBH = 0$ gives the self-adjointness of $\gA_2$.
\end{proof}

\begin{remark}
  \cref{lemma:moment_identities} is the infinitesimal counterpart of the adjointness relation $\gG'_\FF = \gG_\FF^*$ (\cref{prop:adjoint_reverse}). Indeed, plugging the expansions of \cref{thm:moment_expansion,prop:reverse_moment_expansion} into $\int_\gM \gG_\FF[f]  h  \diff\mBH = \int_\gM f  \gG'_\FF[h]  \diff\mBH$, the terms of order $\eps$ and $\eps^2$ show that $\gA_1$ is skew-adjoint and $\gA_2$ is self-adjoint on $L^2(\gM, \mBH)$, which gives back the two identities after an integration by parts.
\end{remark}

Dividing the second identity by the constant $m_0$, the second-order and first-order terms of \cref{eq:pf_general_limit} combine into a divergence,
\begin{equation*}
  \frac{1}{2} \tilde{m}_2^{ij} \partial_i \partial_j f - \frac{1}{2} \big(\tilde{\Lambda}^k + 2 \tilde{s}_1^k\big) \partial_k f = \frac{1}{2 \sigBH} \partial_j \left( \sigBH \tilde{m}_2^{jk} \partial_k f \right) = \frac{1}{2} \operatorname{div}_{BH} \left( \tilde{m}_2 \nabla f \right),
\end{equation*}
where $\tilde{m}_2 \nabla f$ denotes the vector field $\tilde{m}_2^{jk} \partial_k f \partial_j$. Moreover, writing $\omega = \rho^{2(1-\theta)}$, the Leibniz rule gives $\frac{1}{2\omega} \operatorname{div}_{BH}(\omega X) = \frac{1}{2} \operatorname{div}_{BH}(X) + (1-\theta) X^j \partial_j \log \rho$ for any vector field $X$. Taking $X = \tilde{m}_2 \nabla f$, and using $\tilde{m}_2^{jk} = c  \gBL^{jk}$ (\cref{props:moment_expression_unitball}), we get
\begin{equation*}
  \gL^s f = \frac{1}{2\omega} \operatorname{div}_{BH} \left( \omega \tilde{m}_2 \nabla f \right) = \frac{c_2}{\rho^{2(1-\theta)}} \operatorname{div}_{BH} \left( \rho^{2(1-\theta)} \nabla_\text{BL} f \right).
\end{equation*}

Finally, since $\gM$ is compact and without boundary (assumption (A2)), Stokes' theorem \citep[Theorem 16.11]{lee2012introduction} gives $\int_\gM \operatorname{div}_{BH}(X) \diff\mBH = 0$ for any smooth vector field $X$. Applying it to $X = u  \omega \nabla_\text{BL} f$, for smooth $u, f : \gM \to \R$, gives
\begin{equation*}
  \int_\gM u  \gL^s [f]  \omega \diff\mBH = - c_2 \int_\gM \langle \nabla u, \nabla f \rangle_\text{BL}  \omega  \diff\mBH,
\end{equation*}
which is symmetric in $u$ and $f$: $\gL^s$ is self-adjoint on $L^2(\gM, \omega \mBH)$. Its principal part is $c_2 \gBL^{ij} \partial_i \partial_j$, with $\gBL^{ij}$ positive definite, so that it is elliptic. This concludes the proof of \cref{thm:sym_limit}.

\subsection{Links between the moments of the kernel and the unit ball} \label{appendix:proofs:prop:moment_expression_unitball}

We recall that the moments of the kernel are defined as (\cref{def:kernel_moments}):
\begin{align*}
  m_n(x) = \sigBH(x) \int_{T_x\gM} K(\FF(x,v)) v^{\otimes n} \diff\leb_x(v),
\end{align*}
We denote the normalized moments of the kernel as $\tilde{m}_n = m_n / m_0$. We define the unit tangent ball and its indicatrix as:
\begin{align*}
  \tball_x = \{v \in T_x\gM : \FF(x,v) \leq 1\}, \quad \gI_x = \{v \in T_x\gM : \FF(x,v) = 1\}.
\end{align*}
We also introduce the volume, the centroid, and the raw second moment of the unit ball read
\begin{align*}
  \leb_x(\tball_x) = \int_{\tball_x} \diff\leb_x(v), \quad \rvc^i (x) = \frac{1}{\leb_x(\tball_x)} \int_{\tball_x} v^i \diff\leb_x(v), \quad \text{and} \quad \rmS^{ij}(x) = \frac{1}{\leb_x(\tball_x)} \int_{\tball_x} v^i v^j \diff\leb_x(v).
\end{align*}

\begin{proposition}\label{props:moment_expression_unitball}
  We can relate the moments of the kernel to the unit ball as follows:
  \begin{align*}
    m_0(x) &= m \mu_0 \sigBH(x) \leb_x(\tball_x) = m \mu_0 \vole(\eball^m), \\
    \tilde{m}_1^i(x) &= \frac{(m+1) \mu_1}{m\mu_0} \rvc^i(x), \\
    \tilde{m}_2^{ij}(x) &= \frac{(m+2) \mu_2}{m\mu_0}\rmS^{ij}(x),
  \end{align*}
  where $\mu_n = \int_0^\infty K(r) r^{m + n - 1} \diff r$ are the raw moments of the radial kernel $K$.
\end{proposition}

\begin{proof}
  Since $\FF$ is positively homogeneous of degree $1$ in the second variable, we use a polar change of variable $v = r u$ with $r \in \R_+$ and $u \in \gI_x$. Under this change of variable, we have that $\diff\leb_x(v) = r^{m-1} \diff r \diff\omega_x$, where $\omega_x(A) = m \leb_x(\{ru : u \in A, 0 \leq r \leq 1\})$ for $A \subset \gI_x$. With this change of variable, the volume, centroids, and raw second moment of the unit ball as:
  \begin{align}
    & \leb_x(\tball_x) = \int_{\gI_x} \int_0^1 r^{m-1} \diff r \diff\omega_x = \frac{1}{m} \int_{\gI_x} \diff\omega_x, \label{eq:volume} \\
    & \rvc^i(x) = \frac{1}{\leb_x(\tball_x)} \int_{\gI_x} \int_0^1 (ru)^i r^{m-1} \diff r \diff\omega_x = \frac{1}{(m+1) \leb_x(\tball_x)} \int_{\gI_x} u^i \diff\omega_x, \label{eq:centroid} \\
    & \rmS^{ij}(x) = \frac{1}{\leb_x(\tball_x)} \int_{\gI_x} \int_0^1 (ru)^i (ru)^j r^{m-1} \diff r \diff\omega_x = \frac{1}{(m+2) \leb_x(\tball_x)} \int_{\gI_x} u^i u^j \diff\omega_x. \label{eq:raw_second_moment}
  \end{align}

  We now turn to the moments $m_n(x)$ of the kernel. Since $\sigBH(x)$ does not depend on $v$, we compute them up to this factor. For the zeroth moment,
  \begin{align*}
    \frac{m_0(x)}{\sigBH(x)} &= \int_{T_x\gM} K(\FF(x,v)) \diff\leb_x(v) \\
    &= \int_{\gI_x} \int_0^\infty K(r) r^{m-1} \diff r \diff\omega_x  & \text{By the change of variable } v = ru \\
    &= \mu_0 \int_{\gI_x} \diff\omega_x & \text{By Fubini's theorem} \\
    &= m \mu_0 \leb_x(\tball_x) & \text{By \cref{eq:volume}.}
  \end{align*}
  Similarly, the first moment is given by:
  \begin{align*}
    \frac{m_1^i(x)}{\sigBH(x)} &= \int_{T_x\gM} K(\FF(x,v)) v^i \diff\leb_x(v) \\
    &= \int_{\gI_x} \int_0^\infty K(r) (ru)^i r^{m-1} \diff r \diff\omega_x & \text{By the change of variable } v = ru  \\
    &= \mu_1 \int_{\gI_x} u^i \diff\omega_x  & \text{By Fubini's theorem} \\
    &= (m+1) \mu_1 \leb_x(\tball_x) \rvc^i(x) & \text{By \cref{eq:centroid}.}
  \end{align*}
  The second moment is given by:
  \begin{align*}
    \frac{m_2^{ij}(x)}{\sigBH(x)} &= \int_{T_x\gM} K(\FF(x,v)) v^i v^j \diff\leb_x(v) \\
    &= \int_{\gI_x} \int_0^\infty K(r) (ru)^i (ru)^j r^{m-1} \diff r \diff\omega_x & \text{By the change of variable } v = ru  \\
    &= \mu_2 \int_{\gI_x} u^i u^j \diff\omega_x  & \text{By Fubini's theorem} \\
    &= (m+2) \mu_2 \leb_x(\tball_x) \rmS^{ij}(x) & \text{By \cref{eq:raw_second_moment}.}
  \end{align*}

  Multiplying by $\sigBH(x)$, \cref{eq:busemann_haussdorf_measure} gives $m_0(x) = m \mu_0 \sigBH(x) \leb_x(\tball_x) = m \mu_0 \vole(\eball^m)$. Dividing the expression of $m_n$ by $m_0$, in which $\sigBH(x)$ cancels, yields the desired result for the normalized moments $\tilde{m}_n = m_n / m_0$.

\end{proof}

\subsubsection*{Implications for the limit operators}

\paragraph{Zeroth moment.} The zeroth moment $m_0(x)$ of the kernel appears in two ways: as the normalization factor in the definitions of $\tilde{m}_n$, $\tilde{\Lambda}^k$ and $\tilde{s}_1^k$, which is discussed at the end of this subsection, and as the leading coefficient of the expansion of \cref{thm:moment_expansion}. By \cref{props:moment_expression_unitball}, $m_0 = m \mu_0 \vole(\eball^m)$ does not depend on $x$: the Busemann--Hausdorff measure is precisely the one for which small forward balls have, at leading order, the volume of Euclidean balls. As a consequence, $\qe = m_0 \rho + O(\eps^2)$ carries no geometric factor, and $\partial_i \log \phit = (1-\theta) \partial_i \log \rho + O(\eps^2)$ (\cref{appendix:proofs:prop:sym_limit}): the $\theta$-normalization acts exactly as in the Riemannian setting, and $\theta = 1$ removes the sampling density from the limit operator $\gL^s$.

\paragraph{First moment.} By \cref{props:moment_expression_unitball}, the normalized first moment $\tilde{m}_1^i(x)$ is proportional to the centroid $\rvc^i(x)$ of the unit ball $\tball_x$. If the Finsler metric $\FF$ is reversible, the unit ball $\tball_x$ is symmetric with respect to the origin, so that $\rvc^i(x) = 0$. Hence, for reversible Finsler metrics, $\tilde{m}_1^i = 0$ and the anti-symmetric limit operator $\gL^a$ vanishes\footnote{One can note that in this context, the anti-symmetric operator $\Pta$ also vanishes: $\Pta[f] \equiv 0$ for any $f$. This is due to the fact that for a reversible Finsler metric, $\dstF(x,y) = \dstF(y,x)$ so that $W(x,y) = W(y,x)$.}. When $\FF$ is not reversible, the unit ball is in general not symmetric, and its centroid defines a vector field $\rvc$ on $\gM$ which measures the asymmetry of $\FF$ at first order. Note that $\rvc(x) = 0$ does not imply that $\FF$ is reversible at $x$, since a non-symmetric unit ball can have its centroid at the origin. For Randers metrics, the centroid is given explicitly in terms of $b$ by \cref{prop:centroid_randers_metric}. This vector field is exactly what the anti-symmetric part of the graph Laplacian recovers: by \cref{thm:sym_limit},
\begin{equation*}
  \gL^a f(x) = \tilde{m}_1^i(x) \partial_i f(x) = \frac{(m+1) \mu_1}{m \mu_0} \rvc^i(x) \partial_i f(x),
\end{equation*}
so that $\gL^a$ is the derivative along the centroid vector field, up to a constant that only depends on the kernel.

\paragraph{Second moment.} By \cref{props:moment_expression_unitball} and \cref{eq:binet_legendre}, the normalized second moment is proportional to the raw second moment $\rmS^{ij}(x)$ of the unit ball, and hence to the Binet--Legendre metric $\gBL$ (\cref{def:binet_legendre}):
\begin{equation*}
  \tilde{m}_2^{ij}(x) = \frac{(m+2) \mu_2}{m \mu_0} \rmS^{ij}(x) = \frac{\mu_2}{m \mu_0} \gBL^{ij}(x).
\end{equation*}
Hence, up to the constant $2c_2 = \mu_2 / (m \mu_0)$ of \cref{thm:sym_limit}, the normalized second moment coincides with the Binet--Legendre metric, which is a Riemannian metric capturing the local geometry of $(\gM, \FF)$. It depends on the kernel only through this constant, and does not depend on the bandwidth $\eps$.

\paragraph{Normalization by $m_0$.} Dividing by $m_0$ is what makes $\tilde{m}_1$ and $\tilde{m}_2$ the natural geometric objects governing the limit operators. Indeed, by \cref{props:moment_expression_unitball}, the raw moments $m_1$ and $m_2$ carry the factor $\sigBH(x) \leb_x(\tball_x)$, as $m_0$ does, which cancels in the ratio. The normalized moments thus only involve the centroid $\rvc$ and the raw second moment $\rmS$, which are averages over the unit ball, and the kernel only enters through the constants $\mu_n / \mu_0$. This is also why the Binet--Legendre metric, which is itself normalized by $\leb_x(\tball_x)$ in \cref{eq:binet_legendre}, appears in the limit.

\subsection{Proof of \texorpdfstring{\cref{cor:berwald_limit}}{}} \label{appendix:proofs:cor:berwald_limit}

We start from the divergence form of \cref{thm:sym_limit}. Let $\sqrt{\det \gBL}$ denote the density, in a chart, of the Riemannian volume $\mathrm{vol}_\text{BL}$ of $\gBL$, and let
\begin{equation*}
  w = \frac{\sigBH}{\sqrt{\det \gBL}} = \frac{\vole(\eball^m)}{\mathrm{vol}_\text{BL}(\tball_x)}
\end{equation*}
be the density of $\mBH$ with respect to $\mathrm{vol}_\text{BL}$, where we used \cref{eq:busemann_haussdorf_measure}, and $\mathrm{vol}_\text{BL}(\tball_x) = \sqrt{\det \gBL(x)} \leb_x(\tball_x)$ is the volume of the unit ball of $\FF$ for the Binet--Legendre metric. In particular, $w$ does not depend on the chart. Since $\sigBH = w \sqrt{\det \gBL}$, we have $\operatorname{div}_{BH} X = \operatorname{div}_\text{BL} X + X^i \partial_i \log w$ for any vector field $X$, where $\operatorname{div}_\text{BL}$ is the Riemannian divergence of $\gBL$. Applying it in \cref{thm:sym_limit}, and using $\operatorname{div}_\text{BL}(\nabla_\text{BL} f) = \Delta_\text{BL} f$, we get, for any Finsler metric,
\begin{equation} \label{eq:pf_general_BL}
  \gL^s f = c_2 \left( \Delta_\text{BL} f + 2 (1-\theta) \langle \nabla \log \rho, \nabla f \rangle_\text{BL} + \langle \nabla \log w, \nabla f \rangle_\text{BL} \right).
\end{equation}
The discrepancy with a weighted Laplace--Beltrami operator of $\gBL$ is thus a drift along the gradient of $\log w$, which measures the variations of the Binet--Legendre volume of the unit ball of $\FF$ across $\gM$.

\emph{Berwald metrics.} When $\FF$ is Berwald, the parallel transport along any curve from $x$ to $y$ is a linear map $P : T_x\gM \to T_y\gM$ which preserves $\FF$ \citep{bao2012introduction}, so that $P(\tball_x) = \tball_y$. By \cref{eq:binet_legendre}, the Binet--Legendre metric is built from the unit ball alone, and the change of variable $v = Pu$ gives $\gBL^{-1}(y) = P \gBL^{-1}(x) P^\top$, the Jacobian of $P$ being cancelled by the normalization by the volume of the ball. Hence $P$ is an isometry from $(T_x\gM, \gBL(x))$ to $(T_y\gM, \gBL(y))$ which maps $\tball_x$ onto $\tball_y$, so that $\mathrm{vol}_\text{BL}(\tball_x) = \mathrm{vol}_\text{BL}(\tball_y)$. Since $\gM$ is connected, $w$ is constant, and \cref{eq:pf_general_BL} reduces to \cref{eq:berwald_limit}.

\emph{Riemannian metrics.} Let $\FF(x,v) = \sqrt{v^\top \rmA(x) v}$. Its fundamental tensor is $g_{ij}(x,v) = \rmA_{ij}(x)$ for every $v$, hence does not depend on the direction, so $\FF$ is Berwald. Its unit ball $\tball_x = \{v : v^\top \rmA(x) v \leq 1\}$ is the ellipsoid defined by $\rmA(x)$, whose second moment is given by the classical identity $\int_{\tball_x} v^i v^j \diff\leb_x(v) = \frac{\leb_x(\tball_x)}{m+2} \rmA^{ij}(x)$. Substituting it into \cref{eq:binet_legendre} gives $\gBL^{ij} = \rmA^{ij}$, \ie $\gBL = \rmA$. \cref{eq:berwald_limit} then reads
\begin{equation*}
  \gL^s f = c_2 \left( \Delta_\rmA f + 2 (1-\theta) \langle \nabla \log \rho, \nabla f \rangle_\rmA \right),
\end{equation*}
which is the $\theta$-normalized diffusion generator of \citet{coifman_diffusion_2006}. Note also that the S-curvature of a Riemannian metric vanishes, so that $s_0 = s_1 = 0$ and, in this setting, $\Upsilon$ reduces to
\begin{equation*}
  \Upsilon(x) = -\frac{1}{3} \sigBH(x) \int_{T_x\gM} K(\FF(x,v)) \RicF(x,v) \diff\leb_x(v).
\end{equation*} Finally, a Riemannian metric is reversible, so $v \mapsto K(\FF(x,v))$ is an even function on $T_x\gM$, and the change of variable $v \mapsto -v$ gives $m_1 = 0$. By \cref{thm:sym_limit}, $\gL^a f = \tilde{m}_1^i \partial_i f = 0$.

\subsection{Proof of \texorpdfstring{\cref{thm:convergence_discrete_op}}{}} \label{appendix:proofs:thm:convergence_discrete_op}

This subsection is devoted to the proof of \cref{thm:convergence_discrete_op}. We first give the matrix forms of the discrete operators, then we prove the convergence of the empirical operators to their population counterparts, and finally we prove the convergence of the population operators to their limit.

\subsubsection*{Matrix forms}

We start by recalling the basic definitions of the discrete operators. The matrices of \cref{sec:discrete_operators} are all built from the empirical kernel matrix $\rmW_N$, and the two members of the family $\rmP^{(\theta,\bullet)}$ take a slightly different shape. For $\bullet = s$, the matrix $\rmD_N^{(\theta,s)}$ is by construction the degree matrix of $\rmW_N^{(\theta,s)}$, so that the two terms merge into
\begin{equation*}
  \rmP^{(\theta,s)}[f] = \Big(\big(\rmD_N^{(\theta,s)}\big)^{-1}\rmW_N^{(\theta,s)}-\rmI\Big) \rvf,
\end{equation*}
a normalized-graph-Laplacian-type matrix, whose eigenvectors provide the embedding coordinates (\cref{sec:embedding_algo}). For $\bullet = a$, the matrix $\rmD_N^{(\theta,a)}$ is not the degree matrix of $\rmW_N^{(\theta,s)}$, and the two terms remain separate:
\begin{equation*}
  \rmP^{(\theta,a)}[f] = \big(\rmD_N^{(\theta,s)}\big)^{-1}\rmW_N^{(\theta,a)} \rvf - \big(\rmD_N^{(\theta,s)}\big)^{-1}\rmD_N^{(\theta,a)} \rvf
\end{equation*}
is the estimator of the advection term, whose limit is $\eps \tilde{m}_1^i \partial_i f$.

\subsubsection*{From finite samples to the limit operator: setting and proof sketch}

We now let $N \to \infty$ and $\eps = \eps(N) \to 0$ jointly. Throughout, $C$ denotes a constant depending on $\gM$, $\FF$, $K$, $\rho$, $\theta$ and $f$, but never on $N$ or $\eps$, and possibly changing from line to line.

It is convenient to extend the discrete objects to the whole manifold, so that they can be compared pointwise with their continuous counterparts. For $x,y \in \gM$, we set
\begin{align*}
  &\hat{d}_N(x) = \frac{1}{N} \sum_{j=1}^N W(x, X_j), \qquad \hat{d}'_N(x) = \frac{1}{N} \sum_{j=1}^N W(X_j, x), \qquad \hat{q}_N = \frac{\hat{d}_N + \hat{d}'_N}{2}, \\
  &\hat{W}^{(\theta)}_N(x,y) = \frac{W(x,y)}{\hat{q}_N(x)^\theta \hat{q}_N(y)^\theta}, \qquad
  \hat{W}^{(\theta,\bullet)}_N(x,y) = \frac{\hat{W}^{(\theta)}_N(x,y) \pm \hat{W}^{(\theta)}_N(y,x)}{2}, \\
  &\hat{d}^{(\theta,\bullet)}_N(x) = \frac{1}{N} \sum_{j=1}^N \hat{W}^{(\theta,\bullet)}_N(x, X_j),
\end{align*}
with the sign $+$ for $\bullet = s$ and $-$ for $\bullet = a$, and finally
\begin{equation} \label{eq:extended_discrete_operator}
  \hat{\rmP}^{(\theta,\bullet)}_N[f](x) = \frac{1}{\hat{d}^{(\theta,s)}_N(x)} \left( \frac{1}{N} \sum_{j=1}^N \hat{W}^{(\theta,\bullet)}_N(x, X_j) f(X_j) - f(x) \hat{d}^{(\theta,\bullet)}_N(x) \right).
\end{equation}
We also write $\hat{\rmL}^s_N = \eps^{-2} \hat{\rmP}^{(\theta,s)}_N$ and $\hat{\rmL}^a_N = \eps^{-1} \hat{\rmP}^{(\theta,a)}_N$ for the rescaled versions, extending $\rmL^s$ and $\rmL^a$ in the same way. Evaluated at the sample points $x = X_i$, these are exactly the matrices of \cref{sec:discrete_operators}, the factors $1/N$ cancelling in the ratio \cref{eq:extended_discrete_operator}, so that $\hat{\rmP}^{(\theta,\bullet)}_N[f](X_i) = \big(\rmP^{(\theta,\bullet)} \rvf\big)_i$ and $\hat{\rmL}^\bullet_N[f](X_i) = \big(\rmL^\bullet \rvf\big)_i$. Their continuous counterparts are $\dout$, $\din$, $\qe$, $W^{(\theta)}$, $\dts = \gGts[\rho]$, $\dta = \gGta[\rho]$ and $\Ptb$.

Two errors then separate the empirical operator from its limit, and we bound them one at a time. Writing the symmetric case for concreteness, and inserting the population operator $\Pts$ between the two, we have that:
\begin{equation} \label{eq:error_split}
  \big| \hat{\rmL}^s_N[f] - \gL^s f \big|  \leq \underbrace{\frac{1}{\eps^2}\big| \hat{\rmP}^{(\theta,s)}_N[f] - \Pts[f] \big|}_{\text{fluctuation}}  +  \underbrace{\frac{1}{\eps^2}\big| \Pts[f] - \eps^2 \gL^s f \big|}_{\text{bias}} .
\end{equation}
The \emph{bias} is deterministic: it is the gap between the population operator and its limit, and it decreases with $\eps$. The \emph{fluctuation} is stochastic: it is the gap between the empirical operator and its population counterpart, and it grows as $\eps$ decreases, since the kernel then concentrates and fewer samples contribute to each evaluation. The next two subsections bound each of them, and the last one combines the two and reads off the admissible scalings of $\eps(N)$.

\subsubsection*{Bounding the bias}

The deterministic part of the error is the expansion already established in \appref{appendix:proofs:prop:sym_limit}, read quantitatively rather than as a limit.

\begin{lemma}[Bias] \label{lemma:bias}
  Let $f \in \gC^3(\gM)$. There exist $C, \eps_0 > 0$ such that for $\eps \leq \eps_0$,
  \begin{equation*}
    \sup_{x \in \gM} \big| \Pts[f](x) - \eps^2 \gL^s f(x) \big| \leq C \eps^3, \qquad
    \sup_{x \in \gM} \big| \Pta[f](x) - \eps \gL^a f(x) \big| \leq C \eps^2 .
  \end{equation*}
\end{lemma}

\begin{proof}
  The proof of \cref{thm:sym_limit} in \appref{appendix:proofs:prop:sym_limit} establishes these expansions with an explicit remainder, inherited from the $\bigO(\eps^3)$ of \cref{thm:moment_expansion}. The remainder is uniform in $x$, since $\gM$ is compact and the derivatives of $f$ are bounded.
\end{proof}

\subsubsection*{Bounding the fluctuation}

We now bound the stochastic part of the error, which is the deviation of the empirical operator from its population counterpart. The next proposition gives a uniform bound on this deviation, which is the main technical result of this section.
\begin{proposition}[Fluctuation] \label{prop:fluctuation}
  Assume $\eps = \eps(N) \to 0$ with $N\eps^m / \log N \to \infty$, and let $f \in \gC^1(\gM)$. Then there is a constant $C$ such that, almost surely for $N$ large enough,
  \begin{equation*}
    \sup_{x \in \gM} \left| \hat{\rmP}^{(\theta,\bullet)}_N[f](x) - \Ptb[f](x) \right| \leq C \eps \zeta_{N}, \qquad \bullet \in \{s,a\},
  \end{equation*}
  where
  \begin{equation} \label{eq:fluctuation_scale}
  \zeta_{N} = \sqrt{\frac{\log N}{N \eps^m}}.
\end{equation}
\end{proposition}
To prove this bound, we rely on standard concentration inequalities for sums of independent random variables, which give a bound on the deviation of the empirical average from its expectation. To do so, we use a $\delta$-net argument\footnote{A $\delta$-net is a finite set of points in a metric space such that every point in the space is within distance $\delta$ of at least one point in the net.}: we first control the deviation on a finite set of points, then extend it to the whole manifold by continuity, and conclude almost surely with the Borel--Cantelli lemma.

\emph{Proof of \cref{prop:fluctuation}.}
\emph{Defining the net.} The intrinsic distance $\dstF$ is asymmetric, so that it yields two different notions of $\delta$-net, depending on whether we use forward or backward balls. We therefore rely instead on the Euclidean distance $|\cdot|$ of $\R^D$ restricted to the tangent spaces of $\gM$, which is symmetric and induces a natural Riemannian distance $\dist$ on $\gM$. In particular, the distances $\dstF$ and $\dist$ are comparable: the map $(x,v) \mapsto \FF(x,v)$ is continuous and positive on the unit sphere bundle $\{(x,v) \in T\gM : |v| = 1\}$, which is compact by (A2), so that it is bounded there between two constants. By positive homogeneity, we get that
\begin{equation*}
  \kappa_- |v| \leq \FF(x,v) \leq \kappa_+ |v|, \qquad (x,v) \in T \gM, \qquad \text{for some } 0 < \kappa_- \leq \kappa_+ < \infty.
\end{equation*}
Hence, for any $x,y \in \gM$,
\begin{equation} \label{eq:distance_comparison}
  \kappa_- \dist(x,y) \leq \dstF(x,y) \leq \kappa_+ \dist(x,y).
\end{equation}
Once this is established, we can define a $\delta$-net according to this Riemannian distance. Counting the number of points in such a net is standard, and we give the result below for completeness. In this section, $\vol$ denotes the Riemannian volume of $\gM$ associated with the Riemannian metric that induces $\dist$.

\begin{lemma} \label{lemma:net_counting}
  There are $C_\gM, \delta_0 > 0$ such that, for every $\delta \leq \delta_0$, $\gM$ admits a $\delta$-net of cardinality $n_\delta \leq C_\gM \delta^{-m}$.
\end{lemma}

\emph{Proof sketch.} The input is the volume of small balls. On a compact Riemannian manifold, the volume density in normal coordinates around $x$ is $1 + \bigO(r^2)$, with a remainder that is uniform in $x$ since the curvature is bounded and the injectivity radius is bounded below, hence there are $v_-, v_+, \delta_0 > 0$ such that $v_- r^m \leq \vol\big(B(x,r)\big) \leq v_+ r^m$ for every $x \in \gM$ and every $r \leq \delta_0$. Let $x_1, \ldots, x_{n_\delta}$ be a maximal $\delta$-separated subset of $\gM$, which exists by compactness. It is a $\delta$-net: any $x$ at distance more than $\delta$ from all the $x_k$ could be added to the family, contradicting maximality. Moreover the balls $B(x_k, \delta/2)$ are pairwise disjoint, by $\delta$-separation, so that summing their volumes gives $n_\delta v_- (\delta/2)^m \leq \vol(\gM)$, which is the announced bound. The same packing argument, in the Euclidean setting, is detailed in \citet[Section 4.2]{vershynin2019high}. \hfill $\qed$

Using this $\delta$-net, the supremum over $\gM$ of a function $g$ reduces to a maximum over the net points: picking for each $x$ a net point $x_k$ with $\dist(x,x_k) \leq \delta$, and denoting by $L_g$ a Lipschitz constant of $g$,
\begin{equation*}
  \sup_{x \in \gM} |g(x)| \leq \max_{1 \leq k \leq n_\delta} |g(x_k)| + L_g \delta .
\end{equation*}

\emph{Deterministic estimates on the kernel.} Besides the net, we need three elementary properties of the kernel $W$: it is of size $\eps^{-m}$, it is concentrated at scale $\eps$ around the diagonal, and it varies at scale $\eps$. We write $W^\bullet(x,y) = \big( W(x,y) \pm W(y,x) \big)/2$ for its symmetric ($\bullet = s$) and antisymmetric ($\bullet = a$) parts, so that $|W^\bullet| \leq W^{s}$.

\begin{lemma}[Kernel at scale $\eps$] \label{lemma:kernel_scale}
  There is a constant $C$ such that, for every $\eps \leq 1$, every $x, x', y \in \gM$ and every $p \in \{0,1\}$,
  \begin{equation} \label{eq:kernel_scale}
    W^{s}(x,y) \leq \frac{K(0)}{\eps^m}, \qquad
    W^{s}(x,y) \dist(x,y) \leq \frac{C}{\eps^{m-1}}, \qquad
    \int_\gM W^{s}(x,y) \dist(x,y)^p \diff\mathbb{P}(y) \leq C \eps^p,
  \end{equation}
  and, for $\bullet \in \{s,a\}$,
  \begin{equation} \label{eq:kernel_lipschitz}
    \big| W^\bullet(x,y) - W^\bullet(x',y) \big| \leq \frac{C \dist(x,x')}{\eps^{m+1}}.
  \end{equation}
\end{lemma}
\begin{proof}
  Since $K$ is decreasing and $\dstF \geq 0$, both $W(x,y)$ and $W(y,x)$ are bounded by $K(0)\eps^{-m}$, which gives the first bound. For the second one, the distance comparison inequality (\cref{eq:distance_comparison}) gives $\dstF(x,y) \geq \kappa_- \dist(x,y)$ and $\dstF(y,x) \geq \kappa_- \dist(x,y)$, so that $W^{s}(x,y) \leq \eps^{-m} K\big( \kappa_-\dist(x,y)/\eps \big)$. Writing $r = \kappa_-\dist(x,y)/\eps$, it holds that
  \begin{equation*}
    W^{s}(x,y) \dist(x,y) \leq \frac{\eps^{1-m}}{\kappa_-} r K(r) \leq \frac{\eps^{1-m}}{\kappa_-} \sup_{r \geq 0} r K(r),
  \end{equation*}
  which is finite since $K$ is sub-exponential (A1). For the third one, we use that $\vol\big(B(x,r)\big) \leq v_+ r^m$ for every $r > 0$: this is shown in the proof of \cref{lemma:net_counting} for $r \leq \delta_0$, and it extends to $r > \delta_0$ by enlarging $v_+$ by $\vol(\gM)\delta_0^{-m}$. We split $\gM$ into the annuli $A_k = \{ y \in \gM : k\eps \leq \dist(x,y) < (k+1)\eps \}$, $k \geq 0$. On $A_k$, we have $W^{s}(x,y) \leq \eps^{-m} K(\kappa_- k) \leq C_K\eps^{-m} e^{-\nu_K \kappa_- k}$ and $\dist(x,y)^p \leq \big((k+1)\eps\big)^p$, and moreover $\vol(A_k) \leq v_+\big((k+1)\eps\big)^m$. Since $\diff\mathbb{P} = \rho \diff\mBH$, where $\rho$ and the ratio of the densities of $\mBH$ and $\vol$ are bounded on the compact manifold $\gM$, we have $\diff\mathbb{P} \leq C d\vol$, and we get
  \begin{equation*}
    \int_\gM W^{s}(x,y) \dist(x,y)^p \diff\mathbb{P}(y) \leq \frac{C}{\eps^m} \sum_{k \geq 0} e^{-\nu_K \kappa_- k} \big( (k+1) \eps \big)^{m+p} = C \eps^p \sum_{k \geq 0} (k+1)^{m+p} e^{-\nu_K \kappa_- k},
  \end{equation*}
  and the series converges. Finally, for the Lipschitz bound, the triangle inequality for $\dstF$ gives $\dstF(x,y) \leq \dstF(x,x') + \dstF(x',y)$ and $\dstF(x',y) \leq \dstF(x',x) + \dstF(x,y)$, so that, by \cref{eq:distance_comparison},
  \begin{equation*}
    \big| \dstF(x,y) - \dstF(x',y) \big| \leq \max\{ \dstF(x,x'), \dstF(x',x) \} \leq \kappa_+ \dist(x,x'),
  \end{equation*}
  and the same holds in the second argument. Since $K'$ is bounded by (A1), we get $|W(x,y) - W(x',y)| \leq \kappa_+\lVert K' \rVert_\infty \eps^{-m-1} \dist(x,x')$, and similarly for $x \mapsto W(y,x)$. The bound \cref{eq:kernel_lipschitz} follows, since $W^\bullet$ is the half-sum or the half-difference of these two functions.
\end{proof}

We also need the normalization $\qe$ to be bounded above and away from zero so that the normalized kernel is well-defined. On the one hand, the assumption on the degrees gives $\qe = (\din + \dout)/2 > \dmin$. On the other hand, since $\qe(x) = \int_\gM W^{s}(x,y) \diff\mathbb{P}(y)$, the case $p = 0$ of \cref{eq:kernel_scale} shows that $\qe$ is bounded. Hence, there are constants $C_q$, $c_\theta$ and $C_\theta$, independent of $\eps$, such that for every $x \in \gM$,
\begin{equation} \label{eq:q_bounds}
  \dmin \leq \qe(x) \leq C_q \qquad \text{and} \qquad 0 < c_\theta \leq \qe(x)^{-\theta} \leq C_\theta.
\end{equation}

\emph{Reduction to numerator and denominator.} To apply the $\delta$-net argument, the natural choice is the function $g(x) = \hat{\rmP}^{(\theta,\bullet)}_N[f](x) - \Ptb[f](x)$. However, this function is not itself an empirical average, so we apply this argument to its numerator and its denominator separately. One can remark that \cref{eq:extended_discrete_operator} can be rewritten as a ratio of two empirical averages:
\begin{equation*}
  \hat{\rmP}^{(\theta,\bullet)}_N [f](x) = \frac{\frac{1}{N} \sum_{j=1}^N \hat{W}^{(\theta,\bullet)}_N(x, X_j) (f(X_j) - f(x))}{\frac{1}{N} \sum_{j=1}^N \hat{W}^{(\theta,s)}_N(x, X_j)}.
\end{equation*}
Moreover, by using the definition of $\theta$-normalization, it holds:
\begin{equation*}
  \hat{\rmP}^{(\theta,\bullet)}_N [f](x) = \frac{\frac{1}{N} \sum_{j=1}^N W^\bullet(x, X_j) \hat{q}_N(X_j)^{-\theta} (f(X_j) - f(x))}{\frac{1}{N} \sum_{j=1}^N W^s(x, X_j) \hat{q}_N(X_j)^{-\theta}}.
\end{equation*}
Importantly, the same rewriting holds for the population operator $\Ptb[f]$, with $\hat{q}_N$ replaced by $\qe$ and the empirical average replaced by an integral over $\gM$. Finding bounds on the deviation of $\hat{\rmP}^{(\theta,\bullet)}_N[f]$ from $\Ptb[f]$ therefore reduces to bounding the deviation of the numerator and denominator separately. Note that due to its definition, $\hat{q}_N$ is itself an empirical average, so that the numerator and denominator are not sums of independent random variables.

\begin{lemma}[Uniform deviation of an empirical average] \label{lemma:uniform_deviation}
  Let $X_1, \ldots, X_N$ be \iid with law $\mathbb{P}$, and let $\psi : \gM \times \gM \to \R$ be measurable, possibly depending on $N$, and assume that for every $x, x', y \in \gM$,
  \begin{equation*}
    |\psi(y,x)| \leq B, \qquad \int_\gM \psi(y,x)^2 \diff\mathbb{P}(y) \leq \sigma^2, \qquad |\psi(y,x) - \psi(y,x')| \leq L \dist(x,x'),
  \end{equation*}
  where $B$, $\sigma$ and $L$ may depend on $N$. Then, with $A_0 = 2(5m+2)$, almost surely for $N$ large enough,
  \begin{equation} \label{eq:uniform_deviation}
    \sup_{x \in \gM} \abs{\frac{1}{N} \sum_{j=1}^N \psi(X_j,x) - \int_\gM \psi(y,x) \diff\mathbb{P}(y)} \leq A_0 \left( \sigma \sqrt{\frac{\log N}{N}} + \frac{B \log N}{N} \right) + \frac{2L}{N^5}.
  \end{equation}
\end{lemma}
\begin{proof}
  Let $\delta = N^{-5}$ and let $x_1, \ldots, x_{n_\delta}$ be a $\delta$-net of $\gM$, with $n_\delta \leq C_\gM N^{5m}$ by \cref{lemma:net_counting}, which holds for $N$ large so that $\delta \leq \delta_0$. For any $x \in \gM$, let $x_k$ be a net point with $\dist(x,x_k) \leq \delta$. By the triangle inequality,
  \begin{align*}
    \abs{\frac{1}{N} \sum_{j=1}^N \psi(X_j,x) - \int_\gM \psi(y,x) \diff\mathbb{P}(y)}
    &\leq \underbrace{\abs{\frac{1}{N} \sum_{j=1}^N \big( \psi(X_j,x) - \psi(X_j,x_k) \big)}}_{(A)} \\
    &\quad + \underbrace{\abs{\frac{1}{N} \sum_{j=1}^N \psi(X_j,x_k) - \int_\gM \psi(y,x_k) \diff\mathbb{P}(y)}}_{(B)} \\
    &\quad + \underbrace{\abs{\int_\gM \big( \psi(y,x_k) - \psi(y,x) \big) \diff\mathbb{P}(y)}}_{(C)}.
  \end{align*}
  For (A) and (C), the Lipschitz assumption gives (A) $\leq L\delta$ and (C) $\leq L\delta$, for every realization of the sample.

  For (B), fix a net point $x_k$. The net points do not depend on the sample, so $Y_j = \psi(X_j, x_k)$ are \iid random variables with mean $\mu_k = \int_\gM \psi(y, x_k) \diff\mathbb{P}(y)$. By assumption, $Y_j$ and $-Y_j$ are bounded above by $B$ and have second moment at most $\sigma^2$. Bernstein's inequality, applied to both of them, gives for any $t > 0$
  \begin{equation} \label{eq:bernstein_point}
    \Prob \left( \abs{\frac{1}{N} \sum_{j=1}^N Y_j - \mu_k} > t \right) \leq 2 \exp\left( -\frac{N t^2}{2(\sigma^2 + Bt/3)} \right).
  \end{equation}
  We set $t = A_0 \left( \sigma \sqrt{\frac{\log N}{N}} + \frac{B \log N}{N} \right)$ and check that the exponent is then at least $A_0 \log N / 2$. Let $S = \sigma^2 + B^2 \log N / N$, which we can assume positive, since otherwise $\psi = 0$ and there is nothing to prove. For the numerator, dropping the nonnegative cross term in the square, we have
  \begin{equation*}
    N t^2 \geq N A_0^2 \left( \sigma^2 \frac{\log N}{N} + B^2 \frac{(\log N)^2}{N^2} \right) = A_0^2 S \log N.
  \end{equation*}
  For the denominator, the inequality $2ab \leq a^2 + b^2$ gives $\sigma B \sqrt{\log N / N} \leq S/2$, and we have $B^2 \log N / N \leq S$, so that
  \begin{align*}
    2 \left( \sigma^2 + \frac{Bt}{3} \right)
    &= 2 \sigma^2 + \frac{2A_0}{3} \left( \sigma B \sqrt{\frac{\log N}{N}} + \frac{B^2 \log N}{N} \right) \\
    &\leq 2 S + \frac{2A_0}{3} \left( \frac{S}{2} + S \right) = (2 + A_0) S .
  \end{align*}
  Combining the two, the exponent in \cref{eq:bernstein_point} is at least $\frac{A_0^2}{2+A_0} \log N$, and
  \begin{equation*}
    \frac{A_0^2}{2+A_0} \geq \frac{A_0}{2} \qquad \Longleftrightarrow \qquad 2 A_0 \geq 2 + A_0 \qquad \Longleftrightarrow \qquad A_0 \geq 2,
  \end{equation*}
  which holds since $A_0 = 2(5m+2)$. Hence (B) exceeds $t$ at the point $x_k$ with probability at most $2N^{-A_0/2}$. Taking a union bound over the $n_\delta$ net points, the maximum of (B) over the net exceeds $t$ with probability at most
  \begin{equation*}
    2 n_\delta N^{-A_0/2} = 2 n_\delta N^{-(5m+2)} \leq 2 C_\gM N^{-2},
  \end{equation*}
  which is summable in $N$. By the Borel--Cantelli lemma, (B) $\leq t$ for every net point, almost surely for $N$ large enough. Combining the three terms gives \cref{eq:uniform_deviation}.
\end{proof}

To apply this lemma to the quantity of interest, we first need to remove the dependence of the kernel on the sample, which is done by replacing $\hat{q}_N$ by its expectation $\qe$.

We can apply \cref{lemma:uniform_deviation} to $\hat{q}_N$ since it is an empirical average of \iid random variables, $\hat{q}_N(x) = \frac{1}{N}\sum_{j=1}^N W^{s}(x,X_j)$, with mean $\qe(x)$. By \cref{lemma:kernel_scale} and \cref{eq:q_bounds}, the kernel $\psi(y,x) = W^{s}(x,y)$ satisfies the assumptions of the lemma with $B \leq K(0) \eps^{-m}$, $\sigma^2 \leq B \qe(x) \leq C \eps^{-m}$ and $L \leq C \eps^{-m-1}$. Recalling that $\zeta_N^2 = \log N / (N \eps^m)$, the right-hand side of \cref{eq:uniform_deviation} is thus at most
\begin{equation*}
  C \left( \sqrt{\frac{\log N}{N \eps^m}} + \frac{\log N}{N \eps^m} \right) + \frac{C}{\eps^{m+1} N^5} = C \left( \zeta_N + \zeta_N^2 \right) + \frac{C}{\eps^{m+1} N^5}.
\end{equation*}
The assumption $N \eps^m / \log N \to \infty$ gives $\zeta_N \to 0$, as well as $\eps^{-m} \leq N$ and $\eps^{-1} \leq N^{1/m} \leq N$ for $N$ large enough, so that the last term is at most $C N^{-3}$. Since moreover $\zeta_N \geq \sqrt{\log N / N} \geq N^{-1/2}$ and $\eps \geq N^{-1}$, we have $N^{-3} \leq \eps \zeta_N$, so that this term is negligible here and in all the bounds below. Hence, almost surely for $N$ large enough,
\begin{equation} \label{eq:q_uniform}
  \sup_{x \in \gM} \left| \hat{q}_N(x) - \qe(x) \right| \leq C \zeta_N.
\end{equation}
Using this, we can rewrite the numerator of $\hat{\rmP}^{(\theta,\bullet)}_N[f]$ as
\begin{align*}
  &\frac{1}{N} \sum_{j=1}^N W^\bullet(x, X_j) \hat{q}_N(X_j)^{-\theta} (f(X_j) - f(x)) \\
  &\quad = \underbrace{\frac{1}{N} \sum_{j=1}^N W^\bullet(x, X_j) \qe(X_j)^{-\theta} (f(X_j) - f(x))}_{(\star)} \\
  &\qquad + \underbrace{\frac{1}{N} \sum_{j=1}^N W^\bullet(x, X_j) \left( \hat{q}_N(X_j)^{-\theta} - \qe(X_j)^{-\theta} \right) (f(X_j) - f(x))}_{(\star\star)}.
\end{align*}
We can bound the second term $(\star\star)$ by
\begin{align*}
  |(\star\star)| &\leq \sup_{y \in \gM} \left| \hat{q}_N(y)^{-\theta} - \qe(y)^{-\theta} \right| \frac{1}{N} \sum_{j=1}^N W^{s}(x, X_j) |f(X_j) - f(x)| \\
  &\leq C \sup_{y \in \gM} \left| \hat{q}_N(y) - \qe(y) \right| \frac{1}{N} \sum_{j=1}^N W^{s}(x, X_j) |f(X_j) - f(x)| \\
  &\leq C \zeta_N \frac{1}{N} \sum_{j=1}^N W^{s}(x, X_j) |f(X_j) - f(x)|,
\end{align*}
where we used that $|W^\bullet| \leq W^{s}$, and that $t \mapsto t^{-\theta}$ is Lipschitz on $[\dmin/2, 2 C_q]$, which contains all the values of $\hat{q}_N$ and $\qe$ for $N$ large enough by \cref{eq:q_bounds,eq:q_uniform}. Note that the supremum over $y \in \gM$ controls in particular the values at the sample points $X_j$, even though they are random. The last factor is itself an empirical average of \iid random variables, and we show below that it is at most $C \eps$, so that $|(\star\star)| \leq C \eps \zeta_N$. The same reasoning applies to the denominator of $\hat{\rmP}^{(\theta,\bullet)}_N[f]$, without the factor $|f(X_j) - f(x)|$: its second term is at most $C \zeta_N \hat{q}_N(x) \leq C \zeta_N$, since $\hat{q}_N \leq 2 C_q$ for $N$ large. To conclude, we need to bound $(\star)$, which is an empirical average of \iid random variables. To do so, we must show that the lemma applies to the kernels $\psi_\bullet(y,x) = W^\bullet(x,y) \qe(y)^{-\theta} (f(y) - f(x))$ and $\psi_s(y,x) = W^{s}(x,y) \qe(y)^{-\theta}$.

By \cref{lemma:kernel_scale} and \cref{eq:q_bounds}, the kernel $\psi_s$ satisfies the assumptions of \cref{lemma:uniform_deviation} with $B \leq C \eps^{-m}$, $\sigma^2 \leq C \eps^{-m}$ and $L \leq C \eps^{-m-1}$, exactly as for $\hat{q}_N$. The kernel $\psi_\bullet$ carries the increment $f(y) - f(x)$, which is of size $\eps$ on the effective support of $W$. Since $f \in \gC^1(\gM)$ is $L_f$-Lipschitz and $|W^\bullet| \leq W^{s}$, \cref{lemma:kernel_scale} gives
\begin{align*}
  B &\leq C_\theta L_f \sup_{y \in \gM} W^{s}(x,y) \dist(x,y) \leq C \eps^{1-m}, \\
  \sigma^2 &\leq B \int_\gM |\psi_\bullet(y,x)| \diff\mathbb{P}(y) \leq C \eps^{1-m} C_\theta L_f \int_\gM W^{s}(x,y) \dist(x,y) \diff\mathbb{P}(y) \leq C \eps^{2-m},
\end{align*}
and $L \leq C \eps^{-m-1}$, since
\begin{equation*}
  |\psi_\bullet(y,x) - \psi_\bullet(y,x')| \leq 2 C_\theta \lVert f \rVert_\infty |W^\bullet(x,y) - W^\bullet(x',y)| + C_\theta L_f W^{s}(x',y) \dist(x,x').
\end{equation*}
The same bounds hold for the kernel $W^{s}(x,y) |f(y) - f(x)|$ of the last factor in the bound on $(\star\star)$, using $\big| |f(y) - f(x)| - |f(y) - f(x')| \big| \leq L_f \dist(x,x')$. For these kernels, $B$ and $\sigma$ are smaller than for $\psi_s$ by a factor $\eps$, and so is the first term of the right-hand side of \cref{eq:uniform_deviation}, the last one being negligible as seen above. Hence, almost surely for $N$ large enough,
\begin{align*}
  \sup_{x \in \gM} \left| \frac{1}{N} \sum_{j=1}^N W^\bullet(x,X_j) \qe(X_j)^{-\theta} (f(X_j) - f(x)) - \int_\gM W^\bullet(x,y) \qe(y)^{-\theta} (f(y) - f(x)) \diff\mathbb{P}(y) \right| &\leq C \eps \zeta_N,
\end{align*}
and similarly $\frac{1}{N}\sum_{j} W^{s}(x,X_j) |f(X_j) - f(x)| \leq \int_\gM W^{s}(x,y) |f(y) - f(x)| \diff\mathbb{P}(y) + C\eps\zeta_N \leq C \eps$ for every $x \in \gM$, which completes the bound on $(\star\star)$. The same reasoning applies to the denominator:
\begin{equation*}
  \sup_{x \in \gM} \left| \frac{1}{N} \sum_{j=1}^N W^{s}(x,X_j) \qe(X_j)^{-\theta} - \int_\gM W^{s}(x,y) \qe(y)^{-\theta} \diff\mathbb{P}(y) \right| \leq C \zeta_N.
\end{equation*}
Thus, it remains to show that the ratio of these two empirical averages is close to the ratio of their expectations.

Combining the bounds on $(\star)$ and $(\star\star)$, and their analogues for the denominator, we get that, almost surely for $N$ large enough and for every $x \in \gM$, the numerator $\mathcal{N}_N(x)$ and the denominator $\mathcal{D}_N(x)$ of $\hat{\rmP}^{(\theta,\bullet)}_N[f](x)$ satisfy
\begin{equation*}
  |\mathcal{N}_N(x) - \mathcal{N}(x)| \leq C \eps \zeta_N, \qquad |\mathcal{D}_N(x) - \mathcal{D}(x)| \leq C \zeta_N,
\end{equation*}
where $\mathcal{N}(x) = \int_\gM W^\bullet(x,y) \qe(y)^{-\theta} (f(y) - f(x)) \diff\mathbb{P}(y)$ and $\mathcal{D}(x) = \int_\gM W^{s}(x,y) \qe(y)^{-\theta} \diff\mathbb{P}(y)$ are the numerator and the denominator of $\Ptb[f](x)$, up to the common factor $\qe(x)^{-\theta}$. Moreover, since $\int_\gM W^{s}(x,y) \diff\mathbb{P}(y) = \qe(x)$, \cref{eq:q_bounds} and \cref{lemma:kernel_scale} give $\mathcal{D}(x) \geq c_\theta \qe(x) \geq c_\theta \dmin$ and $|\mathcal{N}(x)| \leq C_\theta L_f \int_\gM W^{s}(x,y) \dist(x,y) \diff\mathbb{P}(y) \leq C \eps$. In particular $\mathcal{D}_N(x) \geq c_\theta \dmin / 2$ for $N$ large. Adding and subtracting $\mathcal{N}(x)/\mathcal{D}_N(x)$, we get
\begin{align*}
  \left| \hat{\rmP}^{(\theta,\bullet)}_N[f](x) - \Ptb[f](x) \right| &= \left| \frac{\mathcal{N}_N(x) - \mathcal{N}(x)}{\mathcal{D}_N(x)} + \frac{\mathcal{N}(x) \big( \mathcal{D}(x) - \mathcal{D}_N(x) \big)}{\mathcal{D}_N(x) \mathcal{D}(x)} \right| \\
  &\leq \frac{2 C \eps \zeta_N}{c_\theta \dmin} + \frac{2 C \eps \cdot C \zeta_N}{(c_\theta \dmin)^2} \leq C \eps \zeta_N.
\end{align*}
The second term is small because the deviation of the denominator, of order $\zeta_N$, is multiplied by $\mathcal{N}(x)$, which is of order $\eps$. Since this holds for every $x \in \gM$, almost surely for $N$ large enough, this gives \cref{prop:fluctuation}, which concludes the proof. \hfill $\qed$

\subsubsection*{Combining the bias and the fluctuation}

Combining the fluctuation and the bias gives the quantitative form of \cref{thm:convergence_discrete_op}.

\begin{theorem} \label{thm:convergence_discrete_op_quantitative}
  Let $f \in \gC^3(\gM)$ and let $\eps = \eps(N) \to 0$ with $N\eps^m/\log N \to \infty$. Then, almost surely for $N$ large enough,
  \begin{equation} \label{eq:convergence_rates}
    \sup_{x \in \gM} \big| \hat{\rmL}^s_N[f](x) - \gL^s f(x) \big| \leq C \left( \eps + \frac{\zeta_N}{\eps} \right), \qquad
    \sup_{x \in \gM} \big| \hat{\rmL}^a_N[f](x) - \gL^a f(x) \big| \leq C \left( \eps + \zeta_N \right),
  \end{equation}
  with $\zeta_N$ as in \cref{eq:fluctuation_scale}. In particular, both right-hand sides vanish, and $\rmL^\bullet[f] \to \gL^\bullet f$ uniformly and almost surely, as soon as
  \begin{equation} \label{eq:scaling_condition}
    \eps(N) \to 0 \qquad \text{and} \qquad \frac{N \eps(N)^{m+2}}{\log N} \to \infty .
  \end{equation}
\end{theorem}

\begin{proof}
  By definition $\hat{\rmL}^s_N = \eps^{-2}\hat{\rmP}^{(\theta,s)}_N$, so that, splitting the error into its fluctuation and its bias, we have that:
  \begin{align*}
    \big| \hat{\rmL}^s_N[f] - \gL^s f \big|
    &\leq \frac{1}{\eps^2} \big| \hat{\rmP}^{(\theta,s)}_N[f] - \Pts[f] \big| + \frac{1}{\eps^2}\big| \Pts[f] - \eps^2 \gL^s f \big| \\
    &\leq \frac{C\eps\zeta_N}{\eps^2} + \frac{C\eps^3}{\eps^2} & \text{By \cref{prop:fluctuation,lemma:bias}} \\
    &= C\left( \frac{\zeta_N}{\eps} + \eps \right).
  \end{align*}
  The second bound is obtained in the same way, with $\eps^{-1}$ in place of $\eps^{-2}$. Under \cref{eq:scaling_condition}, we have that $\zeta_N/\eps = \sqrt{\log N/(N\eps^{m+2})} \to 0$, so that both bounds tend to $0$.
\end{proof}

\begin{remark}[Which scaling to choose] \label{rmk:joint_limit}
  The two terms of \cref{eq:convergence_rates} vary in opposite directions with $\eps$: a large bandwidth averages over many samples but blurs the geometry, while a small one is faithful but noisy. Balancing $\eps$ against $\zeta_N/\eps$ gives the optimal choice
  \begin{equation*}
    \eps(N) \asymp \left( \frac{\log N}{N} \right)^{\frac{1}{m+4}}, \qquad \text{for an error of order } \left( \frac{\log N}{N} \right)^{\frac{1}{m+4}},
  \end{equation*}
  which exhibits the usual curse of dimensionality in the intrinsic dimension $m$, as in the Riemannian case \citep{singer_from_2006, hein2007graph}. We note that the antisymmetric part is the easier of the two: being rescaled by $\eps^{-1}$ instead of $\eps^{-2}$, it only requires $N\eps^m/\log N \to \infty$, whereas the symmetric part requires the stronger condition \cref{eq:scaling_condition}. In that sense, estimating the direction of the data is statistically cheaper than estimating its geometry.
\end{remark}

\subsection{Computation of the moments in the Randers setting} \label{appendix:proofs:prop:randers_moments}
If $\FF$ is a Randers metric, the moments of the unit ball $\tball_x$ can be computed in closed form. We recall that a Randers metric is defined as $\FF(x,v) = \sqrt{v^\top \rmA(x) v} + b(x)^\top v$, where $\rmA(x)$ is a positive-definite matrix and $b(x)$ is a vector field such that $\lVert b(x)\rVert_{\rmA^{-1}} < 1$.

\begin{proposition} \label{prop:centroid_randers_metric}
  For a Randers metric, we have that:
  \begin{align}
    \leb_x(\tball_x) &= \vole(\eball^m) (1 - \lVert b(x) \rVert_{\rmA^{-1}}^2)^{-\frac{m+1}{2}} \sqrt{\det \rmA^{-1}(x)},\\
    \rvc^i(x) &= -\frac{b^i(x)}{1 - \lVert b(x) \rVert_{\rmA^{-1}}^2},\\
    \rmS^{ij}(x) &= \frac{R_x^2}{m+2} \rmE^{ij}(x) + \rvc^i(x) \rvc^j(x),
  \end{align}
  where $R_x^2 = \frac{1}{1 - \lVert b(x) \rVert_{\rmA^{-1}}^2}$, $b^i(x)=\rmA^{ij}(x)b_j(x)$, $\rmE_{ij}(x) = \rmA_{ij}(x) - b_i(x) b_j(x)$ and $\rmE^{ij} = (\rmE_{ij})^{-1} = \rmA^{ij} + \frac{b^i b^j}{1 - \lVert b(x) \rVert_{\rmA^{-1}}^2}$.
\end{proposition}

\begin{proof}
  Using \citep[Lemma 10.8]{ohtaComparisonFinslerGeometry2021}, we have that:
  \begin{align*}
    \tball_x = \{v \in T_x\gM :  \rmE_{ij}(x)(v - v_0)^i(v - v_0)^j \leq R_x^2\}
  \end{align*}
  where $\rmE_{ij}(x) = \rmA_{ij}(x) - b_i(x) b_j(x)$, $v_0^i = -\frac{\rmA^{ij}(x)b_j(x)}{1 - \lVert b(x) \rVert_{\rmA^{-1}}^2}$ and $R_x^2 = \frac{1}{1 - \lVert b(x) \rVert_{\rmA^{-1}}^2}$. This shows that the unit tangent ball of a Randers metric is an ellipsoid, described by the matrix $\rmE(x)$, centered at $v_0$ and with a radius $R_x$.

  We can compute the volume of the unit ball by using the change of variable $u = \rmE^{1/2}(x)(v - v_0)$, which yields that:
  \begin{align*}
    \leb_x(\tball_x) &= \frac{R_x^m}{\sqrt{\det \rmE(x)}} \vole(\eball^m) = \vole(\eball^m) R_x^m \sqrt{\det \rmE^{-1}(x)}
  \end{align*}
  We have the expression of $R_x^m$: $R_x^m = (1 - \lVert b(x) \rVert_{\rmA^{-1}}^2)^{-\frac{m}{2}}$. Moreover, the matrix determinant lemma yields $\det(\rmE) = \det(\rmA- b b^\top) = \det(\rmA) (1 - \lVert b(x) \rVert_{\rmA^{-1}}^2)$, which yields that $\sqrt{\det \rmE^{-1}(x)} = \sqrt{\det \rmA^{-1}(x)} (1 - \lVert b(x) \rVert_{\rmA^{-1}}^2)^{-\frac{1}{2}}$. Hence the expression of the volume of the unit ball as:
  \begin{align*}
    \leb_x(\tball_x) &= \vole(\eball^m) (1 - \lVert b(x) \rVert_{\rmA^{-1}}^2)^{-\frac{m+1}{2}} \sqrt{\det \rmA^{-1}(x)}.
  \end{align*}

  Now we can compute the first moments. Because of the symmetry around the center $v_0$, we have that $\int_{\tball_x} (v-v_0)^i \diff\leb_x(v) = 0$. This yields that $\int_{\tball_x} v^i \diff\leb_x(v) = \leb_x(\tball_x) v_0^i$ and thus:
  \begin{align*}
    \rvc^i(x) &= \frac{\int_{\tball_x} v^i \diff\leb_x(v)}{\leb_x(\tball_x)} = v_0^i = -\frac{\rmA^{ij}(x)b_j(x)}{1 - \lVert b(x) \rVert_{\rmA^{-1}}^2} = -\frac{b^i(x)}{1 - \lVert b(x) \rVert_{\rmA^{-1}}^2}.
  \end{align*}

  And to compute the second moments, we observe that:
  \begin{align*}
    v^i v^j &= (v - v_0)^i (v - v_0)^j + (v - v_0)^i v_0^j + v_0^i (v - v_0)^j + v_0^i v_0^j.
  \end{align*}
  Here the cross terms vanish when integrated over $\tball_x$ because of the symmetry around the center $v_0$. Using the change of variable $u = \rmE^{1/2}(x)(v - v_0)$, we can compute the first term as:
  \begin{align*}
    \int_{\tball_x} (v - v_0)^i (v - v_0)^j \diff\leb_x(v) &= \int_{\{u : \lVert u \rVert_2^2 < R_x^2\}} (\rmE^{-1/2}(x) u)^i (\rmE^{-1/2}(x) u)^j \det(\rmE^{-1/2}(x)) \diff u\\
    &= \det(\rmE^{-1/2}(x)) (\rmE^{-1/2}(x))^i_k (\rmE^{-1/2}(x))^j_l \int_{\{u : \lVert u \rVert_2^2 < R_x^2\}} u^k u^l \diff u\\
    &= \det(\rmE^{-1/2}(x)) (\rmE^{-1/2}(x))^i_k (\rmE^{-1/2}(x))^j_l \frac{R_x^{m+2}}{m+2} \vole(\eball^m) \delta^{kl}\\
    &= \frac{R_x^{m+2}}{m+2} \vole(\eball^m) \rmE^{ij}(x) \sqrt{\det \rmE^{-1}(x)} \\
    &= \frac{R_x^2}{m+2} \leb_x(\tball_x) \rmE^{ij}(x).
  \end{align*}
  where we have that $\int_{\{u : \lVert u \rVert_2^2 < R_x^2\}} u^k u^l \diff u = \frac{R_x^{m+2}}{m+2} \vole(\eball^m) \delta^{kl}$ by rotation invariance of the Euclidean ball.

  So we have that:
  \begin{align*}
    \rmS^{ij}(x) &= \frac{\int_{\tball_x} v^i v^j \diff\leb_x(v)}{\leb_x(\tball_x)} = \frac{\int_{\tball_x} (v - v_0)^i (v - v_0)^j \diff\leb_x(v) + \leb_x(\tball_x) v_0^i v_0^j}{\leb_x(\tball_x)}\\
    &= \frac{R_x^2}{m+2} \rmE^{ij}(x) + v_0^i v_0^j.
  \end{align*}

\end{proof}

\begin{remark}
  We can see that in the Randers setting, contrary to the Riemannian setting, the centroid of the unit ball is not zero but is shifted proportionally to the vector field $b$. Similarly, the second moment of the unit ball is not proportional to the inverse of the matrix $\rmA$ but has an additional term proportional to the outer product of the centroid with itself. When $b\to 0$, we recover the Riemannian setting with $\rvc^i(x) \to 0$ and $\rmS^{ij}(x) \to \frac{1}{m+2} \rmA^{ij}(x)$.
\end{remark}

\begin{proposition}\label{prop:randers_moments}
  In the Randers setting, the moments of the kernel $m_0$, $\tilde{m}_1$ and $\tilde{m}_2$ are given by:
  \begin{align}
    m_0(x) &= m \mu_0 \vole(\eball^m),\\
    \tilde{m}_1^i(x) &= -\frac{(m+1)\mu_1}{m\mu_0} \frac{b^i(x)}{1 - \lVert b(x) \rVert_{\rmA^{-1}}^2},\\
    \tilde{m}_2^{ij}(x) &=  \frac{\mu_2}{\mu_0} \frac{m+2}{m} \left(\frac{R_x^2}{m+2} \rmE^{ij}(x) + \rvc(x)^i \rvc(x)^j\right).
  \end{align}
\end{proposition}

\begin{proof}
  This is a direct application of \cref{prop:centroid_randers_metric} and the expressions of the moments of the kernel in \cref{props:moment_expression_unitball}.
\end{proof}

\subsection{The moment-Randers approximation}
\label{appendix:randers_approximation}

This subsection proves the results of \cref{sec:randers_approximation}. We proceed in three steps: we first check that the moment-Randers metric of \cref{def:moment_randers} is well defined, identify its unit ball, and show that it matches the centroid and the second moment of $\FF$; we then show how $\rvc$ and $\gBL$ are recovered from $\gL^s$ and $\gL^a$ through an embedding; and we finally relate the resulting empirical metric $\hat{\FF}_R$ to $\FF_R$.

Throughout, we work in a chart, so that $\rvc(x) \in \R^m$, $\rmS(x) \in \R^{m \times m}$ and $\tball_x \subset \R^m$. We recall from \cref{eq:centroid,eq:raw_second_moment,eq:binet_legendre} that
\begin{equation*}
  \rvc(x) = \frac{1}{\leb_x(\tball_x)}\int_{\tball_x} v \diff\leb_x(v), \qquad
  \rmS(x) = \frac{1}{\leb_x(\tball_x)}\int_{\tball_x} v v^\top \diff\leb_x(v), \qquad
  \gBL(x)^{-1} = (m+2)\rmS(x),
\end{equation*}
so that the matrix $\rmH$ of \cref{def:moment_randers} can be written in the two equivalent ways
\begin{equation*}
  \rmH(x)^{-1} = \gBL(x)^{-1} - (m+2)\rvc(x)\rvc(x)^\top = (m+2)\big(\rmS(x) - \rvc(x)\rvc(x)^\top\big).
\end{equation*}

\subsubsection*{Well-definedness of the moment-Randers metric}

We begin by proving \cref{prop:H_vs_gBL}, which gives a necessary and sufficient condition for $\rmH$ to be positive definite, and relates the norms of $\rvc$ in the two metrics. We restate it here for convenience:

\HvsGBLProp*

\begin{proof}[Proof of \cref{prop:H_vs_gBL}]
We first check that $\rmH(x)$ is positive definite, hence invertible. The matrix $\rmH^{-1}$ is a rank-one downdate of $\gBL^{-1}$, so its positivity only depends on the size of $\rvc$: for any $u \in T_x^*\gM$, the Cauchy--Schwarz inequality gives $\langle \rvc, u\rangle^2 \leq \lVert \rvc\rVert_{\gBL}^2 \lVert u\rVert_{\gBL^{-1}}^2$, so that
\begin{equation*}
  u^\top \rmH^{-1} u = \lVert u \rVert_{\gBL^{-1}}^2 - (m+2)\langle \rvc, u\rangle^2
    \;\geq\; \Big(1 - (m+2)\lVert \rvc\rVert_{\gBL}^2\Big) \lVert u \rVert_{\gBL^{-1}}^2 ,
\end{equation*}
which is positive for every $u \neq 0$ under the hypothesis. The bound is attained at $u = \gBL\rvc$, the equality case of Cauchy--Schwarz, where $u^\top \rmH^{-1}u = \lVert \rvc\rVert_{\gBL}^2\big(1 - (m+2)\lVert \rvc\rVert_{\gBL}^2\big)$, so the hypothesis is also necessary when $\rvc \neq 0$.

We now turn to the identity. We recall that $\rmH^{-1} = (m+2)(\rmS - \rvc\rvc^\top)$, so that $\rmS = \frac{1}{m+2}\rmH^{-1} + \rvc\rvc^\top$. Setting $A = \frac{1}{m+2}\rmH^{-1}$, so that $A^{-1} = (m+2)\rmH$, the Sherman--Morrison formula yields
\begin{align*}
  \rmS^{-1} &= (A + \rvc\rvc^\top)^{-1}
  = A^{-1} - \frac{A^{-1}\rvc\rvc^\top A^{-1}}{1 + \rvc^\top A^{-1}\rvc}
  = (m+2)\rmH - \frac{(m+2)^2 \rmH \rvc\rvc^\top \rmH}{1 + (m+2)\rvc^\top \rmH \rvc}.
\end{align*}
Hence
\begin{align*}
  \rvc^\top \rmS^{-1} \rvc
  &= (m+2)\lVert \rvc \rVert_{\rmH}^2 - \frac{(m+2)^2 \lVert \rvc \rVert_{\rmH}^4}{1 + (m+2)\lVert \rvc \rVert_{\rmH}^2}
  = \frac{(m+2)\lVert \rvc \rVert_{\rmH}^2}{1 + (m+2)\lVert \rvc \rVert_{\rmH}^2},
\end{align*}
and, because $\gBL^{-1} = (m+2)\rmS$,
\begin{equation*}
  \lVert \rvc \rVert_{\gBL}^2 = \rvc^\top \gBL \rvc = \frac{1}{m+2} \rvc^\top \rmS^{-1} \rvc = \frac{\lVert \rvc\rVert_{\rmH}^2}{1+(m+2)\lVert \rvc\rVert_{\rmH}^2}.
\end{equation*}
The right-hand side is an increasing function of $\lVert \rvc\rVert_{\rmH}^2$, and
\begin{align*}
  \lVert \rvc\rVert_{\gBL}^2 < \frac{1}{m+3}
  &\iff (m+3)\lVert \rvc\rVert_{\rmH}^2 < 1 + (m+2)\lVert \rvc\rVert_{\rmH}^2
  \iff \lVert \rvc\rVert_{\rmH}^2 < 1.
\end{align*}
This concludes the proof of \cref{prop:H_vs_gBL}.
\end{proof}

The metric of \cref{def:moment_randers} is written in the $\alpha + \beta$ form \cref{eq:randers_main}. The following remark shows that it can equivalently be written in the \emph{navigation form}.

\begin{remark}[Equivalent navigation form] \label{rem:moment_randers_alpha_beta}
  With the $\alpha + \beta$ form of \cref{def:moment_randers}, we have $\rmA_R = \lambda^{-1}\rmH + \lambda^{-2}(\rmH\rvc)(\rmH\rvc)^\top$ and $b_R = -\lambda^{-1}\rmH\rvc$, where $\lambda = 1 - \lVert \rvc\rVert_{\rmH}^2$. Hence
  \begin{equation*}
  v^\top \rmA_R v = \frac{\lambda \lVert v \rVert_{\rmH}^2 + \langle \rvc, v \rangle_{\rmH}^2}{\lambda^2}, \qquad b_R^\top v = - \frac{\langle \rvc, v \rangle_{\rmH}}{\lambda},
  \end{equation*}
  so that $\FF_R$ can equivalently be written as
  \begin{equation} \label{eq:moment_randers_navigation}
    \FF_R(x,v) = \frac{1}{\lambda(x)}\left( \sqrt{\lambda(x) \lVert v \rVert_{\rmH}^2 + \langle \rvc(x), v \rangle_{\rmH}^2 } - \langle \rvc(x), v \rangle_{\rmH} \right), \qquad \lambda(x) = 1 - \lVert \rvc(x) \rVert_{\rmH}^2.
  \end{equation}
  Importantly, the admissibility condition $\lVert b_R(x)\rVert_{\rmA_R^{-1}} < 1$ is equivalent to $\lVert \rvc(x)\rVert_{\rmH} < 1$, which is exactly the hypothesis of \cref{def:moment_randers}. Indeed,
  \begin{equation*}
    \lVert b_R\rVert_{\rmA_R^{-1}}^2 = b_R^\top \rmA_R^{-1} b_R = \lambda^2 \rvc^\top \rmH^{-1} \left( \lambda(\rmH^{-1} - \rvc\rvc^\top) \right) \rmH^{-1} \rvc = \lambda^2 (\lVert \rvc\rVert_{\rmH}^2 - \lVert \rvc\rVert_{\rmH}^4) / \lambda^2 = \lVert \rvc\rVert_{\rmH}^2.
  \end{equation*}
\end{remark}

We can now justify the name: $\FF_R$ is the unique Randers metric whose unit ball has the same centroid and the same second moment as $\tball_x$, and it therefore leaves $\tilde{m}_1$ and $\tilde{m}_2$ unchanged.

\begin{proposition}[Moment matching] \label{prop:moment_randers_moments}
  Let $\FF$ be a Finsler metric with $\lVert \rvc(x)\rVert_{\rmH} < 1$ for all $x \in \gM$, and let $\FF_R$ be its moment-Randers metric (\cref{def:moment_randers}). Then $\FF_R$ is a Randers metric and:
  \begin{enumerate}[itemsep=0em, topsep=0em]
    \item its unit tangent ball is the ellipsoid $E_x = \{v \in T_x\gM : \lVert v - \rvc(x)\rVert_{\rmH} \leq 1\}$;
    \item $\FF$ and $\FF_R$ have the same centroid and the same raw second moment, $\rvc_R = \rvc$ and $\rmS_R = \rmS$; consequently they have the same Binet--Legendre metric and the same normalized kernel moments $\tilde{m}_1$ and $\tilde{m}_2$;
    \item if $\FF$ is itself a Randers metric, then $\FF_R = \FF$.
  \end{enumerate}
\end{proposition}

\begin{proof}
   \emph{Unit tangent ball.}  We start by showing that $\FF_R$ is a valid Randers metric. If $\lVert \rvc(x)\rVert_{\rmH} < 1$, then $\lambda(x) > 0$. Consequently we have that $\FF_R(x,v) \geq 0$ for all $v \in T_x\gM$. Now if we assume that $\FF_R(x,v) = 0$ for some $v \in T_x\gM$, then we have that $\lambda(x) \lVert v \rVert_{\rmH}^2 + \langle \rvc(x), v \rangle_{\rmH}^2 = \langle \rvc(x), v \rangle_{\rmH}^2$, which implies that $\lambda(x) \lVert v \rVert_{\rmH}^2 = 0$. Since $\lambda(x) > 0$, we have that $\lVert v \rVert_{\rmH}^2 = 0$, which implies that $v = 0$. Therefore, $\FF_R(x,v) = 0$ if and only if $v = 0$. Positive homogeneity and the triangle inequality are trivial. Hence, $\FF_R$ is a valid Randers metric.

  Now let $v \in T_x\gM$. We have the following chain of equivalences:
  \begin{align*}
    \FF_R(x,v) \leq 1 &\iff
      \sqrt{\lambda(x) \lVert v \rVert_{\rmH}^2 + \langle \rvc(x), v \rangle_{\rmH}^2 } - \langle \rvc(x), v \rangle_{\rmH} \leq \lambda(x)\\
      &\iff\sqrt{\lambda(x) \lVert v \rVert_{\rmH}^2 + \langle \rvc(x), v \rangle_{\rmH}^2 } \leq \lambda(x) + \langle \rvc(x), v \rangle_{\rmH}\\
      &\iff\lambda(x) \lVert v \rVert_{\rmH}^2 + \langle \rvc(x), v \rangle_{\rmH}^2 \leq (\lambda(x) + \langle \rvc(x), v \rangle_{\rmH})^2\\
      &\iff\lambda(x) \lVert v \rVert_{\rmH}^2 + \langle \rvc(x), v \rangle_{\rmH}^2 \leq \lambda(x)^2 + 2\lambda(x)\langle \rvc(x), v \rangle_{\rmH} + \langle \rvc(x), v \rangle_{\rmH}^2\\
      &\iff\lambda(x) \lVert v \rVert_{\rmH}^2 \leq \lambda(x)^2 + 2\lambda(x)\langle \rvc(x), v \rangle_{\rmH}\\
      &\iff\lVert v \rVert_{\rmH}^2 \leq \lambda(x) + 2\langle \rvc(x), v \rangle_{\rmH}\\
      &\iff\lVert v - \rvc(x) + \rvc(x) \rVert_{\rmH}^2 \leq \lambda(x) + 2\langle \rvc(x), v \rangle_{\rmH} \\
      &\iff\lVert v - \rvc(x) \rVert_{\rmH}^2 + 2\langle v - \rvc(x), \rvc(x) \rangle_{\rmH} + \lVert \rvc(x) \rVert_{\rmH}^2 \leq 1-\lVert \rvc(x) \rVert_{\rmH}^2 + 2\langle \rvc(x), v \rangle_{\rmH} \\
      &\iff\lVert v - \rvc(x) \rVert_{\rmH}^2 + 2\langle v, \rvc(x) \rangle_{\rmH} - 2\lVert \rvc(x) \rVert_{\rmH}^2 + \lVert \rvc(x) \rVert_{\rmH}^2 \leq 1-\lVert \rvc(x) \rVert_{\rmH}^2 + 2\langle \rvc(x), v \rangle_{\rmH} \\
      &\iff\lVert v - \rvc(x) \rVert_{\rmH}^2 \leq 1
  \end{align*}

  \emph{Moment matching.} By the first item above, the unit tangent ball of $\FF_R$ is $E_x = \{v \in T_x\gM : \lVert v - \rvc(x)\rVert_{\rmH} \leq 1\}$. Hence its centroid is $\rvc(x)$.

  We now compute the second moment of the unit tangent ball of the moment-Randers metric:
  \begin{align*}
    \rmS_R(x) &= \frac{1}{\leb_x(E_x)}\int_{E_x} v v^\top \diff\leb_x(v) \\
    &= \frac{1}{\leb_x(E_x)}\int_{E_x} (v - \rvc(x) + \rvc(x))(v - \rvc(x) + \rvc(x))^\top \diff\leb_x(v) \\
    &= \frac{1}{\leb_x(E_x)}\int_{E_x} (v - \rvc(x))(v - \rvc(x))^\top \diff\leb_x(v) + \rvc(x)\rvc(x)^\top,
  \end{align*}
  where we used that $\int_{E_x} (v - \rvc(x)) \diff\leb_x(v) = 0$. Now, as $E_x$ is an ellipsoid, we have that $\int_{E_x} (v - \rvc(x))(v - \rvc(x))^\top \diff\leb_x(v) = \frac{1}{m+2}\rmH(x)^{-1}$, and thus
  \begin{align*}
    \rmS_R(x) &= \frac{1}{m+2} \rmH(x)^{-1} + \rvc(x)\rvc(x)^\top \\
    &= \rmS(x)
  \end{align*}
  where we used that $\rmH(x)^{-1} = (m+2)(\rmS(x) - \rvc(x)\rvc(x)^\top)$.

  \emph{Uniqueness.} Since $\tilde{\FF}(x,\cdot)$ is positively homogeneous of degree $1$ for any Finsler metric $\tilde{\FF}$, we have $\tilde{\FF}(x,v/t) = \tilde{\FF}(x,v)/t$ for $t>0$, so that $v/t \in \{\tilde{\FF}(x,\cdot) \leq 1\} \iff \tilde{\FF}(x,v) \leq t$, and thus
  \begin{equation*}
    \tilde{\FF}(x,v) = \inf\{t > 0 : v/t \in \{\tilde{\FF}(x,\cdot) \leq 1\}\}, \qquad \forall v \in T_x\gM.
  \end{equation*}
  In particular, a Finsler metric is determined by its unit tangent balls: two Finsler metrics with the same unit tangent ball at every $x \in \gM$ are equal. The goal of this paragraph is to show that, for a Randers metric, the unit tangent ball is uniquely determined by its centroid and its raw second moment.

  By definition, the centroid $\rvc$ and the raw second moment $\rmS$ of a Randers metric are given by
  \begin{align*}
    \rvc &= \frac{1}{\leb_x(\tball_x)}\int_{\tball_x} v \diff\leb_x(v), \qquad
    \rmS = \frac{1}{\leb_x(\tball_x)}\int_{\tball_x} v v^\top \diff\leb_x(v),
  \end{align*}
  where $\tball_x = \{v \in T_x\gM : \FF(x,v) \leq 1\}$ is the unit tangent ball of the Randers metric. By \cref{prop:centroid_randers_metric}, we have that $\tball_x$ is an ellipsoid, and by \cref{prop:randers_moments}, we have that $\rvc$ and $\rmS$ are given by explicit formulas in terms of the parameters $\rmA$ and $b$ of the Randers metric. These formulas are invertible, so that $\rmA$ and $b$ can be recovered from $\rvc$ and $\rmS$. Hence the unit tangent ball $\tball_x$ is uniquely determined by $\rvc$ and $\rmS$, and so is the Randers metric $\FF$.
\end{proof}

\subsubsection*{Recovering the centroid and the Binet--Legendre metric from an embedding}

Let $\Psi = (\psi_1, \ldots, \psi_\ell) : \gM \to \R^\ell$ be a smooth embedding, and denote by $J(x) \in \mathbb{R}^{\ell \times m}$ its Jacobian at $x$, whose $k$-th row is $\nabla\psi_k(x)^\top$. We assume throughout that $J(x)$ has full row rank $m$, so that $\ell \geq m$ and $J(x)^+ = (J(x)^\top J(x))^{-1} J(x)^\top$ is the pseudo-inverse of $J(x)$, \ie $J(x)^+ J(x) = I_m$. Operators are applied to $\Psi$ coordinate-wise, so that $\gL^a[\Psi] = (\gL^a[\psi_1], \ldots, \gL^a[\psi_\ell])^\top$. Recall from \cref{props:moment_expression_unitball} and \cref{thm:sym_limit} that the anti-symmetric operator $\gL^a$ is given by:
\begin{align*}
  \gL^a[f](x) = \big\langle \tilde{m}_1(x), \nabla f(x) \big\rangle = c_1 \big\langle \rvc(x), \nabla f(x) \big\rangle,
\end{align*}
where $\rvc(x)$ is the centroid of the unit ball $\tball_x$ and $c_1 = ((m+1)\mu_1)/(m\mu_0)$ is a constant that depends on the kernel. Applying $\gL^a$ to the embedding $\Psi$ yields:
\begin{align*}
  V(x):=\gL^a[\Psi](x) &= c_1 \begin{pmatrix} \big\langle \rvc(x), \nabla \psi_1(x) \big\rangle \\ \vdots \\ \big\langle \rvc(x), \nabla \psi_\ell (x) \big\rangle \end{pmatrix} = c_1 J(x) \rvc(x).
\end{align*}
Consequently, we can recover the centroid $\rvc(x)$ from $\gL^a\Psi(x)$ by
\begin{align*}
  \rvc(x) &= \frac{1}{c_1} J(x)^+ \gL^a\Psi(x).
\end{align*}

The Binet--Legendre metric is recovered in the same way, from the symmetric part, through the carré du champ of $\gL^s$. The point is that this bilinear form only retains the principal part of $\gL^s$: the sampling density $\rho$ and the normalization $\theta$, which enter $\gL^s$ only through first-order terms, disappear. This is the content of \cref{prop:carre_du_champ}, which we now prove.
\begin{proof}[Proof of \cref{prop:carre_du_champ}]
  By \cref{thm:sym_limit}, in a chart $\gL^s$ is a second-order differential operator with principal part $c_2 \gBL^{ij}\partial_i\partial_j$ and no zeroth-order term. We denote by $\xi$ the coefficients collecting all first-order contributions, so that $\gL^s = c_2 \gBL^{ij}\partial_i\partial_j + \xi^i\partial_i$. For smooth $f$ and $h$, the Leibniz rule gives $\partial_i(fh) = f\partial_i h + h\partial_i f$ and
  \begin{equation*}
    \partial_i\partial_j(fh) = f\partial_i\partial_j h + h \partial_i\partial_j f + \partial_i f\, \partial_j h + \partial_j f\, \partial_i h ,
  \end{equation*}
  so that all terms involving $\xi$ and all second derivatives of $f$ and $h$ cancel in the combination $\gL^s[fh] - f\gL^s[h] - h\gL^s[f]$. Using the symmetry of $\gBL^{ij}$, we are left with
  \begin{equation*}
    \frac{1}{2}\Big(\gL^s[fh] - f\gL^s[h] - h\gL^s[f]\Big) = c_2\, \gBL^{ij}\, \partial_i f\, \partial_j h .
  \end{equation*}
  Applying this with $f = \psi_k$ and $h = \psi_l$ gives $\Gamma^{kl} = c_2\, \gBL^{ij}\partial_i\psi_k \partial_j \psi_l$, which is the $(k,l)$ entry of $c_2 J \gBL^{-1}J^\top$, with $c_2 = \frac{\mu_2}{2m\mu_0}$.
\end{proof}

\cref{prop:carre_du_champ} gives the following corollary, which allows us to compute the norm of the drift in the Binet--Legendre metric from the embedding coordinates $\Psi(x)$ and the operators $\gL^s$ and $\gL^a$.
\begin{corollary} \label{cor:centroid_norm_estimate}
  If $J(x)$ has full column rank, then
  \begin{equation*}
    \Gamma(x)^+ = \frac{1}{c_2} J(x)^{+\top} \gBL(x) J(x)^+ , \qquad \text{and} \qquad
    \lVert \rvc(x)\rVert^2_{\gBL} = \frac{c_2}{c_1^2} \lVert V(x) \rVert^2_{\Gamma^+} .
  \end{equation*}
\end{corollary}
Consequently, by \cref{prop:H_vs_gBL}, the moment-Randers metric of $\FF$ is well defined at $x$ if and only if $\frac{c_2}{c_1^2} \lVert V(x)\rVert^2_{\Gamma^+} < (m+3)^{-1}$, which is computable from $\gL^s$ and $\gL^a$ alone.
\begin{proof}
  Fix $x \in \gM$, write $M(x) = c_2^{-1}J(x)^{+\top}\gBL(x)J(x)^+$, and let $P(x) = J(x)J(x)^+$ be the orthogonal projector onto the range of $J(x)$\footnote{Since $J(x)$ has full column rank, $J(x)^+ = \big(J(x)^\top J(x)\big)^{-1}J(x)^\top$, so $P(x) = J(x)\big(J(x)^\top J(x)\big)^{-1}J(x)^\top$ is symmetric.}; both $\Gamma(x)$ and $M(x)$ are symmetric. Using \cref{prop:carre_du_champ} and $J(x)^+J(x) = I_m$,
  \begin{equation*}
    \Gamma(x) M(x) = J(x) \gBL(x)^{-1}\big(J(x)^+J(x)\big)^\top \gBL(x)J(x)^+ = J(x)J(x)^+ = P(x).
  \end{equation*}
  Since $P(x)$ is symmetric, we also have $M(x)\Gamma(x) = P(x)$. Applying this, it holds that:
  \begin{equation*}
    \Gamma(x)M(x)\Gamma(x) = \Gamma(x)P(x) = \Gamma(x), \qquad
    M(x)\Gamma(x)M(x) = M(x)P(x) = M(x).
  \end{equation*}
  Together with the symmetry of $\Gamma(x)M(x)$ and $M(x)\Gamma(x)$, these are exactly the Penrose conditions characterizing the Moore--Penrose pseudoinverse, so $M(x) = \Gamma(x)^+$. The second identity then follows from $V(x) = c_1 J(x) \rvc(x)$ and the definition of the pseudo-inverse:
  \begin{equation*}
    \lVert \rvc(x)\rVert^2_{\gBL} = \frac{1}{c_1^2} V(x)^\top J(x)^{+\top} \gBL(x) J(x)^+ V(x) = \frac{c_2}{c_1^2} V(x)^\top \Gamma(x)^+ V(x) = \frac{c_2}{c_1^2} \norm{V(x)}_{\Gamma^+}^2 . \qedhere
  \end{equation*}
\end{proof}

To sum up, we can recover both the direction and the magnitude, in the Binet--Legendre metric, of the vector field $\rvc(x)$ from the vector field $V(x)$ in the embedding space. Note that neither the sampling density $\rho$ nor the normalization exponent $\theta$ appears in $\lVert \rvc(x) \rVert^2_{\gBL}$.

\subsubsection*{The empirical moment-Randers metric}

Building on the above, we can now define a metric from the observed quantities $\Gamma(x)$ and $V(x)$. Following \cref{sec:randers_approximation}, set
\begin{align} \label{eq:empirical_randers_data}
  &\hat{g}_{BL}(x)^{-1} = \Gamma(x), \qquad
  &&\hat{\rvc}(x) = \frac{\sqrt{c_2}}{c_1}\, V(x), \\
  &\hat{\rmH}(x)^{-1} = \Gamma(x) - (m+2)\frac{c_2}{c_1^2} V(x)V(x)^\top, \qquad
  &&\hat{\lambda}(x) = 1 - \lVert \hat{\rvc}(x)\rVert^2_{\hat{\rmH}} , \notag
\end{align}
with $\hat{\rmH} := (\hat{\rmH}(x)^{-1})^+$. Let $\hat{\FF}_R$ be defined by \cref{def:moment_randers} with $(\hat{\rvc}, \hat{\rmH}, \hat{\lambda})$, so that $\hat{\FF}_R$ is defined as:
\begin{align*}
  \hat{\FF}_R(x, v) = \frac{1}{\hat{\lambda}(x)}\left( \sqrt{\hat{\lambda}(x) \lVert v \rVert_{\hat{\rmH}}^2 + \langle \hat{\rvc}(x), v \rangle_{\hat{\rmH}}^2 } - \langle \hat{\rvc}(x), v \rangle_{\hat{\rmH}} \right).
\end{align*}

\begin{lemma} \label{lem:empirical_pushforward}
  If $J(x)$ has full column rank, then $\hat{\rmH}^{-1}$ and $\hat{\rvc}$ are the pushforwards of $\rmH^{-1}$ and $\rvc$ by the Jacobian $J(x)$, \ie
  \begin{align*}
    \hat{\rmH}^{-1}(x) &= c_2 J(x) \rmH^{-1}(x) J(x)^\top, \qquad
    \hat{\rvc}(x) = \sqrt{c_2} J(x) \rvc(x), \qquad \hat{\rmH}(x) = \frac{1}{c_2} J(x)^{+\top} \rmH(x) J(x)^{+},
  \end{align*}
  and $\hat{\lambda}(x) = \lambda(x)$.
\end{lemma}

In particular, this lemma implies that $\norm{\hat{\rvc}(x)}_{\hat{\rmH}} = \norm{\rvc(x)}_{\rmH}$, so that the admissibility condition $\norm{\hat{\rvc}(x)}_{\hat{\rmH}} < 1$ is equivalent to $\norm{\rvc(x)}_{\rmH} < 1$, which is exactly the hypothesis of \cref{def:moment_randers}. Hence, if $J(x)$ has full column rank, then $\hat{\FF}_R$ is a valid Randers metric on the embedded tangent space $J(T_x\gM)$.

\begin{proof}
  By \cref{prop:carre_du_champ}, $\Gamma(x) = c_2 J(x) \gBL(x)^{-1} J(x)^\top$, and $V(x)V(x)^\top = c_1^2 J(x)\rvc(x)\rvc(x)^\top J(x)^\top$ since $V(x) = c_1 J(x)\rvc(x)$. Substituting into the definition of $\hat{\rmH}^{-1}(x)$ and using $\rmH(x)^{-1} = \gBL(x)^{-1} - (m+2)\rvc(x)\rvc(x)^\top$ (\cref{def:moment_randers}),
  \begin{align*}
    \hat{\rmH}^{-1}(x) &= c_2 J(x) \gBL(x)^{-1} J(x)^\top - (m+2) c_2 J(x) \rvc(x)\rvc(x)^\top J(x)^\top \\
    &= c_2 J(x) \left( \gBL(x)^{-1} - (m+2) \rvc(x)\rvc(x)^\top \right) J(x)^\top \\
    &= c_2 J(x) \rmH(x)^{-1} J(x)^\top,
  \end{align*}
  and, similarly, $\hat{\rvc}(x) = \frac{\sqrt{c_2}}{c_1}\, V(x) = \frac{\sqrt{c_2}}{c_1}\, c_1 J(x) \rvc(x) = \sqrt{c_2} J(x) \rvc(x)$.

  For $\hat{\rmH}(x)$: the identity $\hat{\rmH}^{-1}(x) = c_2 J(x) \rmH(x)^{-1} J(x)^\top$ just established has exactly the same form as $\Gamma(x) = c_2 J(x) \gBL(x)^{-1} J(x)^\top$, with $\gBL(x)^{-1}$ replaced by $\rmH(x)^{-1}$. Since $\rmH(x)^{-1} \succ 0$ (\cref{prop:H_vs_gBL}) and $J(x)$ has full column rank, the proof of \cref{cor:centroid_norm_estimate} applies verbatim under this substitution, giving $\hat{\rmH}(x) = \big(\hat{\rmH}^{-1}(x)\big)^+ = c_2^{-1} J(x)^{+\top} \rmH(x) J(x)^+$.

  It remains to compare the two norms. Using $J(x)^+J(x) = I_m$,
  \begin{equation*}
    \lVert \hat{\rvc}(x)\rVert^2_{\hat{\rmH}} = c_2\, \rvc(x)^\top J(x)^\top \cdot \frac{1}{c_2}J(x)^{+\top}\rmH(x) J(x)^+ \cdot J(x)\rvc(x) = \rvc(x)^\top \rmH(x) \rvc(x) = \lVert \rvc(x)\rVert^2_{\rmH},
  \end{equation*}
  so that $\hat{\lambda}(x) = 1 - \lVert \hat{\rvc}(x)\rVert^2_{\hat{\rmH}} = 1 - \lVert \rvc(x)\rVert^2_{\rmH} = \lambda(x)$.
\end{proof}

Linking the two Randers metrics $\FF_R$ and $\hat{\FF}_R$ together results in \cref{prop:embedded_randers}, which we restate here for convenience:

\EmpiricalRandersProp*

\begin{proof}[Proof of \cref{prop:embedded_randers}]
  By \cref{cor:centroid_norm_estimate} and \cref{prop:H_vs_gBL} the assumption $\lVert \hat{\rvc}(x)\rVert_{\hat{g}_{BL}} < (m+3)^{-1/2}$ is equivalent to $\lVert \rvc(x)\rVert_{\rmH} < 1$, so that $\FF_R$ is a valid Randers metric at $x$. By \cref{lem:empirical_pushforward} and $J^+J = I_m$, for every $v \in T_x\gM$,
  \begin{equation*}
    \lVert J v\rVert^2_{\hat{\rmH}} = \frac{1}{c_2}\lVert v \rVert^2_{\rmH} , \qquad
    \langle \hat{\rvc}, Jv\rangle_{\hat{\rmH}} = \sqrt{c_2}\,\rvc^\top J^\top \frac{1}{c_2}J^{+\top}\rmH J^+ J v = \frac{1}{\sqrt{c_2}}\langle \rvc, v\rangle_{\rmH} .
  \end{equation*}
  Substituting these and $\hat{\lambda} = \lambda$ into \cref{eq:moment_randers_navigation},
  \begin{align*}
    \hat{\FF}_R(x, J(x)v)
    &= \frac{1}{\lambda}\left( \sqrt{\frac{\lambda}{c_2} \lVert v \rVert_{\rmH}^2 + \frac{1}{c_2}\langle \rvc, v \rangle_{\rmH}^2 } - \frac{1}{\sqrt{c_2}}\langle \rvc, v \rangle_{\rmH} \right)
    = \frac{1}{\sqrt{c_2}}\, \FF_R(x,v) . \qedhere
  \end{align*}
\end{proof}

In other words, $\hat{\FF}_R = c_2^{-1/2} \Psi_*\FF_R$ on the embedded tangent space $J(x)(T_x\gM)$: the moment-Randers construction commutes with the embedding, which is what makes it computable. Note that $\hat{\FF}_R$ is a Randers metric on $J(x)(T_x\gM)$ only, and not on the whole of $\R^\ell$, since $\hat{\rmH}^{-1}$ has rank $m$.

\paragraph{Discrete setting.} All the quantities above are evaluated from the discrete operators of \cref{eq:discrete_operators} by replacing $\gL^s$ and $\gL^a$ with $\rmL^s$ and $\rmL^a$, and $\Psi$ with the embedded samples $Y_i = \Psi(X_i)$. Writing $\psi_k \in \R^N$ for the $k$-th embedding coordinate evaluated at the samples and $\odot$ for the entrywise product,
\begin{equation*}
  V_i = \rmL^a[\Psi]_i , \qquad
  \Gamma_i^{kl} = \frac{1}{2}\Big[ \big(\rmL^s(\psi_k \odot \psi_l)\big)_i - \psi_k(X_i)\big(\rmL^s \psi_l\big)_i - \psi_l(X_i)\big(\rmL^s \psi_k\big)_i \Big] ,
\end{equation*}
and $\lVert \rvc(X_i)\rVert^2_{\gBL} = \frac{c_2}{c_1^2} V_i^\top \Gamma_i^+ V_i$ by \cref{cor:centroid_norm_estimate}. Whenever this quantity is below $(m+3)^{-1}$, \cref{eq:empirical_randers_data} defines $\hat{\rvc}_i$, $\hat{\rmH}_i$ and $\hat{\lambda}_i$, and hence the metric $\hat{\FF}_R(X_i, \cdot)$ at the sample $X_i$. By \cref{thm:convergence_discrete_op}, $\rmL^s$ and $\rmL^a$ converge uniformly and almost surely to $\gL^s$ and $\gL^a$, so these empirical quantities converge to their continuous counterparts at every sample.

\subsection{Computation of the kernel moments}
\label{appendix:kernel_moments}

  The values of the kernel moments $\mu_n$ can be computed analytically for the usual kernel functions $K$.
\subsubsection*{Exponential kernel}
  If $K$ is the exponential kernel $K(x) = \exp(-x)$, then we can compute, for $m\in \mathbb{N}$:
  \begin{align*}
    \mu_n = \int_0^{+\infty} K(r) r^{n+m-1} \diff r &= \int_0^{+\infty} e^{-r} r^{n+m-1} \diff r = \Gamma(n+m) = (n+m-1)!
  \end{align*}

\subsubsection*{Gaussian kernel}
  If $K$ is the Gaussian kernel $K(x) = \exp(-x^2)$, then we can compute, for $m\in \mathbb{N}$, using the change of variable $u = r^2$, so that $\diff r = \diff u / (2\sqrt{u})$:
  \begin{align*}
    \mu_n = \int_0^{+\infty} K(r) r^{n+m-1} \diff r &= \int_0^{+\infty} e^{-r^2} r^{n+m-1} \diff r\\
    &= \frac{1}{2} \int_0^{+\infty} e^{-u} u^{(n+m-1)/2 - 1/2} \diff u \\
    &= \frac{1}{2} \int_0^{+\infty} e^{-u} u^{(n+m)/2 - 1} \diff u \\
    &= \frac{1}{2} \Gamma\left(\frac{n+m}{2}\right)
  \end{align*}

\end{document}

%% file: math_commands.tex
\usepackage{amsmath,amsfonts,bm}
\usepackage{tabularx}
\usepackage{makecell}
\usepackage{amsmath}
\usepackage{subcaption}
\usepackage{todonotes}
\usepackage{thmtools, thm-restate}
\usepackage{xcolor}
\usepackage[cal=boondoxo]{mathalfa} %
\usepackage[normalem]{ulem} %
\usepackage{enumitem}
\usepackage{float}
\usepackage{xr}
\usepackage{balance}
\usepackage{xspace}
\usepackage{url}
\usepackage{expl3}
\usepackage{amsmath}
\usepackage{amssymb}
\usepackage{mathtools}
\usepackage{dsfont}
\usepackage{upgreek}
\usepackage{bbm}
\usepackage{graphicx}
\usepackage{epstopdf}
\usepackage{grffile}
\usepackage{tikz}
\usetikzlibrary{arrows.meta}
\usepackage{float}
\usepackage{svg}
\usepackage{algorithm}
\usepackage{algorithmic}
\usepackage[labelfont = bf]{caption}
\usepackage{subcaption}
\usepackage{relsize}  %
\usepackage{makecell}
\usepackage{times}
\usepackage{url}
\usepackage{hyperref}
\usepackage{graphicx}
\usepackage{refcount}
\usepackage{xstring}
\usepackage{geometry}
\usepackage{amsmath}
\usepackage{adjustbox}
\usepackage{amssymb}
\usepackage{mathtools}
\usepackage[table]{xcolor}
\usepackage{xcolor}
\usepackage{multirow}
\usepackage{stmaryrd}
\usepackage{standalone}
\usepackage{natbib}
\usepackage{rotating}
\usepackage{amsthm}

\def\eqref#1{equation~\ref{#1}}

\def\1{\bm{1}}

\def\eps{{\varepsilon}}

\def\rvc{{\mathbf{c}}}

\def\rvf{{\mathbf{f}}}

\def\rmA{{\mathbf{A}}}

\def\rmD{{\mathbf{D}}}
\def\rmE{{\mathbf{E}}}

\def\rmH{{\mathbf{H}}}
\def\rmI{{\mathbf{I}}}
\def\rmJ{{\mathbf{J}}}

\def\rmL{{\mathbf{L}}}

\def\rmP{{\mathbf{P}}}
\def\rmQ{{\mathbf{Q}}}
\def\rmR{{\mathbf{R}}}
\def\rmS{{\mathbf{S}}}

\def\rmV{{\mathbf{V}}}
\def\rmW{{\mathbf{W}}}

\def\rmY{{\mathbf{Y}}}

\DeclareMathAlphabet{\mathsfit}{\encodingdefault}{\sfdefault}{m}{sl}
\SetMathAlphabet{\mathsfit}{bold}{\encodingdefault}{\sfdefault}{bx}{n}

\def\gA{{\mathcal{A}}}

\def\gC{{\mathcal{C}}}

\def\gF{{\mathcal{F}}}
\def\gG{{\mathcal{G}}}

\def\gI{{\mathcal{I}}}

\def\gL{{\mathcal{L}}}
\def\gM{{\mathcal{M}}}

\def\gP{{\mathcal{P}}}

\def\sN{{\mathbb{N}}}

\newcommand{\R}{\mathbb{R}}

\usepackage{aliascnt}

\newtheorem{theorem}{Theorem}

\newaliascnt{property}{theorem}

\aliascntresetthe{property}

\newaliascnt{assumption}{theorem}

\aliascntresetthe{assumption}

\newaliascnt{proposition}{theorem}
\newtheorem{proposition}[proposition]{Proposition}
\aliascntresetthe{proposition}

\newaliascnt{corollary}{theorem}
\newtheorem{corollary}[corollary]{Corollary}
\aliascntresetthe{corollary}

\newaliascnt{lemma}{theorem}
\newtheorem{lemma}[lemma]{Lemma}
\aliascntresetthe{lemma}

\newaliascnt{definition}{theorem}
\newtheorem{definition}[definition]{Definition}
\aliascntresetthe{definition}

\newaliascnt{example}{theorem}
\newtheorem{example}[example]{Example}
\aliascntresetthe{example}

\newaliascnt{remark}{theorem}
\newtheorem{remark}[remark]{Remark}
\aliascntresetthe{remark}

\usepackage{thmtools}

\usepackage[capitalise,noabbrev]{cleveref}

\crefname{theorem}{Theorem}{Theorems}
\crefname{property}{Property}{Properties}
\crefname{assumption}{Assumption}{Assumptions}
\crefname{proposition}{Proposition}{Propositions}
\crefname{corollary}{Corollary}{Corollaries}
\crefname{lemma}{Lemma}{Lemmas}
\crefname{definition}{Definition}{Definitions}
\crefname{example}{Example}{Examples}
\crefname{remark}{Remark}{Remarks}

\crefname{figure}{Figure}{Figures}

\newcommand{\appref}[1]{Appendix~\ref{#1}}

\definecolor{blue1}{HTML}{2851CC}
\definecolor{red1}{HTML}{FF0000}
\definecolor{green1}{HTML}{3BB273}
\definecolor{orange1}{HTML}{F96C39}
\definecolor{pink1}{HTML}{F056BF}
\definecolor{gray1}{HTML}{696A62}
\definecolor{gray2}{HTML}{E3E4DB}

\definecolor{todocolor}{HTML}{F17105}
\definecolor{newcolor}{HTML}{682ACB}
\definecolor{newakcolor}{HTML}{07B053}

\newcommand{\norm}[1]{\left\lVert#1\right\rVert}
\newcommand{\abs}[1]{\left|#1\right|}
\newcommand{\N}{\mathbb{N}}

\newcommand{\MM}{\gM}
\newcommand{\FF}{\gF}
\newcommand{\backFF}{{\overleftarrow{\FF}}}
\newcommand{\backm}{{\overleftarrow{m}}}

\newcommand{\Pts}{\gP^{(\theta,s)}}
\newcommand{\Ptb}{\gP^{(\theta, \bullet)}}
\newcommand{\Pta}{\gP^{(\theta,a)}}
\newcommand{\gGa}{{\gG^{a}_\FF}}
\newcommand{\gGs}{{\gG^{s}_\FF}}
\newcommand{\backgGt}{{\gG^{(\theta)}_{\backFF}}}
\newcommand{\gGt}{{\gG^{(\theta)}_\FF}}
\newcommand{\gGtb}{{\gG^{(\theta,\bullet)}_\FF}}
\newcommand{\gGta}{{\gG^{(\theta,a)}_\FF}}
\newcommand{\gGts}{{\gG^{(\theta,s)}_\FF}}
\newcommand{\phit}{{\phi^{(\theta)}}}

\newcommand{\qe}{{q_\eps}}

\newcommand{\dts}{d^{(\theta,s)}}
\newcommand{\dta}{d^{(\theta,a)}}
\newcommand{\dstF}{{\text{dst}_\gF}}

\newcommand{\expF}{\exp^\FF}
\newcommand{\expa}{\exp^\alpha}
\newcommand{\dst}{\text{dst}}
\newcommand{\ChrisF}{\tilde{\Gamma}}
\newcommand{\RicF}{\text{Ric}^\FF}
\newcommand{\backRicF}{\text{Ric}^{\backFF}}
\newcommand{\backChrisF}{\overleftarrow{\ChrisF}}
\newcommand{\backG}{\overleftarrow{G}}
\newcommand{\backLambda}{\overleftarrow{\Lambda}}
\newcommand{\backR}{\overleftarrow{R}}
\newcommand{\backg}{\overleftarrow{g}}
\newcommand{\backN}{\overleftarrow{N}}
\newcommand{\backS}{\overleftarrow{S}}
\newcommand{\backtau}{\overleftarrow{\tau}}
\newcommand{\backs}{\overleftarrow{s}}
\newcommand{\backUpsilon}{\overleftarrow{\Upsilon}}

\newcommand{\dout}{{d}}
\newcommand{\din}{{d'}}
\newcommand{\dint}{{d'^{(\theta)}}}
\newcommand{\doutt}{{d^{(\theta)}}}
\newcommand{\dmin}{{d_{\min}}}

\newcommand{\vole}{{\operatorname{Vol}_\textnormal{Euc}}}
\newcommand{\vol}{{\operatorname{Vol}}}
\newcommand{\diff}{\mathrm{d}}
\newcommand{\mBH}{\mathfrak{m}_\textnormal{BH}}
\newcommand{\leb}{\lambda}
\newcommand{\Jac}{\mathcal{J}_x}
\newcommand{\tball}{\mathcal{B}}   %
\newcommand{\eball}{\mathbb{B}}    %
\newcommand{\fball}{B^{+}_{\FF}}   %

\newcommand{\gBL}{g_\textnormal{BL}} %
\newcommand{\sigBH}{\sigma_\textnormal{BH}} %

\newcommand{\ie}   			{i.e.\@\xspace}
\newcommand{\eg}   			{e.g.\@\xspace}

\newcommand{\iid}   		{i.i.d.\@\xspace}

\newcommand{\bigO} 		  {\mathcal{O}} %

\newcommand{\Prob}      {\mathbb{P}}

\renewcommand*{\top}{{\mkern-1.5mu\mathsf{T}}}
\newcommand{\dist} {\operatorname{dist}}

\newcommand{\padded}[1] {\,#1\,}
\newcommand{\void}[1] {\padded{\cdot}}

\makeatletter
\def\adl@drawiv#1#2#3{%
        \hskip.5\tabcolsep
        \xleaders#3{#2.5\@tempdimb #1{1}#2.5\@tempdimb}%
                #2\z@ plus1fil minus1fil\relax
        \hskip.5\tabcolsep}
\newcommand{\cmidruledashed}[1]{%
  \noalign{\vskip\aboverulesep
           \global\let\@dashdrawstore\adl@draw
           \global\let\adl@draw\adl@drawiv}
  \cdashline{#1}%
	\noalign{\vskip-\belowrulesep}
	\noalign{\vskip-\belowrulesep}
}
\makeatother

\makeatletter
\newcommand{\pushright}[1]{\ifmeasuring@#1\else\omit\hfill$\displaystyle#1$\fi\ignorespaces}
\newcommand{\pushleft}[1]{\ifmeasuring@#1\else\omit$\displaystyle#1$\hfill\fi\ignorespaces}
\makeatother

\makeatletter
\let\KV@Gin@trim@old\KV@Gin@trim
\define@key{Gin}{trim}{%
  \begingroup
  \edef\x{\endgroup
    \noexpand\setkeys{Gin}{trim@old={#1}}%
  }\x
}
\makeatother
\makeatletter
\let\KV@Gin@viewport@old\KV@Gin@viewport
\define@key{Gin}{viewport}{%
  \begingroup
  \edef\x{\endgroup
    \noexpand\setkeys{Gin}{viewport@old={#1}}%
  }\x
}
\makeatother

\newcommand{\VOID}[1]{}

%% file: euclidean_distance.tex
\begin{tikzpicture}[scale=1.5]

\def\manifoldcoords{(-0.5,0.5) (0.8,0.8) (2.2,0.6) (2.8,0) (2.5,-0.8) (1.5,-1) (0.3,-0.9) (-0.2,-0.2)}

\fill[blue1!15, opacity=0.6] plot[smooth cycle, tension=0.8] coordinates {\manifoldcoords};

\coordinate (A) at (0.5, 0.2);
\coordinate (B) at (2.0, -0.3);

\draw[red1, thick, latex-latex] (A) -- (B);

\fill[black] (A) circle (0.04) node[left] {$x$};
\fill[black] (B) circle (0.04) node[right] {$y$};

\node[blue1, font=\small] at (2.5, 0.7) {$\mathcal{M}$};

\end{tikzpicture}%

%% file: riemannian_distance.tex
\begin{tikzpicture}[scale=1.5]

\def\manifoldcoords{(-0.5,0.5) (0.8,0.8) (2.2,0.6) (2.8,0) (2.5,-0.8) (1.5,-1) (0.3,-0.9) (-0.2,-0.2)}

\fill[blue1!15, opacity=0.6] plot[smooth cycle, tension=0.8] coordinates {\manifoldcoords};

\coordinate (A) at (0.5, 0.2);
\coordinate (B) at (2.0, -0.3);

\draw[red1, thick, latex-latex] (A) to[out=70, in=120, bend left=20] (B);

\fill[black] (A) circle (0.04) node[left] {$x$};
\fill[black] (B) circle (0.04) node[right] {$y$};

\node[blue1, font=\small] at (2.5, 0.7) {$\mathcal{M}$};

\end{tikzpicture}%

%% file: finsler_distance.tex
\begin{tikzpicture}[scale=1.5]

\def\manifoldcoords{(-0.5,0.5) (0.8,0.8) (2.2,0.6) (2.8,0) (2.5,-0.8) (1.5,-1) (0.3,-0.9) (-0.2,-0.2)}

\fill[blue1!15, opacity=0.6] plot[smooth cycle, tension=0.8] coordinates {\manifoldcoords};

\coordinate (A) at (0.5, 0.2);
\coordinate (B) at (2.0, -0.3);

\draw[red1, thick, -latex] (A) to[out=10, in=180, bend left=30] (B);

\draw[green1, thick, -latex] (B) to[out=180, in=-10, bend left=30] (A);

\fill[black] (A) circle (0.04) node[left] {$x$};
\fill[black] (B) circle (0.04) node[right] {$y$};

\node[blue1, font=\small] at (2.5, 0.7) {$\mathcal{M}$};

\end{tikzpicture}%

%% file: disbm_cycle_model.tex
\begin{tikzpicture}[scale=1.5,
    community/.style={circle, fill=#1, minimum size=0.95cm, inner sep=0pt, text=white, font=\small},
    next/.style={#1, line width=1.1pt, -{Latex[length=2.2mm]}},
    prev/.style={gray1, thin, dashed, -{Latex[length=1.6mm]}},
    same/.style={#1, thick, -{Latex[length=1.8mm]}},
    lbl/.style={font=\small, inner sep=1.5pt},
  ]

\node[community=blue1]   (C1) at (90:1.3)  {$C_1$};
\node[community=orange1] (C2) at (-30:1.3) {$C_2$};
\node[community=green1]  (C3) at (210:1.3) {$C_3$};

\draw[next=blue1]   (C1) to[bend left=28] node[lbl, pos=0.5, above right] {$p$} (C2);
\draw[next=orange1] (C2) to[bend left=28] node[lbl, pos=0.5, below]       {$p$} (C3);
\draw[next=green1]  (C3) to[bend left=28] node[lbl, pos=0.5, above left]  {$p$} (C1);

\draw[prev] (C2) to[bend left=18] node[lbl, pos=0.5, below left]  {$q$} (C1);
\draw[prev] (C3) to[bend left=18] node[lbl, pos=0.5, above]       {$q$} (C2);
\draw[prev] (C1) to[bend left=18] node[lbl, pos=0.5, below right] {$q$} (C3);

\draw[same=blue1]   (C1) to[out=65,  in=115, looseness=6] node[lbl, above, inner sep=4pt] {$r$} (C1);
\draw[same=orange1] (C2) to[out=-55, in=-5,  looseness=6] node[lbl, below right, inner sep=3pt] {$r$} (C2);
\draw[same=green1]  (C3) to[out=185, in=235, looseness=6] node[lbl, below left, inner sep=3pt]  {$r$} (C3);

\end{tikzpicture}%

%% file: disbm_cycle_graph.tex
\begin{tikzpicture}[scale=1.5,
    vertex/.style={circle, fill=#1, minimum size=5pt, inner sep=0pt},
    intra/.style={gray1!70, thin, -{Latex[length=1.1mm]}},
    fwd/.style={#1, semithick, -{Latex[length=1.6mm]}},
    bwd/.style={gray1, thin, dashed, -{Latex[length=1.4mm]}},
  ]

\coordinate (c1) at (90:1.3);
\coordinate (c2) at (-30:1.3);
\coordinate (c3) at (210:1.3);

\foreach \k/\col in {1/blue1, 2/orange1, 3/green1}
  \fill[\col!15, opacity=0.8] (c\k) circle (0.5);

\foreach \k/\col in {1/blue1, 2/orange1, 3/green1}
  \foreach \i/\dx/\dy in {1/0/0, 2/0.3/0.14, 3/-0.29/0.17, 4/0.02/0.34, 5/0.24/-0.26, 6/-0.21/-0.28}
    \node[vertex=\col] (n\k\i) at ([shift={(\dx,\dy)}]c\k) {};

\foreach \k in {1,2,3}
  \foreach \a/\b in {1/2, 3/1, 4/2, 3/4, 5/1, 6/5, 3/6, 1/4}
    \draw[intra] (n\k\a) -- (n\k\b);

\draw[fwd=blue1]   (n14) to[bend left=22] (n22);
\draw[fwd=blue1]   (n12) to[bend left=12] (n24);
\draw[fwd=blue1]   (n15) to[bend left=6]  (n23);
\draw[fwd=orange1] (n25) to[bend left=22] (n36);
\draw[fwd=orange1] (n21) to[bend left=12] (n35);
\draw[fwd=orange1] (n26) to[bend left=6]  (n31);
\draw[fwd=green1]  (n36) to[bend left=22] (n13);
\draw[fwd=green1]  (n33) to[bend left=12] (n11);
\draw[fwd=green1]  (n34) to[bend left=6]  (n16);

\draw[bwd] (n23) to[bend left=28] (n16);
\draw[bwd] (n32) to[bend left=28] (n26);
\draw[bwd] (n15) to[bend left=28] (n32);

\end{tikzpicture}%

%% file: main_iclr.bbl
\begin{thebibliography}{57}
\providecommand{\natexlab}[1]{#1}
\providecommand{\url}[1]{\texttt{#1}}
\expandafter\ifx\csname urlstyle\endcsname\relax
  \providecommand{\doi}[1]{doi: #1}\else
  \providecommand{\doi}{doi: \begingroup \urlstyle{rm}\Url}\fi

\bibitem[Arvanitidis et~al.(2019)Arvanitidis, Hauberg, Hennig, and
  Schober]{arvanitidis2019fast}
Georgios Arvanitidis, Soren Hauberg, Philipp Hennig, and Michael Schober.
\newblock Fast and robust shortest paths on manifolds learned from data.
\newblock In \emph{International Conference on Artificial Intelligence and
  Statistics}, 2019.

\bibitem[Asanjarani(2021)]{asanjarani2021finslertraffic}
Azam Asanjarani.
\newblock A {F}insler geometrical programming for the nonlinear complementarity
  problem of traffic equilibrium.
\newblock \emph{Preprint arXiv:2109.01256}, 2021.

\bibitem[Bao et~al.(2012)Bao, Chern, and Shen]{bao2012introduction}
David Bao, Shiing-Shen Chern, and Zhongmin Shen.
\newblock \emph{An introduction to {R}iemann-{F}insler geometry}.
\newblock Springer Science \& Business Media, 2012.

\bibitem[Belkin \& Niyogi(2001)Belkin and Niyogi]{belkin_laplacian_2001}
Mikhail Belkin and Partha Niyogi.
\newblock Laplacian eigenmaps and spectral techniques for embedding and
  clustering.
\newblock In \emph{Advances in Neural Information Processing Systems}, 2001.

\bibitem[Bergen et~al.(2021)Bergen, Soldatov, Kharchenko, and
  Theis]{bergen2021rna}
Volker Bergen, Ruslan~A Soldatov, Peter~V Kharchenko, and Fabian~J Theis.
\newblock {RNA} velocity—current challenges and future perspectives.
\newblock \emph{Molecular Systems Biology}, 17\penalty0 (8):\penalty0
  MSB202110282, 2021.

\bibitem[Calder \& Trillos(2022)Calder and Trillos]{calder_improved_2022}
Jeff Calder and Nicolas~Garcia Trillos.
\newblock Improved spectral convergence rates for graph {L}aplacians on
  $\varepsilon$-graphs and k-{NN} graphs.
\newblock \emph{Applied and Computational Harmonic Analysis}, 60:\penalty0
  123--175, 2022.

\bibitem[Chansri et~al.(2018)Chansri, Chansangiam, and
  Sabau]{chansri2018geometry}
P~Chansri, Pattrawut Chansangiam, and Sorin~V Sabau.
\newblock The geometry on the slope of a mountain.
\newblock \emph{Preprint arXiv:1811.02123}, 2018.

\bibitem[Chen et~al.(2024)Chen, Mirebeau, Shu, and Cohen]{chen2024region}
Da~Chen, Jean-Marie Mirebeau, Huazhong Shu, and Laurent~D Cohen.
\newblock A region-based {R}anders geodesic approach for image segmentation.
\newblock \emph{International Journal of Computer Vision}, 132\penalty0
  (2):\penalty0 349--391, 2024.

\bibitem[Chen et~al.(2007)Chen, Yang, and Tang]{chen2007directed}
Mo~Chen, Qiong Yang, and Xiaoou Tang.
\newblock Directed graph embedding.
\newblock In \emph{International Joint Conference on Artificial Intelligence},
  2007.

\bibitem[Chung(2005)]{chung_laplacians_2005}
Fan Chung.
\newblock Laplacians and the {C}heeger inequality for directed graphs.
\newblock \emph{Annals of Combinatorics}, 9\penalty0 (1):\penalty0 1--19, 2005.

\bibitem[Coifman \& Lafon(2006)Coifman and Lafon]{coifman_diffusion_2006}
Ronald~R Coifman and St{\'e}phane Lafon.
\newblock Diffusion maps.
\newblock \emph{Applied and Computational Harmonic Analysis}, 21\penalty0
  (1):\penalty0 5--30, 2006.

\bibitem[Cucuringu et~al.(2020)Cucuringu, Li, Sun, and
  Zanetti]{cucuringu_hermitian_2020}
Mihai Cucuringu, Huan Li, He~Sun, and Luca Zanetti.
\newblock Hermitian matrices for clustering directed graphs: insights and
  applications.
\newblock In \emph{International Conference on Artificial Intelligence and
  Statistics}, 2020.

\bibitem[Dag\`es et~al.(2025)Dag\`es, Weber, Lin, Talmon, Cremers, Lindenbaum,
  Bruckstein, and Kimmel]{dages_finsler_2025}
Thomas Dag\`es, Simon Weber, Ya-Wei~Eileen Lin, Ronen Talmon, Daniel Cremers,
  Michael Lindenbaum, Alfred~M. Bruckstein, and Ron Kimmel.
\newblock Finsler multi-dimensional scaling: Manifold learning for asymmetric
  dimensionality reduction and embedding.
\newblock In \emph{IEEE/CVF Conference on Computer Vision and Pattern
  Recognition}, 2025.

\bibitem[Dag{\`e}s et~al.(2026)Dag{\`e}s, Weber, Cremers, and
  Kimmel]{dages_harnessing_2026}
Thomas Dag{\`e}s, Simon Weber, Daniel Cremers, and Ron Kimmel.
\newblock Harnessing data asymmetry: Manifold learning in the {F}insler world.
\newblock \emph{Preprint arXiv:2603.11396}, 2026.

\bibitem[Debaussart-Joniec et~al.(2026)Debaussart-Joniec, Sevi, Jonckheere, and
  Kalogeratos]{debaussart_parpic_2026}
Gwendal Debaussart-Joniec, Harry Sevi, Matthieu Jonckheere, and Argyris
  Kalogeratos.
\newblock Parametrized power-iteration clustering for directed graphs.
\newblock In \emph{International Conference on Machine Learning}, 2026.

\bibitem[Donoho \& Grimes(2003)Donoho and Grimes]{donoho2003hessian}
David~L Donoho and Carrie Grimes.
\newblock Hessian eigenmaps: Locally linear embedding techniques for
  high-dimensional data.
\newblock \emph{Proceedings of the National Academy of Sciences}, 2003.

\bibitem[Fanuel et~al.(2018)Fanuel, Alaíz, Ángela Fernández, and
  Suykens]{fanuel_magnetic_2018}
Michaël Fanuel, Carlos~M. Alaíz, Ángela Fernández, and Johan~A.K. Suykens.
\newblock Magnetic eigenmaps for the visualization of directed networks.
\newblock \emph{Applied and Computational Harmonic Analysis}, 44\penalty0
  (1):\penalty0 189--199, 2018.

\bibitem[Finsler(1918)]{finsler1918ueber}
Paul Finsler.
\newblock \emph{{\"U}ber Kurven und Fl{\"a}chen in allgemeinen R{\"a}umen}.
\newblock Philos. Fak., Georg-August-Univ., 1918.

\bibitem[Gahtan et~al.(2026)Gahtan, Shpund, and
  Bronstein]{gahtan2026differentiable}
Barak Gahtan, Jacob Shpund, and Alex~M Bronstein.
\newblock Differentiable {R}anders-{F}insler eikonal solvers.
\newblock \emph{Computer Graphics Forum}, pp.\  e70489, 2026.

\bibitem[He et~al.(2023)He, He, Shi, Huang, and Suykens]{he2023learning}
Mingzhen He, Fan He, Lei Shi, Xiaolin Huang, and Johan~A.K. Suykens.
\newblock Learning with asymmetric kernels: Least squares and feature
  interpretation.
\newblock \emph{IEEE Transactions on Pattern Analysis and Machine
  Intelligence}, 45\penalty0 (8):\penalty0 10044--10054, 2023.

\bibitem[Hein et~al.(2007)Hein, Audibert, and Luxburg]{hein2007graph}
Matthias Hein, Jean-Yves Audibert, and Ulrike~von Luxburg.
\newblock Graph {L}aplacians and their convergence on random neighborhood
  graphs.
\newblock \emph{Journal of Machine Learning Research}, 8\penalty0 (6), 2007.

\bibitem[Kipf \& Welling(2016)Kipf and Welling]{kipf2016semi}
Thomas~N Kipf and Max Welling.
\newblock Semi-supervised classification with graph convolutional networks.
\newblock \emph{Preprint arXiv:1609.02907}, 2016.

\bibitem[Lafon(2004)]{lafon2004diffusion}
St{\'e}phane~S Lafon.
\newblock \emph{Diffusion maps and geometric harmonics}.
\newblock Yale University, 2004.

\bibitem[Lee(2012)]{lee2012introduction}
John~M. Lee.
\newblock \emph{Introduction to Smooth Manifolds}.
\newblock Graduate Texts in Mathematics. Springer, 2 edition, 2012.

\bibitem[Lopez et~al.(2021)Lopez, Pozzetti, Trettel, Strube, and
  Wienhard]{lopez_2021_Symmetric}
Federico Lopez, Beatrice Pozzetti, Steve Trettel, Michael Strube, and Anna
  Wienhard.
\newblock Symmetric spaces for graph embeddings: A {F}insler-{R}iemannian
  approach.
\newblock In \emph{International Conference on Machine Learning}, 2021.

\bibitem[Markvorsen(2016)]{markvorsenFinslerGeodesicSpray2016}
Steen Markvorsen.
\newblock A {F}insler geodesic spray paradigm for wildfire spread modelling.
\newblock \emph{Nonlinear Analysis: Real World Applications}, 28:\penalty0
  208--228, 2016.

\bibitem[Matsumoto(1989)]{matsumoto1989slope}
Makoto Matsumoto.
\newblock A slope of a mountain is a {F}insler surface with respect to a time
  measure.
\newblock \emph{Journal of Mathematics of Kyoto University}, 29\penalty0
  (1):\penalty0 17--25, 1989.

\bibitem[Matsumoto(1992)]{matsumoto_theory_1992}
Makoto Matsumoto.
\newblock Theory of {F}insler spaces with ($\alpha$, $\beta$)-metric.
\newblock \emph{Reports on Mathematical Physics}, 31\penalty0 (1):\penalty0
  43--83, 1992.

\bibitem[Matveev \& Troyanov(2012)Matveev and Troyanov]{matveev2012binet}
Vladimir~S Matveev and Marc Troyanov.
\newblock The {B}inet--{L}egendre metric in {F}insler geometry.
\newblock \emph{Geometry \& Topology}, 16\penalty0 (4):\penalty0 2135--2170,
  2012.

\bibitem[Melonakos et~al.(2008)Melonakos, Pichon, Angenent, and
  Tannenbaum]{melanokos_2008_finsler}
John Melonakos, Eric Pichon, Sigurd Angenent, and Allen Tannenbaum.
\newblock Finsler active contours.
\newblock \emph{IEEE Transactions on Pattern Analysis and Machine
  Intelligence}, 30\penalty0 (3), 2008.

\bibitem[Ohta(2021)]{ohtaComparisonFinslerGeometry2021}
Shin-ichi Ohta.
\newblock \emph{Comparison {{Finsler Geometry}}}.
\newblock Springer {{Monographs}} in {{Mathematics}}. Springer, 2021.

\bibitem[Ou et~al.(2016)Ou, Cui, Pei, Zhang, and Zhu]{ou2016asymmetric}
Mingdong Ou, Peng Cui, Jian Pei, Ziwei Zhang, and Wenwu Zhu.
\newblock Asymmetric transitivity preserving graph embedding.
\newblock In \emph{ACM SIGKDD International Conference on Knowledge Discovery
  and Data Mining}, 2016.

\bibitem[Peoples \& Harlim(2025)Peoples and
  Harlim]{wilson_peoples_spectral_2025}
J.~Wilson Peoples and John Harlim.
\newblock Spectral convergence of symmetrized graph {L}aplacian on manifolds
  with boundary.
\newblock \emph{Foundations of Data Science}, 8\penalty0 (0):\penalty0
  119--167, 2025.

\bibitem[Perrault-Joncas \& Meila(2011)Perrault-Joncas and
  Meila]{perrault2011directed}
Dominique Perrault-Joncas and Marina Meila.
\newblock Directed graph embedding: an algorithm based on continuous limits of
  {L}aplacian-type operators.
\newblock \emph{Advances in Neural Information Processing Systems}, 24, 2011.

\bibitem[Pfeifer(2019)]{pfeifer2019finsler}
Christian Pfeifer.
\newblock Finsler spacetime geometry in physics.
\newblock \emph{International Journal of Geometric Methods in Modern Physics},
  16\penalty0 (supp02):\penalty0 1941004, 2019.

\bibitem[Pouplin et~al.(2023)Pouplin, Eklund, Ek, and
  Hauberg]{pouplin2023identifyinglatentdistancesfinslerian}
Alison Pouplin, David Eklund, Carl~Henrik Ek, and Søren Hauberg.
\newblock Identifying latent distances with {F}inslerian geometry.
\newblock \emph{Preprint arXiv:2212.10010}, 2023.

\bibitem[Randers(1941)]{randers1941asymmetrical}
Gunnar Randers.
\newblock On an asymmetrical metric in the four-space of general relativity.
\newblock \emph{Physical Review}, 59\penalty0 (2):\penalty0 195, 1941.

\bibitem[Rygaard \& Hauberg(2025)Rygaard and Hauberg]{rygaard2025georce}
Frederik~M{\"o}bius Rygaard and S{\o}ren Hauberg.
\newblock {GEORCE}: A fast new control algorithm for computing geodesics.
\newblock \emph{Preprint arXiv:2505.05961}, 2025.

\bibitem[Satuluri \& Parthasarathy(2011)Satuluri and
  Parthasarathy]{satuluri_symmetrizations_2011}
Venu Satuluri and Srinivasan Parthasarathy.
\newblock Symmetrizations for clustering directed graphs.
\newblock In \emph{International Conference on Extending Database Technology},
  2011.

\bibitem[Sevi et~al.(2026)Sevi, Debaussart-Joniec, Hacini, Jonckheere, and
  Kalogeratos]{sevi_generalized_2026}
Harry Sevi, Gwendal Debaussart-Joniec, Malik Hacini, Matthieu Jonckheere, and
  Argyris Kalogeratos.
\newblock Generalized {D}irichlet energy and graph {L}aplacians for clustering
  directed and undirected graphs.
\newblock \emph{Transactions on Machine Learning Research}, 2026.

\bibitem[Shaska(2025)]{shaska_finsler_2025}
Tony Shaska.
\newblock Geometric learning and {F}insler dissimilarity in weighted projective
  spaces.
\newblock \emph{Preprint arXiv:2507.00001}, 2025.

\bibitem[Shen \& Shen(2016)Shen and Shen]{shen2016introduction}
Yi-Bing Shen and Zhongmin Shen.
\newblock \emph{Introduction to modern {F}insler geometry}.
\newblock World Scientific Publishing Company, 2016.

\bibitem[Shen(2001)]{shen2001lectures}
Zhongmin Shen.
\newblock \emph{Lectures on {F}insler geometry}.
\newblock World Scientific, 2001.

\bibitem[Singer(2006)]{singer_from_2006}
Amit Singer.
\newblock From graph to manifold {L}aplacian: The convergence rate.
\newblock \emph{Applied and Computational Harmonic Analysis}, 21\penalty0
  (1):\penalty0 128--134, 2006.
\newblock Special Issue: Diffusion Maps and Wavelets.

\bibitem[Sommer et~al.(2020)Sommer, Fletcher, and
  Pennec]{Sommer2020introduction}
Stefan Sommer, Tom Fletcher, and Xavier Pennec.
\newblock 1 - introduction to differential and {R}iemannian geometry.
\newblock In Xavier Pennec, Stefan Sommer, and Tom Fletcher (eds.),
  \emph{Riemannian Geometric Statistics in Medical Image Analysis}, pp.\
  3--37. Academic Press, 2020.

\bibitem[Tao et~al.(2024)Tao, Tonin, Lambert, Chen, Patrinos, and
  Suykens]{tao2024learning}
Qinghua Tao, Francesco Tonin, Alex Lambert, Yingyi Chen, Panagiotis Patrinos,
  and Johan~A.K. Suykens.
\newblock Learning in feature spaces via coupled covariances: asymmetric kernel
  {SVD} and {N}ystr{\"o}m method.
\newblock In \emph{International Conference on Machine Learning}, 2024.

\bibitem[Tenenbaum et~al.(2000)Tenenbaum, Silva, and
  Langford]{tenenbaum2000global}
Joshua~B Tenenbaum, Vin~de Silva, and John~C Langford.
\newblock A global geometric framework for nonlinear dimensionality reduction.
\newblock \emph{Science}, 290\penalty0 (5500):\penalty0 2319--2323, 2000.

\bibitem[Vershynin(2019)]{vershynin2019high}
Roman Vershynin.
\newblock \emph{High-dimensional probability}, volume~47 of \emph{Cambridge
  Series in Statistical and Probabilistic Mathematics}.
\newblock Cambridge University Press, 2019.

\bibitem[Von~Luxburg(2007)]{von2007tutorial}
Ulrike Von~Luxburg.
\newblock A tutorial on spectral clustering.
\newblock \emph{Statistics and Computing}, 17\penalty0 (4):\penalty0 395--416,
  2007.

\bibitem[Weber et~al.(2024)Weber, Dag\`es, Gao, and
  Cremers]{Weber_finsler_2024}
Simon Weber, Thomas Dag\`es, Maolin Gao, and Daniel Cremers.
\newblock {F}insler-{L}aplace-{B}eltrami operators with application to shape
  analysis.
\newblock In \emph{IEEE/CVF Conference on Computer Vision and Pattern
  Recognition}, 2024.

\bibitem[Wu et~al.(2010)Wu, Xu, Li, and Oyama]{wu2010asymmetric}
Wei Wu, Jun Xu, Hang Li, and Satoshi Oyama.
\newblock Asymmetric kernel learning.
\newblock Technical report, Microsoft Research, 2010.

\bibitem[Yeon(2015)]{yeonHistoryBirthFinsler}
Won~Dae Yeon.
\newblock On the {{History}} of the {{Birth}} of {{Finsler Geometry}} at
  {{Göttingen}}.
\newblock \emph{Journal for History of Mathematics}, 2015.

\bibitem[Yuan et~al.(2022)Yuan, Calder, and Osting]{yuan2022continuum}
Amber Yuan, Jeff Calder, and Braxton Osting.
\newblock A continuum limit for the {P}age{R}ank algorithm.
\newblock \emph{European Journal of Applied Mathematics}, 33\penalty0
  (3):\penalty0 472--504, 2022.

\bibitem[Zermelo(1931)]{zermelo1931navigationsproblem}
Ernst Zermelo.
\newblock {\"U}ber das {N}avigationsproblem bei ruhender oder
  ver{\"a}nderlicher {W}indverteilung.
\newblock \emph{Journal of Applied Mathematics and Mechanics}, 11\penalty0
  (2):\penalty0 114--124, 1931.

\bibitem[Zhang et~al.(2021)Zhang, He, Brugnone, Perlmutter, and
  Hirn]{zhang_magnet_2021}
Xitong Zhang, Yixuan He, Nathan Brugnone, Michael Perlmutter, and Matthew Hirn.
\newblock {MagNet}: A neural network for directed graphs.
\newblock In \emph{Advances in Neural Information Processing Systems}, 2021.

\bibitem[Zhou et~al.(2003)Zhou, Bousquet, Lal, Weston, and
  Sch{\"o}lkopf]{zhou2003learning}
Dengyong Zhou, Olivier Bousquet, Thomas Lal, Jason Weston, and Bernhard
  Sch{\"o}lkopf.
\newblock Learning with local and global consistency.
\newblock \emph{Advances in Neural Information Processing Systems}, 16, 2003.

\bibitem[Zweig et~al.(2026)Zweig, Zhang, Knowles, and
  Azizi]{zweig_learning_2026}
Aaron Zweig, Mingxuan Zhang, David~A Knowles, and Elham Azizi.
\newblock Learning lineage-guided geodesics with finsler geometry.
\newblock \emph{Preprint arXiv:2603.16708}, 2026.

\end{thebibliography}
